\documentclass{article}

\PassOptionsToPackage{numbers,sort&compress}{natbib}

\usepackage[preprint]{neurips_2026}

\workshoptitle{Neural Network Artifacts as a New Data Modality}

\usepackage[utf8]{inputenc}
\usepackage[T1]{fontenc}
\usepackage{microtype}

\usepackage{mathtools}
\usepackage{amssymb}
\usepackage{amsthm}
\usepackage{bm}

\usepackage{graphicx}
\usepackage{booktabs}
\usepackage{multirow}
\usepackage{makecell}
\usepackage{longtable}
\usepackage{adjustbox}

\usepackage{xcolor}
\usepackage{url}
\usepackage{hyperref}

\newcommand{\R}{\mathbb{R}}
\newcommand{\E}{\mathbb{E}}
\newcommand{\norm}[1]{\left\lVert #1 \right\rVert}

\newcommand{\vecop}{\operatorname{vec}}
\newcommand{\Col}{\operatorname{Col}}
\newcommand{\St}{\operatorname{St}}

\theoremstyle{plain}
\newtheorem{assumption}{Assumption}
\newtheorem{proposition}{Proposition}
\newtheorem{theorem}{Theorem}
\newtheorem{corollary}{Corollary}

\newtheorem{lemma}{Lemma}

\title{Low-Rank Prompt Learning for Vision-Language Models
with Fixed-Token Bases}

\author{Tanvir Muntakim Tonoy, Sajjad Ghiasvand, Mahnoosh Alizadeh, \& Ramtin Pedarsani \\
Department of Electrical and Computer Engineering\\
 UC Santa Barbara\\
Santa Barbara, CA 93106, USA \\
\texttt{\{{tanvirmuntakim,sajjad,alizadeh,ramtin\}@ucsb.edu} 
}
}

\begin{document}

\maketitle

\begin{abstract}

Prompt learning adapts CLIP to downstream recognition by replacing hand-written
templates with learned continuous context vectors, which in Context Optimization
(CoOp) form a dense prompt matrix $\bm{P}\in\mathbb{R}^{m\times d}$ trained from only
a few examples per class. We study whether this matrix is over-parameterized by
factorizing it as $\bm{P}=\bm{B}\bm{A}$, which cuts the trainable prompt parameters
from $md$ to $r(m+d)$, and to $rd$ once the token-side factor $\bm{B}$ is fixed.
Across seven few-shot benchmarks and two CLIP backbones, low-rank prompts match or
improve dense CoOp at far fewer parameters, with the clearest gains on low-shot
base-to-new generalization. We then find that the token-side factor need not be
learned at all: fixing $\bm{B}$ to a Gaussian, orthogonal, SVD-derived, or even random
basis and training only the embedding-side factor $\bm{A}$ stays on par with the fully
trainable factorization, and a source-trained $\bm{B}$ offers no advantage over a
random one. A prompt-factor asymmetry and a local update-space dimension gap show why
fixing $\bm{B}$ is far less restrictive than fixing $\bm{A}$, and a smoothness-only
guarantee certifies that optimizing $\bm{A}$ over a fixed $\bm{B}$ converges. In the
CLIP prompt setting, the embedding-side coefficients carry the adaptation while the
token basis can simply be fixed.

\end{abstract}

\section{Introduction}
\begin{figure}[t]
\centering
\includegraphics[width=0.75\linewidth]{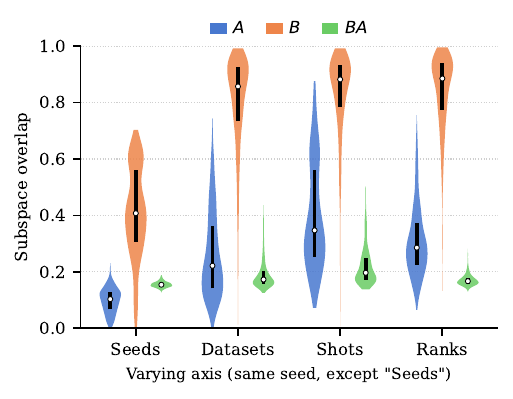}
\caption{Subspace overlap between independently trained factorized prompts, for
the token-side factor $\bm{B}$, embedding-side factor $\bm{A}$, and induced prompt
$\bm{B}\bm{A}$. Each violin aggregates all run pairs that vary the indicated axis
while holding the others (and, except in ``Seeds'', the seed) fixed; white dots mark
medians and bars the interquartile range. $\bm{B}$ is highly aligned across datasets,
shots, and ranks but much less so across seeds, while $\bm{A}$ and $\bm{B}\bm{A}$ stay
weakly aligned throughout.}
\label{fig:factor_geometry}
\end{figure}

Large vision--language models such as CLIP~\cite{radford2021learning} have become
a default interface for visual recognition: an image is classified by comparing its
visual embedding against text embeddings produced from class-name
prompts~\cite{radford2021learning, jia2021scaling}. Because the pretrained encoders
already capture broad open-vocabulary concepts, the dominant way to specialize them
to a downstream task is to keep the backbone frozen and adapt only the textual
prompt~\cite{zhou2022learning, gao2024clip, zhang2022tip}. Context Optimization
(CoOp)~\cite{zhou2022learning} is the canonical instance of this idea: it replaces
the hand-written template with $m$ learnable context vectors, stacked into a dense
prompt matrix $\bm{P}\in\mathbb{R}^{m\times d}$, and trains them by backpropagation
while the CLIP encoders stay frozen. Subsequent work conditions the prompt on the
image~\cite{cocoop} or extends prompting to the visual and cross-modal
branches~\cite{jia2022visual, bahng2022exploring, khattak2023maple, lee2023multimodal},
but the text-side prompt itself remains a dense, dataset-specific matrix.
 
In the few-shot regime that motivates prompt learning, the model sees only a handful
of labeled examples per class, yet $\bm{P}$ still carries $md$ free parameters
($16\times512=8{,}192$ in the standard CoOp configuration). This raises a basic
question: does effective prompt adaptation really require so many degrees of freedom,
or is the dense prompt over-parameterized for the data available? We study this by
constraining the prompt to be explicitly low rank,
\begin{equation}
\label{eq:intro-factorization}
\bm{P}=\bm{B}\bm{A},\qquad
\bm{B}\in\mathbb{R}^{m\times r},\quad
\bm{A}\in\mathbb{R}^{r\times d},\quad
r\ll\min(m,d).
\end{equation}
where the token-side factor $\bm{B}$ mixes $r$ latent components across the context
positions and the embedding-side factor $\bm{A}$ places them in the CLIP
text-embedding space. This reduces the trainable prompt parameters from $md$ to
$r(m+d)$, and to only $rd$ when $\bm{B}$ is held fixed.
 
Low-rank structure is central to parameter-efficient adaptation more broadly. LoRA~\cite{hu2021lora}
injects trainable low-rank updates into frozen weights, and vision--language variants
such as CLIP-LoRA~\cite{zanella2024low} and Block-LoRA~\cite{zhou2025one} adapt the
backbone projections, while prompt-space methods like DIP~\cite{hao2023towards} and
MMLoP~\cite{ghiasvand2026mmlop} impose low-rank structure on the prompt itself. A
common thread is that these methods train \emph{both} low-rank factors. A parallel
line of work on LoRA, however, finds the two factors are not interchangeable:
\citet{zhu2024asymmetry} show it is often preferable to freeze the input-side
projector, even at random, and tune only the output-side factor. Whether the same
asymmetry holds for prompt factorization, where the factors play distinct token-side
and embedding-side roles, has not been examined.
 
A simple diagnostic suggests that it does. When we train factorized prompts
independently and compare their factors across datasets, shot counts, and ranks, the
token-side factor $\bm{B}$ varies little and occupies a nearly common subspace,
whereas the embedding-side factor $\bm{A}$ and the induced prompt $\bm{B}\bm{A}$
change substantially (Fig.~\ref{fig:factor_geometry}). The token basis is largely
shared across tasks while the coefficients carry the adaptation. This sharpens our
question: must $\bm{B}$ be learned at all? If not, prompt adaptation reduces to
optimizing the $rd$ embedding-side coefficients over a fixed low-dimensional token
basis.
 
Our experiments indicate that learning $\bm{B}$ is often
unnecessary in the evaluated CoOp settings. Across seven
few-shot benchmarks and two CLIP backbones, low-rank prompts match or improve dense
CoOp while using far fewer parameters, with the clearest gains on base-to-new
generalization in the low-shot regime. Fixing $\bm{B}$ to a Gaussian, orthogonal,
SVD-derived, or learned-then-frozen basis and training only $\bm{A}$ stays within a
few tenths of a point of the fully trainable factorization, and transfer controls
show a random fixed $\bm{B}$ performs as well as one carried over from a source
dataset, so the token basis need not encode anything task-specific. Theoretically, a
prompt-factor asymmetry analysis shows that fixing $\bm{B}$ restricts the prompt by a
projection in the small $m$-dimensional token space, whereas fixing $\bm{A}$
restricts it in the much larger $d$-dimensional embedding space; a local update-space
argument gives fixing $\bm{B}$ exactly $d/m$ ($=32$ in our setting) times more
adaptation directions than fixing $\bm{A}$; and, under the stated smoothness and lower-boundedness assumptions, we establish a bound on the minimum squared
gradient norm attained by gradient descent over a fixed
$\bm{B}$.
 
We summarize our contributions as follows:
\begin{itemize}
\item We introduce a low-rank factorized formulation of continuous prompt learning
for CLIP, replacing the dense CoOp prompt with $\bm{P}=\bm{B}\bm{A}$, and show that it
matches or improves dense CoOp across seven few-shot benchmarks while using
substantially fewer trainable parameters, with the largest gains on low-shot
base-to-new generalization.
\item Through fixed-basis ablations and transfer controls, we demonstrate that the
token-side factor need not be learned: Gaussian, orthogonal, SVD-derived, and even
random bases match the fully trainable factorization, and a source-trained $\bm{B}$
offers no advantage over a random one.
\item We provide a theoretical account, comprising a prompt-factor asymmetry result,
a local update-space dimension gap, and a smoothness-only convergence guarantee, that
explains why optimizing $\bm{A}$ over a fixed $\bm{B}$ remains effective.
\end{itemize}

\section{Related Work}
\textbf{Vision-Language Models.}
Pre-trained vision-language models (VLMs)~\cite{radford2021learning, jia2021scaling, yao2021filip, yuan2021florence, li2022fine}
have recently attracted significant attention for their strong performance across diverse multi-modal tasks. These models are pre-trained in a self-supervised manner on large-scale image--text pairs collected from the web; CLIP~\cite{radford2021learning},
ALIGN~\cite{jia2021scaling}, LiT~\cite{zhai2022lit}, and FILIP~\cite{yao2021filip}, for instance, leverage hundreds of
millions to billions of pairs, optimizing a contrastive objective that pulls paired image--text features together while pushing unpaired ones apart~\cite{radford2021learning}.
The resulting representations capture rich open-vocabulary concepts, making them effective for a broad range of downstream applications~\cite{gao2024clip,zhang2022tip, bangalath2022bridging, maaz2022class,zhou2022detecting,gu2021open} such as image classification~\cite{zhang2022tip,gao2024clip,zhou2022detecting}, object
detection~\cite{li2024learning,maaz2022class,zhou2022detecting,gu2021open,feng2022promptdet}, and semantic segmentation~\cite{li2024omg,rao2022denseclip,li2024transformer,luddecke2022image}. Nonetheless, adapting these foundational models to downstream tasks without eroding their original
generalization ability remains a key challenge~\cite{khattak2023maple}. \textit{In this work, we focus on the parameter-efficient adaptation of VLMs via prompt learning.}

\textbf{Prompt Learning.}
Prompt learning, originally introduced in NLP~\cite{jiang2020can,shin2020autoprompt}, has become a dominant paradigm for adapting VLMs to downstream tasks without modifying the pretrained backbone~\cite{jiang2020can}. These methods broadly fall into three categories: textual prompt learning~\cite{bulat2023lasp,zhou2022learning,yao2023visual,lu2022prompt} that optimizes continuous prompt vectors in CLIP's language branch; visual prompt learning~\cite{bahng2022exploring,jia2022visual,wang2022learning} that introduces learnable tokens into the visual input space while keeping the
backbone frozen; and multi-modal prompt learning~\cite{khattak2023maple,ghiasvand2025pfedmma,lee2023multimodal, ghiasvand2026zomp} that applies prompts to both branches for stronger cross-modal alignment. As the
canonical text-only method, CoOp~\cite{zhou2022learning} optimizes a continuous context matrix in
the language branch for few-shot recognition, while CoCoOp~\cite{cocoop} conditions the prompt on image features to mitigate overfitting. Subsequent deep multi-modal variants such as MaPLe~\cite{khattak2023maple} further improve performance by learning prompts
jointly across both encoders, though at a higher parameter cost. \textit{Orthogonally to
adding modalities, we instead ask whether the text-side prompt is itself low-rank. We study the effect of low-rank prompts in the CoOp-style text-only setting, factorizing the prompt as $\bm{P} = \bm{B}\bm{A}$ and asking which of the two factors must actually be learned.}

\textbf{Low-Rank Factorization and Factor Asymmetry.}
Inspired by LoRA~\cite{hu2021lora}, recent work brings low-rank adaptation to VLMs in two forms. Weight-space methods inject low-rank adapters into the frozen backbone. For example, CLIP-LoRA~\cite{zanella2024low} adapts the attention projections of both CLIP
encoders, and Block-LoRA~\cite{zhou2025one} shares down-projection blocks across layers to
reduce adapter redundancy, though both update the pretrained weights. Prompt-space methods instead leave the backbone untouched and impose low-rank structure on the learnable prompt itself, as in DIP~\cite{hao2023towards} and MMLoP~\cite{ghiasvand2026mmlop}, which factorizes deep prompts in both encoders. These methods all train both low-rank factors. A parallel
line of work on LoRA shows that two factors are not interchangeable: \citet{zhu2024asymmetry} find it preferable to freeze the input-side
projector, even when random, and tune only the output-side factor. \textit{Whether
this asymmetry carries over to prompt factorization, where the factors play
distinct token-side and embedding-side roles, remains open. We analyze this question both theoretically and empirically.}

\section{Preliminaries}
\label{sec:prelim}
\textbf{CLIP.} We denote CLIP's image and text encoders by $f$ and $g$, with frozen pretrained
parameters $\theta_{\text{CLIP}} = \{\theta_f, \theta_g\}$. An input image
$\bm{X} \in \mathbb{R}^{C \times H \times W}$ is split into $M$ patches and linearly
projected into patch embeddings; together with a class token $\bm{e}_{\text{cls}}$
these form the encoder input
$\tilde{\bm{X}} = \{\bm{e}_{\text{cls}}, \bm{e}_1, \dots, \bm{e}_M\}$, which $f$ maps
to a visual feature $\tilde{\bm{f}} = f(\tilde{\bm{X}}; \theta_f) \in \mathbb{R}^d$.
On the text side, each class name $c_k$ is inserted into a hand-crafted template such
as ``a photo of a \{class\}'' and tokenized as
$\tilde{\bm{Y}}_k = \{\bm{t}_{\text{SOS}}, \bm{t}_1, \dots, \bm{t}_N, c_k,
\bm{t}_{\text{EOS}}\}$, where $\{\bm{t}_n\}_{n=1}^N$ are the template word embeddings
and $\bm{t}_{\text{SOS}}, \bm{t}_{\text{EOS}}$ mark the start and end of the sequence.
The text encoder produces a class feature
$\tilde{\bm{g}}_k = g(\tilde{\bm{Y}}_k; \theta_g) \in \mathbb{R}^d$. Zero-shot
prediction then scores the image against all $C$ class descriptions through cosine
similarity:
\begin{equation}
    p(y = k \mid \bm{X}) = \frac{\exp\!\big(\text{sim}(\tilde{\bm{g}}_k, \tilde{\bm{f}}) / \tau\big)}
    {\sum_{i=1}^{C} \exp\!\big(\text{sim}(\tilde{\bm{g}}_i, \tilde{\bm{f}}) / \tau\big)},
\end{equation}
where $\text{sim}(\cdot, \cdot)$ is cosine similarity and $\tau$ is a temperature.

\textbf{Prompt Learning with CoOp.}
Rather than relying on a fixed template, CoOp~\cite{zhou2022learning} replaces the
template word tokens with a set of learnable continuous context vectors that are
optimized on a few labeled examples. Concretely, the $N$ fixed tokens are substituted
by $m$ learnable vectors $\{\bm{p}_1, \dots, \bm{p}_m\}$ with
$\bm{p}_i \in \mathbb{R}^d$, which we stack into a single prompt matrix
\begin{equation}
    \bm{P} = [\bm{p}_1; \bm{p}_2; \dots; \bm{p}_m] \in \mathbb{R}^{m \times d},
\end{equation}
where $m = n_{\text{ctx}}$ is the number of context tokens and $d$ is the CLIP
text-embedding dimension ($d = 512$). Following the unified-context setting, $\bm{P}$
is shared across all classes. The prompted text input for class $k$ becomes
$\tilde{\bm{Y}}_p^{k} = \{\bm{t}_{\text{SOS}}, \bm{p}_1, \dots, \bm{p}_m, c_k,
\bm{t}_{\text{EOS}}\}$, yielding the prompted class feature
$\tilde{\bm{g}}_p^{k} = g(\tilde{\bm{Y}}_p^{k}; \theta_g)$. Only $\bm{P}$ is trained,
by minimizing the cross-entropy classification loss over the few-shot training set
$\mathcal{D}$ while the encoders $\theta_{\text{CLIP}}$ remain frozen:
\begin{equation}
    \bm{P}^{\star} = \mathop{\arg\min}_{\bm{P}} \;
    \mathbb{E}_{(\bm{X}, y) \sim \mathcal{D}} \;
    \mathcal{L}\big(\text{sim}(\tilde{\bm{f}}, \tilde{\bm{g}}_p), \, y\big).
    \label{eq:coop-ce}
\end{equation}
The prompt matrix $\bm{P}$ is thus the sole object of adaptation, and is the
quantity we factorize in the next section.

\section{Low-Rank Prompt Learning}
\label{sec:method}
\subsection{Factorized Prompt Parameterization}
\label{sec:factorized}

 We replace the dense CoOp prompt
$\bm{P}\in\R^{m\times d}$ of Section~\ref{sec:prelim} with the low-rank product
\begin{equation}
    \bm{P} = \bm{B}\bm{A}, \qquad
    \bm{B}\in\R^{m\times r}, \;\;
    \bm{A}\in\R^{r\times d}, \;\; r\ll d,
    \label{eq:factorization_method}
\end{equation}
where the \emph{token-side} factor $\bm{B}$ mixes $r$ latent components across the
$m=n_{\rm ctx}$ context-token positions and the \emph{embedding-side} factor
$\bm{A}$ selects directions in the CLIP text-embedding space. Beyond reducing the
parameter count, this separation lets us ask not just whether a low-rank prompt is
expressive enough, but which of the two factors must actually be learned. We
therefore study three training regimes: jointly training both factors, freezing
$\bm{B}$ and training $\bm{A}$, and freezing $\bm{A}$ and training $\bm{B}$. The last
regime serves as a control that tests whether the two factors play interchangeable
roles. The factorization steadily reduces the trainable prompt budget,
\begin{equation}
    \underbrace{md}_{\text{dense CoOp}} \;\longrightarrow\;
    \underbrace{r(m+d)}_{\text{train }\bm{B},\bm{A}} \;\longrightarrow\;
    \underbrace{rd}_{\text{fixed }\bm{B},\ \text{train }\bm{A}},
    \label{eq:param_counts}
\end{equation}
with $mr$ parameters in the fixed-$\bm{A}$ control; for $m=16$, $d=512$, and $r=4$
this is $8{,}192 \to 2{,}112 \to 2{,}048$ parameters, and $64$ for fixed $\bm{A}$.

The fixed-$\bm{B}$ regime optimizes only $\bm{A}$ over a
non-trainable token basis. We consider Gaussian,
scaled-orthogonal, SVD-derived, and previously learned
same-task bases. In the fixed-basis ablations, we initialize
the trainable factor by
\[
\bm{A}_0=\bm{B}_{\rm fixed}^{\dagger}\bm{P}_0,
\]
where $\bm{P}_0$ is a sampled dense prompt and
$\bm{B}^{\dagger}$ denotes the Moore--Penrose pseudoinverse.
The initial product is therefore the least-squares
projection of $\bm{P}_0$ onto the selected basis.

For joint factorization, we use the balanced truncated
SVD initialization
$\bm{B}_0=\bm{U}_r\bm{\Sigma}_r^{1/2}$ and
$\bm{A}_0=\bm{\Sigma}_r^{1/2}\bm{V}_r^\top$.
These procedures do not generally produce identical
initial prompt products across variants. Transfer runs
use a separate initialization procedure described in
Appendix~\ref{app:protocol}. In fixed-$\bm{B}$ training,
only $\bm{A}$ is optimized under the cross-entropy
objective of Eq.~\eqref{eq:coop-ce}.

Two analyses support this regime. A prompt-factor asymmetry result shows that fixing
$\bm{B}$ is far less restrictive than fixing $\bm{A}$, which is what makes the
token-side factor, rather than the embedding-side factor, the one to freeze. Separately, 
under the stated smoothness and lower-boundedness
assumptions, we establish a bound on the minimum squared
gradient norm attained by gradient descent over a fixed
$\bm{B}$.; this certifies that the procedure we run is well-posed,
but it is symmetric in the two factors and does not by itself favor fixing one over
the other. Expanded statements and proofs appear in
Appendix~\ref{app:theory}.

\subsection{Prompt-Factor Asymmetry}
\label{sec:asymmetry}

 To understand why the two factors are not interchangeable, we analyze a simplified
prompt-space approximation problem that isolates the geometry of the factorized
parameterization from the nonlinear CLIP text encoder.

\begin{assumption}[Linear prompt-space surrogate]
\label{asm:surrogate}
There exists an ideal prompt matrix $\bm{P}^{\star}\in\R^{m\times d}$ such that local
prompt adaptation can be approximated by the least-squares problem
$\min_{\bm{A},\bm{B}}\lVert \bm{P}^{\star}-\bm{B}\bm{A}\rVert_F^2$.
\end{assumption}

Under Assumption~\ref{asm:surrogate}, each one-sided regime reduces to an orthogonal
projection of $\bm{P}^{\star}$. Fixing $\bm{B}=\bm{U}$ with orthonormal columns leaves
the residual $\lVert(\bm{I}_m-\bm{U}\bm{U}^{\top})\bm{P}^{\star}\rVert_F^2$, a
projection in the $m$-dimensional token-position space, whereas fixing
$\bm{A}=\bm{Q}$ with orthonormal rows leaves
$\lVert\bm{P}^{\star}(\bm{I}_d-\bm{Q}^{\top}\bm{Q})\rVert_F^2$, a projection in the
$d$-dimensional embedding space (proved in the supplement). Since typically
$d\gg m$, the latter restriction is far more severe. Taking expectations over random
bases makes this precise.

\begin{theorem}[Expected prompt-factor asymmetry]
\label{thm:asymmetry}
Let $\bm{P}^{\star}\in\R^{m\times d}$ be nonzero and $1\le r\le m<d$. Let
$\bm{U}\sim\mathrm{Unif}(\St(m,r))$ be Haar-uniform with orthonormal columns and
$\bm{Q}^{\top}\sim\mathrm{Unif}(\St(d,r))$, so $\bm{Q}\in\R^{r\times d}$ has
orthonormal rows. Then
\begin{align}
\E_{\bm{U}}\min_{\bm{A}}\lVert\bm{P}^{\star}-\bm{U}\bm{A}\rVert_F^2
  &= \Big(1-\tfrac{r}{m}\Big)\lVert\bm{P}^{\star}\rVert_F^2, \nonumber\\[2pt]
\E_{\bm{Q}}\min_{\bm{B}}\lVert\bm{P}^{\star}-\bm{B}\bm{Q}\rVert_F^2
  &= \Big(1-\tfrac{r}{d}\Big)\lVert\bm{P}^{\star}\rVert_F^2 .
\end{align}
Consequently, since $m<d$ and $\bm{P}^{\star}\neq\bm{0}$,
\begin{equation}
\E_{\bm{U}}\min_{\bm{A}}\lVert\bm{P}^{\star}-\bm{U}\bm{A}\rVert_F^2
\;<\;
\E_{\bm{Q}}\min_{\bm{B}}\lVert\bm{P}^{\star}-\bm{B}\bm{Q}\rVert_F^2 .
\end{equation}
\end{theorem}

In words, under the linear surrogate, fixing the token-side factor $\bm{B}$ and
optimizing $\bm{A}$ has strictly smaller expected approximation error than fixing the
embedding-side factor $\bm{A}$ and optimizing $\bm{B}$. The gap is governed entirely
by the dimensions the two factors live in.

\begin{corollary}
\label{cor:asymmetry}
In our setting $m=16$ and $d=512$, so for any fixed rank $r\le 16$ the residual ratio
is $(1-r/16)/(1-r/512)$. A random fixed token-side basis $\bm{B}$ removes a fraction
$r/16$ of the ideal prompt energy in expectation, whereas a random fixed
embedding-side factor $\bm{A}$ removes only $r/512$. Fixing $\bm{B}$ and training
$\bm{A}$ therefore preserves a much larger fraction of the ideal prompt than the
reverse.
\end{corollary}

\subsection{Local Update-Space Dimension Gap}
\label{sec:dimgap}

The linear surrogate captures the approximation geometry, but the CoOp objective is
not the Frobenius problem $\min_{\bm{P}}\norm{\bm{P}-\bm{P}^\star}_F^2$; it is a
classification loss $\mathcal{L}(\bm{P})$ obtained after passing $\bm{P}$ through the
frozen text encoder. We now connect the two. Writing $\bm{p}=\vecop(\bm{P})$ and
expanding $\mathcal{L}$ to second order around a reference prompt
$\bm{p}_0=\vecop(\bm{P}_0)$ reduces local prompt learning to projecting an ideal
update $\bm{\delta}^\star$ onto the subspace $\mathcal{S}$ of updates admitted by the
parameterization, under the metric induced by a curvature matrix $\bm{H}\succeq 0$.
The two one-sided regimes admit different subspaces, and their dimensions differ
sharply.

\begin{assumption}[Factorized reference prompt]
\label{assump:factorized_reference_prompt}
The reference prompt admits a rank-$r$ factorization $\bm{P}_0=\bm{B}_0\bm{A}_0$ with
$\bm{B}_0\in\R^{m\times r}$ of full column rank $r$ and $\bm{A}_0\in\R^{r\times d}$ of
full row rank $r$.
\end{assumption}

\begin{proposition}[Local update spaces]
\label{prop:update_spaces}
Under Assumption~\ref{assump:factorized_reference_prompt}, fixing $\bm{B}_0$ and
writing $\bm{A}=\bm{A}_0+\Delta\bm{A}$ gives $\Delta\bm{P}=\bm{B}_0\Delta\bm{A}$, so
the update space is
$\mathcal{S}_{\bm{A}}(\bm{B}_0)=\Col(\bm{I}_d\otimes\bm{B}_0)$ with
$\dim\mathcal{S}_{\bm{A}}(\bm{B}_0)=rd$. Fixing $\bm{A}_0$ and writing
$\bm{B}=\bm{B}_0+\Delta\bm{B}$ gives $\Delta\bm{P}=\Delta\bm{B}\,\bm{A}_0$, so the
update space is
$\mathcal{S}_{\bm{B}}(\bm{A}_0)=\Col(\bm{A}_0^\top\otimes\bm{I}_m)$ with
$\dim\mathcal{S}_{\bm{B}}(\bm{A}_0)=mr$.
\end{proposition}

\begin{corollary}[Local dimension gap]
\label{cor:local_dimension_gap}
Under Assumption~\ref{assump:factorized_reference_prompt},
$\dim\mathcal{S}_{\bm{A}}(\bm{B}_0)/\dim\mathcal{S}_{\bm{B}}(\bm{A}_0)=rd/(mr)=d/m$,
which equals $32$ for $m=16$ and $d=512$. Fixing $\bm{B}$ and training $\bm{A}$ thus
provides $32$ times more local update directions than fixing $\bm{A}$ and training
$\bm{B}$.
\end{corollary}

When $\bm{H}$ is isotropic, the local objective
$\min_{\bm{\delta}\in\mathcal{S}}\norm{\bm{\delta}-\bm{\delta}^\star}_{\bm{H}}^2$
reduces to Euclidean projection and recovers the Frobenius surrogate of
Theorem~\ref{thm:asymmetry} exactly. The asymmetry result and the dimension gap are
thus two views of the same effect: the embedding-side factor $\bm{A}$ both preserves
more of the ideal prompt and offers more directions to adapt it, which is why
$\bm{A}$ carries most of the useful adaptation while $\bm{B}$ can be fixed.

\subsection{Convergence Under a Fixed Factor}
\label{sec:convergence}

The asymmetry and dimension-gap results say \emph{which} factor to fix. We now give a
separate optimization guarantee that the resulting procedure is well-posed: once one
factor is fixed, gradient-based optimization of the other converges to a first-order
stationary point under directional smoothness alone, with neither convexity nor
strong convexity required. We state it for the fixed-$\bm{B}$ regime we use, but the
argument is symmetric in the two factors and yields the same guarantee for fixed
$\bm{A}$.

Let $F:\R^{m\times d}\to\R$ be the prompt objective as a function of
$\bm{P}=\bm{B}\bm{A}$, and for a fixed $\bm{B}$ define the restricted objective
$\phi_{\bm{B}}(\bm{A})=F(\bm{B}\bm{A})$.

\begin{assumption}[$\bm{A}$-directional smoothness]
\label{assump:A_directional_smoothness}
For every fixed $\bm{B}\in\R^{m\times r}$, the map $\bm{A}\mapsto\phi_{\bm{B}}(\bm{A})$
is $L_A$-smooth in the Frobenius norm:
$\norm{\nabla_{\bm{A}}\phi_{\bm{B}}(\bm{A}_1)-\nabla_{\bm{A}}\phi_{\bm{B}}(\bm{A}_2)}_F
\le L_A\norm{\bm{A}_1-\bm{A}_2}_F$ for all $\bm{A}_1,\bm{A}_2$.
\end{assumption}

\begin{assumption}[Lower boundedness]
\label{assump:lower_bounded_phi_B}
For every fixed $\bm{B}$, there exists $\phi_{\bm{B}}^{\inf}>-\infty$ with
$\phi_{\bm{B}}(\bm{A})\ge\phi_{\bm{B}}^{\inf}$ for all $\bm{A}\in\R^{r\times d}$.
\end{assumption}

\begin{theorem}[Full-gradient convergence]
\label{thm:fixed_B_gd_convergence_directional}
Under Assumptions~\ref{assump:A_directional_smoothness}
and~\ref{assump:lower_bounded_phi_B}, gradient descent
$\bm{A}_{t+1}=\bm{A}_t-\eta\nabla_{\bm{A}}\phi_{\bm{B}}(\bm{A}_t)$ with
$0<\eta\le 1/L_A$ satisfies the descent inequality
$\phi_{\bm{B}}(\bm{A}_{t+1})\le\phi_{\bm{B}}(\bm{A}_t)
-\tfrac{\eta}{2}\norm{\nabla_{\bm{A}}\phi_{\bm{B}}(\bm{A}_t)}_F^2$
and, after $T$ iterations,
\begin{equation}
\label{eq:fixed_B_gd_rate}
\min_{0\le t<T}
\norm{\nabla_{\bm{A}}\phi_{\bm{B}}(\bm{A}_t)}_F^2
\le
\frac{2\big(\phi_{\bm{B}}(\bm{A}_0)
-\phi_{\bm{B}}^{\inf}\big)}{\eta T}.
\end{equation}
Choosing $\eta=1/L_A$ gives an $O(1/T)$ rate.
\end{theorem}

\begin{theorem}[Stochastic-gradient convergence]
\label{thm:fixed_B_sgd_convergence_directional}
If in addition the stochastic gradients $\bm{G}_t$ are unbiased with variance at most
$\sigma^2$, then stochastic gradient descent
$\bm{A}_{t+1}=\bm{A}_t-\eta\bm{G}_t$ with $0<\eta\le 1/L_A$ satisfies
\begin{equation}
\label{eq:fixed_B_sgd_rate}
\min_{0\le t<T}
\E\big[\norm{\nabla_{\bm{A}}\phi_{\bm{B}}(\bm{A}_t)}_F^2\big]
\le
\frac{2\big(\phi_{\bm{B}}(\bm{A}_0)
-\phi_{\bm{B}}^{\inf}\big)}{\eta T}
+ L_A\eta\sigma^2.
\end{equation}
Choosing $\eta=\Theta(T^{-1/2})$ yields an $O(T^{-1/2})$ rate.
\end{theorem}

While the guarantee itself is neutral between the factors, the fixed-$\bm{B}$ regime
admits a clean control on its constant: when $F$ is $L$-smooth, the directional
constant is bounded by $L_A\le L\norm{\bm{B}}_2^2$, so the scale of the fixed basis
$\bm{B}$ sets the conditioning of the optimization of $\bm{A}$. An
orthonormal $\bm{B}$ ($\norm{\bm{B}}_2=1$) is the cleanest case, which motivates
norm-matching the fixed bases in our experiments.

\section{Experimental Setup}
\textbf{Datasets and backbones.}
We evaluate on seven open-source image-classification datasets from the standard
CLIP~\cite{radford2021learning} and CoOp~\cite{zhou2022learning} benchmark suite:
Caltech101~\cite{feifeicaltech101}, DTD~\cite{cimpoidtd}, FGVC
Aircraft~\cite{majifgvcaircraft}, Food101~\cite{bossardfood101}, Oxford
Flowers~\cite{nilsbackoxfordflowers}, Oxford Pets~\cite{parkhioxfordpets}, and
UCF101~\cite{soomroucf101}. Together they span object, scene, texture, and action
recognition, covering a diverse range of downstream tasks. For the image encoder we
use both the CLIP ViT-B/16 and RN50 backbones. Following CoOp, the class name is placed after the learned context tokens, and class-specific context is disabled.


\begin{table}[t]
\centering
\scriptsize
\setlength{\tabcolsep}{2.6pt}
\renewcommand{\arraystretch}{1.12}
\caption{Base-to-new generalization: factorized CoOp ($r{=}4$) vs.\ dense CoOp.
Means over three seeds; subscripts are per-seed standard deviations.
\textbf{S}/\textbf{U}/\textbf{H} denote seen, unseen, and their harmonic mean, and
bold marks the better mean in each column. Average rows omit $\pm$, since the spread
there reflects between-dataset heterogeneity rather than run-to-run variance.}
\label{tab:base2new}
\begin{adjustbox}{max width=\linewidth}
\begin{tabular}{ll *{9}{c}}
\toprule
 & & \multicolumn{3}{c}{1-shot} & \multicolumn{3}{c}{4-shot} & \multicolumn{3}{c}{16-shot} \\
\cmidrule(lr){3-5}\cmidrule(lr){6-8}\cmidrule(lr){9-11}
Dataset & Method & S & U & H & S & U & H & S & U & H \\
\midrule
\multicolumn{11}{l}{\textit{RN50}}\\
\midrule
\multirow{2}{*}{Caltech101} & Dense & 91.33$_{\pm0.81}$ & 86.27$_{\pm0.87}$ & 88.73$_{\pm0.83}$ & 93.57$_{\pm0.15}$ & 83.70$_{\pm0.85}$ & 88.36$_{\pm0.50}$ & 95.37$_{\pm0.15}$ & \textbf{84.83}$_{\pm3.97}$ & \textbf{89.76}$_{\pm2.20}$ \\
 & Fact. ($r{=}4$) & \textbf{91.70}$_{\pm1.21}$ & \textbf{89.33}$_{\pm0.86}$ & \textbf{90.50}$_{\pm0.68}$ & \textbf{93.80}$_{\pm0.85}$ & \textbf{87.60}$_{\pm2.16}$ & \textbf{90.58}$_{\pm1.16}$ & \textbf{95.40}$_{\pm0.52}$ & 84.10$_{\pm4.26}$ & 89.35$_{\pm2.14}$ \\
\addlinespace[1pt]
\multirow{2}{*}{DTD} & Dense & \textbf{51.00}$_{\pm2.26}$ & 44.33$_{\pm4.27}$ & \textbf{47.31}$_{\pm2.22}$ & \textbf{65.50}$_{\pm0.69}$ & \textbf{42.93}$_{\pm4.23}$ & \textbf{51.80}$_{\pm3.13}$ & 74.27$_{\pm0.57}$ & 34.77$_{\pm1.94}$ & 47.34$_{\pm1.75}$ \\
 & Fact. ($r{=}4$) & 49.10$_{\pm3.02}$ & \textbf{44.77}$_{\pm4.27}$ & 46.70$_{\pm2.16}$ & 65.23$_{\pm1.03}$ & 42.90$_{\pm2.14}$ & 51.74$_{\pm1.61}$ & \textbf{75.33}$_{\pm0.57}$ & \textbf{36.03}$_{\pm2.25}$ & \textbf{48.72}$_{\pm2.10}$ \\
\addlinespace[1pt]
\multirow{2}{*}{FGVC Aircraft} & Dense & 9.53$_{\pm7.07}$ & 5.90$_{\pm6.85}$ & 7.04$_{\pm7.25}$ & 21.40$_{\pm1.84}$ & 9.03$_{\pm4.50}$ & 12.42$_{\pm4.66}$ & \textbf{31.40}$_{\pm0.46}$ & \textbf{12.73}$_{\pm1.53}$ & \textbf{18.08}$_{\pm1.51}$ \\
 & Fact. ($r{=}4$) & \textbf{18.87}$_{\pm0.81}$ & \textbf{18.57}$_{\pm0.97}$ & \textbf{18.71}$_{\pm0.84}$ & \textbf{22.13}$_{\pm1.10}$ & \textbf{19.17}$_{\pm1.36}$ & \textbf{20.51}$_{\pm0.80}$ & 28.30$_{\pm1.14}$ & 11.53$_{\pm3.61}$ & 16.20$_{\pm3.58}$ \\
\addlinespace[1pt]
\multirow{2}{*}{Food101} & Dense & 71.73$_{\pm0.31}$ & 74.30$_{\pm2.21}$ & 72.98$_{\pm1.07}$ & 74.67$_{\pm2.07}$ & 73.10$_{\pm2.71}$ & 73.87$_{\pm2.38}$ & 79.57$_{\pm0.71}$ & 74.53$_{\pm1.57}$ & 76.97$_{\pm1.14}$ \\
 & Fact. ($r{=}4$) & \textbf{77.87}$_{\pm1.17}$ & \textbf{77.33}$_{\pm4.42}$ & \textbf{77.58}$_{\pm2.79}$ & \textbf{76.30}$_{\pm1.01}$ & \textbf{74.67}$_{\pm0.84}$ & \textbf{75.47}$_{\pm0.86}$ & \textbf{81.70}$_{\pm0.66}$ & \textbf{79.67}$_{\pm0.55}$ & \textbf{80.67}$_{\pm0.22}$ \\
\addlinespace[1pt]
\multirow{2}{*}{Oxford Flowers} & Dense & 75.37$_{\pm2.14}$ & 62.90$_{\pm4.26}$ & 68.55$_{\pm3.29}$ & 88.60$_{\pm1.48}$ & 54.07$_{\pm2.20}$ & 67.12$_{\pm1.43}$ & \textbf{96.30}$_{\pm0.46}$ & 54.17$_{\pm2.97}$ & 69.30$_{\pm2.45}$ \\
 & Fact. ($r{=}4$) & \textbf{77.43}$_{\pm1.59}$ & \textbf{65.67}$_{\pm0.97}$ & \textbf{71.06}$_{\pm0.81}$ & \textbf{88.90}$_{\pm1.40}$ & \textbf{61.20}$_{\pm1.65}$ & \textbf{72.49}$_{\pm1.46}$ & 95.73$_{\pm0.21}$ & \textbf{59.70}$_{\pm1.42}$ & \textbf{73.53}$_{\pm1.14}$ \\
\addlinespace[1pt]
\multirow{2}{*}{Oxford Pets} & Dense & 80.03$_{\pm3.69}$ & 91.00$_{\pm3.50}$ & 85.09$_{\pm1.91}$ & 87.20$_{\pm1.57}$ & 88.40$_{\pm3.60}$ & 87.76$_{\pm1.83}$ & 89.20$_{\pm0.30}$ & 85.03$_{\pm6.51}$ & 86.99$_{\pm3.51}$ \\
 & Fact. ($r{=}4$) & \textbf{84.87}$_{\pm2.87}$ & \textbf{93.53}$_{\pm0.95}$ & \textbf{88.97}$_{\pm1.77}$ & \textbf{88.03}$_{\pm4.22}$ & \textbf{93.37}$_{\pm0.32}$ & \textbf{90.59}$_{\pm2.22}$ & \textbf{89.83}$_{\pm0.81}$ & \textbf{92.83}$_{\pm1.66}$ & \textbf{91.30}$_{\pm0.95}$ \\
\addlinespace[1pt]
\multirow{2}{*}{UCF101} & Dense & 64.00$_{\pm3.97}$ & 56.97$_{\pm2.91}$ & 60.25$_{\pm2.94}$ & \textbf{72.93}$_{\pm1.50}$ & 52.50$_{\pm4.11}$ & 60.98$_{\pm2.71}$ & \textbf{80.43}$_{\pm0.64}$ & 46.27$_{\pm2.74}$ & 58.72$_{\pm2.30}$ \\
 & Fact. ($r{=}4$) & \textbf{68.83}$_{\pm2.11}$ & \textbf{63.20}$_{\pm2.78}$ & \textbf{65.89}$_{\pm2.46}$ & 72.90$_{\pm2.07}$ & \textbf{54.00}$_{\pm5.01}$ & \textbf{62.00}$_{\pm4.08}$ & 80.23$_{\pm1.10}$ & \textbf{50.33}$_{\pm5.80}$ & \textbf{61.72}$_{\pm4.22}$ \\
\cmidrule(lr){2-11}
\multirow{2}{*}{\textit{Average}} & Dense & 63.28 & 60.24 & 61.42 & 71.98 & 57.68 & 63.19 & \textbf{78.08} & 56.05 & 63.88 \\
 & Fact. ($r{=}4$) & \textbf{66.95} & \textbf{64.63} & \textbf{65.63} & \textbf{72.47} & \textbf{61.84} & \textbf{66.20} & 78.07 & \textbf{59.17} & \textbf{65.93} \\
\midrule
\multicolumn{11}{l}{\textit{ViT-B/16}}\\
\midrule
\multirow{2}{*}{Caltech101} & Dense & 96.03$_{\pm0.42}$ & 92.43$_{\pm0.38}$ & 94.20$_{\pm0.37}$ & 96.90$_{\pm0.20}$ & 90.67$_{\pm3.00}$ & 93.66$_{\pm1.54}$ & 97.93$_{\pm0.40}$ & 89.27$_{\pm3.23}$ & 93.38$_{\pm1.69}$ \\
 & Fact. ($r{=}4$) & \textbf{96.87}$_{\pm0.57}$ & \textbf{92.83}$_{\pm1.46}$ & \textbf{94.81}$_{\pm1.01}$ & \textbf{97.40}$_{\pm0.36}$ & \textbf{93.60}$_{\pm1.31}$ & \textbf{95.46}$_{\pm0.73}$ & \textbf{98.00}$_{\pm0.17}$ & \textbf{89.83}$_{\pm0.93}$ & \textbf{93.74}$_{\pm0.49}$ \\
\addlinespace[1pt]
\multirow{2}{*}{DTD} & Dense & 58.73$_{\pm3.26}$ & 46.60$_{\pm4.01}$ & 51.80$_{\pm1.47}$ & \textbf{71.27}$_{\pm1.00}$ & 47.33$_{\pm3.40}$ & 56.85$_{\pm2.78}$ & \textbf{80.97}$_{\pm1.87}$ & 44.63$_{\pm7.08}$ & 57.39$_{\pm6.23}$ \\
 & Fact. ($r{=}4$) & \textbf{59.50}$_{\pm3.72}$ & \textbf{53.33}$_{\pm3.98}$ & \textbf{56.08}$_{\pm0.84}$ & 71.00$_{\pm1.51}$ & \textbf{49.87}$_{\pm3.31}$ & \textbf{58.57}$_{\pm2.80}$ & 80.50$_{\pm0.36}$ & \textbf{47.13}$_{\pm3.47}$ & \textbf{59.41}$_{\pm2.73}$ \\
\addlinespace[1pt]
\multirow{2}{*}{FGVC Aircraft} & Dense & 26.97$_{\pm0.55}$ & 23.67$_{\pm5.70}$ & 24.95$_{\pm3.51}$ & 31.27$_{\pm1.81}$ & 15.00$_{\pm9.27}$ & 19.20$_{\pm9.75}$ & \textbf{41.97}$_{\pm1.27}$ & 22.50$_{\pm1.87}$ & 29.26$_{\pm1.64}$ \\
 & Fact. ($r{=}4$) & \textbf{29.47}$_{\pm1.10}$ & \textbf{27.97}$_{\pm0.59}$ & \textbf{28.69}$_{\pm0.79}$ & \textbf{32.27}$_{\pm0.65}$ & \textbf{27.47}$_{\pm0.35}$ & \textbf{29.67}$_{\pm0.48}$ & 40.53$_{\pm0.86}$ & \textbf{25.10}$_{\pm4.61}$ & \textbf{30.80}$_{\pm3.12}$ \\
\addlinespace[1pt]
\multirow{2}{*}{Food101} & Dense & 84.90$_{\pm0.53}$ & 87.00$_{\pm1.51}$ & 85.93$_{\pm0.76}$ & 86.17$_{\pm0.42}$ & 85.13$_{\pm1.74}$ & 85.64$_{\pm0.67}$ & 87.90$_{\pm0.17}$ & 83.63$_{\pm3.02}$ & 85.70$_{\pm1.62}$ \\
 & Fact. ($r{=}4$) & \textbf{87.67}$_{\pm1.25}$ & \textbf{87.63}$_{\pm2.39}$ & \textbf{87.65}$_{\pm1.74}$ & \textbf{87.63}$_{\pm0.81}$ & \textbf{87.23}$_{\pm2.11}$ & \textbf{87.43}$_{\pm1.46}$ & \textbf{89.37}$_{\pm0.38}$ & \textbf{87.13}$_{\pm1.48}$ & \textbf{88.23}$_{\pm0.84}$ \\
\addlinespace[1pt]
\multirow{2}{*}{Oxford Flowers} & Dense & 83.27$_{\pm2.99}$ & 64.97$_{\pm1.26}$ & 72.98$_{\pm1.77}$ & \textbf{92.10}$_{\pm1.49}$ & 60.63$_{\pm3.43}$ & 73.08$_{\pm2.37}$ & \textbf{97.70}$_{\pm0.53}$ & 56.97$_{\pm3.87}$ & 71.92$_{\pm3.09}$ \\
 & Fact. ($r{=}4$) & \textbf{83.70}$_{\pm3.30}$ & \textbf{72.33}$_{\pm2.10}$ & \textbf{77.54}$_{\pm0.42}$ & 91.23$_{\pm1.39}$ & \textbf{66.17}$_{\pm2.65}$ & \textbf{76.70}$_{\pm2.26}$ & 97.47$_{\pm0.50}$ & \textbf{61.53}$_{\pm1.86}$ & \textbf{75.43}$_{\pm1.54}$ \\
\addlinespace[1pt]
\multirow{2}{*}{Oxford Pets} & Dense & 89.13$_{\pm2.03}$ & 94.23$_{\pm2.78}$ & 91.57$_{\pm0.23}$ & 91.10$_{\pm2.95}$ & 92.73$_{\pm4.45}$ & 91.87$_{\pm2.96}$ & 93.93$_{\pm0.45}$ & 92.40$_{\pm3.90}$ & 93.13$_{\pm1.76}$ \\
 & Fact. ($r{=}4$) & \textbf{93.50}$_{\pm0.66}$ & \textbf{97.50}$_{\pm0.17}$ & \textbf{95.46}$_{\pm0.36}$ & \textbf{93.33}$_{\pm1.31}$ & \textbf{96.43}$_{\pm0.96}$ & \textbf{94.86}$_{\pm1.11}$ & \textbf{94.00}$_{\pm0.66}$ & \textbf{96.17}$_{\pm0.12}$ & \textbf{95.07}$_{\pm0.39}$ \\
\addlinespace[1pt]
\multirow{2}{*}{UCF101} & Dense & 71.83$_{\pm2.28}$ & 61.40$_{\pm4.05}$ & 66.12$_{\pm1.89}$ & \textbf{79.57}$_{\pm0.83}$ & 52.93$_{\pm2.08}$ & 63.55$_{\pm1.36}$ & \textbf{84.60}$_{\pm1.04}$ & 53.83$_{\pm8.78}$ & 65.50$_{\pm6.21}$ \\
 & Fact. ($r{=}4$) & \textbf{74.77}$_{\pm2.50}$ & \textbf{70.17}$_{\pm1.85}$ & \textbf{72.38}$_{\pm1.83}$ & 79.20$_{\pm1.61}$ & \textbf{64.23}$_{\pm2.10}$ & \textbf{70.93}$_{\pm1.66}$ & 84.57$_{\pm0.71}$ & 53.83$_{\pm1.70}$ & \textbf{65.77}$_{\pm1.21}$ \\
\cmidrule(lr){2-11}
\multirow{2}{*}{\textit{Average}} & Dense & 72.98 & 67.19 & 69.65 & 78.34 & 63.49 & 69.12 & \textbf{83.57} & 63.32 & 70.90 \\
 & Fact. ($r{=}4$) & \textbf{75.07} & \textbf{71.68} & \textbf{73.23} & \textbf{78.87} & \textbf{69.29} & \textbf{73.37} & 83.49 & \textbf{65.82} & \textbf{72.64} \\
\bottomrule
\end{tabular}
\end{adjustbox}
\end{table}

\textbf{Training details.}
We follow the standard CoOp training protocol~\cite{zhou2022learning}: the CLIP image
and text encoders are frozen and only the prompt parameters are optimized, using a
unified context of length $m=16$ shared across all classes. All variants share the
same optimizer, learning-rate schedule, data augmentation, batch size, and
train/validation/test splits as the dense CoOp baseline, so that differences in
accuracy reflect the parameterization rather than the training pipeline. We run the
1-, 4-, and 16-shot settings with three random seeds and report the mean and sample
standard deviation. Dense CoOp optimizes the full context matrix
$\bm{P}\in\R^{m\times d}$, while the factorized models parameterize it as
$\bm{P}=\bm{B}\bm{A}$ with ranks $r\in\{1,2,4,8\}$. The jointly trained factorization
optimizes both factors; the fixed-basis variants freeze $\bm{B}$ (Gaussian,
orthogonal, SVD-derived, or learned-then-frozen) and train only $\bm{A}$; and the
transfer variants load one factor from a source dataset and freeze it while training
the other on the target.

\paragraph{Initialization.}
Dense prompts use Gaussian initialization.
Joint factorization uses a balanced truncated SVD of
a sampled dense prompt. Fixed-basis ablations initialize
the trainable factor by least-squares projection onto
the selected basis, whereas transfer experiments retain
the target initialization of the trainable factor after
replacing the frozen factor. These procedures do not
generally produce identical initial prompt products.
Appendix~\ref{app:protocol} provides the implementation
details and the limitations of the recovered
initialization records.

\section{Experimental Results}
\begin{table}[t]
\centering
\scriptsize
\setlength{\tabcolsep}{3pt}
\caption{Fixing the token-side basis $\bm{B}$ matches training it across ranks and
shots. Mean accuracy (\%) over seven datasets; fixed-$\bm{B}$ variants train only
$\bm{A}$. Best per row within each shot in \textbf{bold}.}
\label{tab:fixedB}
\begin{adjustbox}{max width=\linewidth}
\begin{tabular}{ll ccccc ccccc ccccc}
\toprule
 & & \multicolumn{5}{c}{1-shot} & \multicolumn{5}{c}{4-shot} & \multicolumn{5}{c}{16-shot} \\
\cmidrule(lr){3-7}\cmidrule(lr){8-12}\cmidrule(lr){13-17}
Backbone & Rank
 & Train.\ & Gauss.\ & Orth.\ & SVD & Learn.\
 & Train.\ & Gauss.\ & Orth.\ & SVD & Learn.\
 & Train.\ & Gauss.\ & Orth.\ & SVD & Learn.\ \\
\midrule
\multirow{4}{*}{ViT-B/16}
 & 1 & 69.86 & 69.58 & 69.58 & \textbf{70.00} & 69.10 & 73.13 & \textbf{73.75} & \textbf{73.75} & 73.29 & 72.67 & 77.75 & \textbf{77.88} & \textbf{77.88} & 77.48 & 77.35 \\
 & 2 & 69.88 & 69.85 & 70.01 & 69.80 & \textbf{70.38} & \textbf{74.94} & 74.18 & 74.40 & 74.74 & 73.60 & \textbf{79.43} & 79.29 & 78.91 & 79.31 & 79.09 \\
 & 4 & \textbf{70.06} & 69.77 & 69.42 & 70.03 & 69.90 & 75.32 & 75.12 & 75.23 & \textbf{75.68} & 75.49 & 80.08 & 80.01 & 80.06 & 80.05 & \textbf{80.22} \\
 & 8 & 69.64 & 69.26 & 69.09 & \textbf{70.28} & 69.34 & 75.01 & \textbf{75.88} & 75.33 & 75.46 & 75.59 & 80.46 & 80.49 & 80.33 & 80.44 & \textbf{80.51} \\
\midrule
\multirow{4}{*}{RN50}
 & 1 & 62.48 & 62.84 & 62.82 & \textbf{62.98} & 62.30 & 66.66 & \textbf{66.91} & 66.88 & 66.50 & 66.56 & \textbf{70.75} & 70.19 & 70.35 & 70.46 & 70.45 \\
 & 2 & 61.06 & \textbf{61.73} & 61.71 & 61.17 & 61.49 & 65.74 & 66.86 & \textbf{67.07} & 66.18 & 66.34 & \textbf{72.65} & 72.45 & 72.41 & 72.19 & 72.55 \\
 & 4 & 60.50 & \textbf{60.81} & 60.30 & 60.64 & 60.20 & 67.42 & \textbf{67.84} & 67.62 & 67.55 & 67.66 & \textbf{73.66} & 73.29 & 73.43 & 73.55 & 73.60 \\
 & 8 & 59.68 & \textbf{61.10} & 59.84 & 59.94 & 59.73 & 67.61 & \textbf{68.11} & 67.52 & 67.56 & 67.26 & 73.86 & 73.79 & \textbf{74.08} & 73.90 & 73.79 \\
\bottomrule
\end{tabular}
\end{adjustbox}
\end{table}
\begin{table}[t]
\centering
\small
\setlength{\tabcolsep}{4pt}
\caption{Low-rank factorization matches dense CoOp at a fraction of the parameters. Mean
accuracy (\%) over seven datasets and three seeds. Header counts are trainable
prompt parameters; $\Delta$ is factorized minus Dense-4. }
\label{tab:headline}
\begin{adjustbox}{max width=\linewidth}
\begin{tabular}{llcccc}
\toprule
\multirow{2}{*}{Backbone} & \multirow{2}{*}{Shots}
  & Dense-16 & Dense-4 & Fact.\ $r{=}4$ & \multirow{2}{*}{$\Delta$} \\
  & & {\scriptsize (8{,}192)} & {\scriptsize (2{,}048)} & {\scriptsize (2{,}112)} & \\
\midrule
\multirow{3}{*}{ViT-B/16}
  & 1  & 68.22 & 68.14 & \textbf{70.07} & $+1.93$ \\
  & 4  & 75.27 & 75.10 & \textbf{75.32} & $+0.22$ \\
  & 16 & \textbf{80.56} & 80.01 & 80.08 & $+0.07$ \\
\midrule
\multirow{3}{*}{RN50}
  & 1  & 58.50 & 58.51 & \textbf{60.50} & $+1.99$ \\
  & 4  & 67.04 & \textbf{67.56} & 67.42 & $-0.14$ \\
  & 16 & \textbf{73.91} & 73.61 & 73.66 & $+0.05$ \\
\bottomrule
\end{tabular}
\end{adjustbox}
\end{table}


\vspace{-0.4cm}
\textbf{Generalization to New Classes.}
We first ask whether constraining the prompt to be low rank harms the open-vocabulary
behavior that makes CLIP attractive in the first place. Following the standard
base-to-new protocol, we train on the seen classes and evaluate on both the seen and
held-out unseen classes, reporting their harmonic mean~$H$.
Table~\ref{tab:base2new} compares dense CoOp with rank-4 factorized CoOp across all
seven datasets and three shot settings. Despite using only a fraction of dense CoOp's
$md$ trainable prompt parameters, the factorized prompt improves $H$ in every one of
the 21 dataset--shot cells on ViT-B/16 and in 17 of 21 on RN50. The $H$ gain,
averaged over datasets, is largest in the low-shot regime: $+3.58$, $+4.25$, and
$+1.74$ points at 1, 4, and 16 shots on ViT-B/16, and $+4.21$, $+3.01$, and $+2.05$
on RN50. This is consistent with the low-rank constraint acting as a regularizer that
helps most when supervision is scarce and matters less as more data becomes
available. The gains concentrate on the unseen classes but do not come at the expense
of seen accuracy: factorized prompts also match or improve the seen split in most
cells (15 of 21 on RN50, 14 of 21 on ViT-B/16). Dense CoOp regains a seen-accuracy
edge mainly at higher shots on a few datasets such as DTD and Oxford Flowers, and
even there its weaker unseen accuracy leaves factorized ahead on~$H$, consistent with
the dense prompt mildly overfitting the seen classes once supervision becomes ample.


\textbf{Few-Shot Classification Accuracy.}
Beyond generalization, we verify that the low-rank prompt is competitive on standard
few-shot accuracy. Table~\ref{tab:headline} reports mean accuracy across the seven
datasets for the rank-4 factorization against two dense baselines: the full 16-token
CoOp prompt (Dense-16) and a 4-token prompt (Dense-4) whose $2{,}048$ trainable
parameters nearly match the factorization's $2{,}112$. Against this parameter-matched
control, the factorization wins clearly at 1 shot ($+1.93$ on ViT-B/16, $+1.99$ on
RN50) and is statistically indistinguishable at 4 and 16 shots, where the three
methods fall within seed variance. The comparison against the approximately four-times-larger
Dense-16 prompt tells the same story from the other side: the factorized prompt
exceeds it at 1 shot (by $+1.85$ and $+2.00$ points) and matches it to within half a
point at 4 and 16 shots, all while using roughly a quarter of its parameters. The
full rank sweep (Table~\ref{tab:factorized_coop_results}
in Appendix~\ref{app:results}) shows
this behavior is consistent across ranks: accuracy is already near-saturated at
$r{=}2$ and changes little up to $r{=}8$, so a small rank suffices. Together with the
generalization gains reported above, this few-shot parity indicates
that the dense prompt matrix is over-parameterized for few-shot adaptation, and
motivates the question of the next section: if so few effective directions are
needed, must the token-side factor $\bm{B}$ be learned at all?



\textbf{Fixing the Token-Side Basis.}
Section~\ref{sec:asymmetry} predicts that the token-side factor $\bm{B}$ is far less
critical than the embedding-side factor $\bm{A}$. We test this directly by freezing
$\bm{B}$ to a non-trainable basis and optimizing only $\bm{A}$, which reduces the
trainable prompt budget to $rd$. We compare four fixed bases (Gaussian, orthogonal,
SVD-derived, and learned-then-frozen) against the fully trainable factorization.
Table~\ref{tab:fixedB} reports mean accuracy across the seven datasets for every rank
and shot. Across ranks and shots, each fixed-basis variant stays close to the
trainable reference, usually within a few tenths of a point and within roughly $1.5$
points in the worst case. A non-trainable basis is in fact the best of the five
variants in $18$ of the $24$ rank--shot--backbone settings, so fixing $\bm{B}$ is not
merely tolerable but often preferable; the random Gaussian basis is among the most
frequent winners, while the learned-then-frozen basis is the weakest of the four.
Fixing $\bm{B}$ thus costs nothing relative to learning it, consistent with the
asymmetry analysis: the $rd$ update directions available when training $\bm{A}$
already span the useful adaptation, so the additional $mr$ directions from training
$\bm{B}$ contribute little. The embedding-side coefficients $\bm{A}$ are therefore the
only factor that must be learned.

\begin{table}[t]
\centering
\small
\setlength{\tabcolsep}{5pt}
\caption{Per-pair transfer (ViT-B/16, rank 4, 16-shot; mean over three seeds). One
factor is trained while the other is frozen, with the frozen factor taken from the
source or random. References: dense CoOp and a from-scratch rank-4 factorization on
the target. Best transfer mode per row in \textbf{bold}.}
\label{tab:transfer_pairs}
\begin{adjustbox}{max width=\linewidth}
\begin{tabular}{ll cc cc cc}
\toprule
 & & \multicolumn{2}{c}{Train $A$ (freeze $B$)} & \multicolumn{2}{c}{Train $B$ (freeze $A$)} & \multicolumn{2}{c}{Reference} \\
\cmidrule(lr){3-4}\cmidrule(lr){5-6}\cmidrule(lr){7-8}
Target & Source & Src $B$ & Rand $B$ & Src $A$ & Rand $A$ & Dense & Fact. \\
\midrule
\multirow{2}{*}{Caltech101} & Oxford Pets & \textbf{95.52} & 95.51 & 94.12 & 92.51 & 95.47 & 95.51 \\
 & UCF101 & \textbf{95.52} & 95.51 & 93.38 & 92.51 & 95.47 & 95.51 \\
DTD & Oxford Flowers & \textbf{68.83} & 68.51 & 46.53 & 47.61 & 69.10 & 68.26 \\
\multirow{2}{*}{Food101} & Oxford Flowers & 85.50 & 85.80 & 86.09 & \textbf{86.44} & 85.10 & 85.78 \\
 & UCF101 & 85.82 & 85.80 & \textbf{86.47} & 86.44 & 85.10 & 85.78 \\
\multirow{3}{*}{Oxford Flowers} & DTD & \textbf{96.03} & 95.96 & 70.74 & 70.12 & 97.07 & 96.20 \\
 & Food101 & \textbf{96.15} & 95.96 & 70.76 & 70.12 & 97.07 & 96.20 \\
 & Oxford Pets & \textbf{96.07} & 95.96 & 70.22 & 70.12 & 97.07 & 96.20 \\
\multirow{2}{*}{Oxford Pets} & Oxford Flowers & \textbf{92.72} & 92.53 & 91.53 & 91.48 & 92.33 & 92.48 \\
 & Caltech101 & 92.39 & \textbf{92.53} & 91.55 & 91.48 & 92.33 & 92.48 \\
\multirow{2}{*}{UCF101} & Caltech101 & \textbf{81.88} & 81.56 & 69.69 & 69.20 & 82.17 & 81.86 \\
 & Food101 & \textbf{81.66} & 81.56 & 70.56 & 69.20 & 82.17 & 81.86 \\
\midrule
\multicolumn{2}{l}{\textit{Average}} & \textbf{89.01} & 88.93 & 78.47 & 78.10 & 89.20 & 89.01 \\
\bottomrule
\end{tabular}
\end{adjustbox}
\end{table}

\textbf{Transfer Controls.}
A natural reading of the previous section is that some token-side bases are
intrinsically better than others. We rule this out with a transfer study that
disentangles which factor is trained from what the frozen factor contains. For each
ordered source--target pair we either freeze $\bm{B}$ and train $\bm{A}$, or freeze
$\bm{A}$ and train $\bm{B}$, and in each case the frozen factor is either carried over
from a model trained on the source dataset or drawn at random.
Table~\ref{tab:transfer_pairs} reports all four modes for every pair, together with
two references: dense CoOp and a from-scratch rank-4 factorization trained directly on
the target.

Three findings stand out. First, training $\bm{A}$ dwarfs training $\bm{B}$: freezing
$\bm{B}$ and adapting $\bm{A}$ averages about $89\%$, whereas freezing $\bm{A}$ and
adapting $\bm{B}$ reaches only about $78\%$. This ten-point gap mirrors the
prompt-factor asymmetry of Section~\ref{sec:asymmetry} and holds on every target
except the easy, large-scale Food101, where all four modes lie within a point.
Second, the content of the frozen $\bm{B}$ is nearly irrelevant: the
source-initialized and random train-$\bm{A}$ columns differ by at most about $0.3$
points on any pair and by less than $0.1$ on average, and a given target reaches
essentially the same train-$\bm{A}$ accuracy from any source. Oxford Flowers attains
$96.0$ to $96.2\%$ whether transferred from DTD, Food101, or Oxford Pets, and
Caltech101 is identical from either source. A source-trained $\bm{B}$ thus carries no
reusable target-specific content beyond what a random basis already supplies. Third,
the references show that freezing $\bm{B}$ costs nothing: the
factorization that trains both factors averages $89.01\%$, identical to
freeze-$\bm{B}$/train-$\bm{A}$, so learning $\bm{B}$ adds no value over holding it
fixed. Both lie within about $0.2$ points of dense CoOp ($89.2\%$) while training only
$\bm{A}$, and match or exceed dense on Caltech101, Food101, and Oxford Pets. What
transfers is therefore not $\bm{B}$ but the ability to optimize $\bm{A}$ over it:
$\bm{B}$ functions as a fixed token-side basis and $\bm{A}$ carries the adaptation,
the transfer-level counterpart of the fixed-basis result above.


 \begin{table}[t]
\centering
\small
\setlength{\tabcolsep}{5pt}
\caption{Factor geometry across independently trained runs. Median subspace overlap
(higher is more aligned) and mean principal angle in degrees (lower is more aligned)
for the token-side factor $\bm{B}$, embedding-side factor $\bm{A}$, and induced
prompt $\bm{B}\bm{A}$, under four controlled comparisons. The most stable factor in
each condition is in \textbf{bold}.}
\label{tab:factor_geometry}
\begin{adjustbox}{max width=\linewidth}
\begin{tabular}{l ccc ccc}
\toprule
 & \multicolumn{3}{c}{Subspace overlap} & \multicolumn{3}{c}{Principal angle ($^\circ$)} \\
\cmidrule(lr){2-4}\cmidrule(lr){5-7}
Comparison & $A$ & $B$ & $BA$ & $A$ & $B$ & $BA$ \\
\midrule
Same config, diff.\ seeds  & 0.10 & \textbf{0.41} & 0.15 & 84.1 & \textbf{65.3} & 81.1 \\
Diff.\ datasets, same seed & 0.22 & \textbf{0.86} & 0.17 & 77.2 & \textbf{29.7} & 80.0 \\
Diff.\ shots, same seed    & 0.35 & \textbf{0.88} & 0.20 & 69.6 & \textbf{26.8} & 78.5 \\
Diff.\ ranks, same seed    & 0.29 & \textbf{0.89} & 0.17 & 73.4 & \textbf{27.1} & 80.3 \\
\bottomrule
\end{tabular}
\end{adjustbox}
\end{table}

\textbf{Factor Geometry.}
We now analyze in detail the factor geometry previewed in the introduction
(Fig.~\ref{fig:factor_geometry}). For every pair of independently trained runs that
differs along exactly one controlled axis (a different seed, dataset, shot count, or
rank, with the remaining axes held fixed), we measure how aligned their factors are
using the mean cosine of principal angles,
\begin{equation}
\label{eq:cosine_angle}
\operatorname{overlap}(\bm{U}_1,\bm{U}_2)
=\frac{1}{k}\sum_{i=1}^{k}
\sigma_i(\bm{U}_1^\top\bm{U}_2),
\end{equation}
where $\bm{U}_1$ and $\bm{U}_2$ are orthonormal bases
and $k$ is the smaller retained subspace rank.
We also compute the mean principal angle in degrees.
Appendix~\ref{app:protocol} describes the implementation
and the limitations of these diagnostics.

The token-side factor $\bm{B}$ is far more stable than $\bm{A}$ or $\bm{B}\bm{A}$.
Across different datasets, shot counts, and ranks, independently trained $\bm{B}$
subspaces overlap by $0.86$ to $0.89$ on average (principal angles near
$27$--$30^\circ$), whereas $\bm{A}$ and $\bm{B}\bm{A}$ overlap by only $0.17$ to
$0.35$ (angles near $70$--$80^\circ$). The embedding-side factor and the prompt it
induces therefore move substantially when the task changes, while the token basis
barely does. The seed comparison explains why: when only the seed varies, $\bm{B}$'s
overlap drops to $0.41$, lower than across any other axis. In other words, $\bm{B}$ is
more alike across different datasets sharing a seed than across different seeds on the
same dataset, which means its final position is governed mainly by its initialization
rather than by the data; training moves $\bm{B}$ very little. This is precisely why a
random fixed $\bm{B}$ suffices in the fixed-basis experiment above, and why a
source-trained $\bm{B}$ offers no advantage over a random one in the transfer
controls: $\bm{B}$ never needs to encode anything task-specific, because it scarcely
departs from where it starts. The geometry thus supports a fixed-basis reading rather
than a semantic one. $\bm{B}$ provides a reusable low-dimensional token-side basis,
and the embedding-side coefficients $\bm{A}$ carry the target adaptation.

\section{Conclusion}
We studied whether the dense CoOp prompt is over-parameterized for few-shot
adaptation by factorizing it as $\bm{P}=\bm{B}\bm{A}$ and asking which factor must be
learned. Low-rank prompts match or improve dense CoOp at a fraction of the
parameters, and across fixed-basis ablations and transfer controls the token-side
factor $\bm{B}$ can be frozen, even to a random basis, with no loss relative to
training it. A prompt-factor asymmetry and an update-space dimension gap explain why:
the embedding-side coefficients $\bm{A}$ carry the adaptation, while the token basis
serves only as a fixed low-dimensional subspace. CLIP prompt learning thus reduces to
optimizing $\bm{A}$ over a fixed $\bm{B}$. A limitation is that our analysis targets
the CoOp-style text-only prompt; extending the fixed-basis view to deep or multi-modal
prompts, and to backbones beyond CLIP, is a natural direction for future work.

\begin{ack}
Computations for this study ran on the Pod cluster at the
UCSB Center for Scientific Computing (CSC). We acknowledge
the CSC's supporting institutions, the California NanoSystems
Institute and the UCSB Materials Research Science and
Engineering Center (NSF DMR-2308708), together with NSF
infrastructure awards CNS-1725797 and OAC-1925717.
\end{ack}

\clearpage
\bibliographystyle{ieeenat_fullname}
\bibliography{main}

\clearpage
\appendix
\section{Theoretical details and proofs}
\label{app:theory}

\subsection{A Linear Surrogate for Prompt-Factor Asymmetry}

We first analyze a simplified prompt-space approximation problem. Let
\(\bm{P}^\star\in\mathbb{R}^{m\times d}\) denote an ideal prompt matrix, where
\(m=n_{\rm ctx}\) is the number of context tokens and \(d\) is the CLIP text
embedding dimension. A rank-\(r\) factorized prompt is parameterized as
\begin{equation}
\label{eq:app-linear-01}
    \bm{P} = \bm{B}\bm{A}, \qquad
    \bm{B}\in\mathbb{R}^{m\times r},\quad
    \bm{A}\in\mathbb{R}^{r\times d}.
\end{equation}
We compare two one-sided adaptation modes: fixing \(\bm{B}\) and optimizing \(\bm{A}\),
or fixing \(\bm{A}\) and optimizing \(\bm{B}\).

\begin{assumption}[Linear prompt-space surrogate]
\label{assump:linear-surrogate}
There exists an ideal prompt matrix \(\bm{P}^\star\in\mathbb{R}^{m\times d}\)
such that local prompt adaptation can be approximated by the least-squares
problem
\begin{equation}
\label{eq:app-linear-02}
    \min_{\bm{A},\bm{B}}\|\bm{P}^\star-\bm{B}\bm{A}\|_F^2.
\end{equation}
\end{assumption}
\noindent This surrogate ignores the nonlinear CLIP text encoder, but isolates the
geometry induced by the factorized prompt parameterization. We now show that
the two factors play asymmetric roles.

\begin{lemma}[Fixed token-side basis]
\label{lem:fixed-B}
Let \(\bm{U}\in\mathbb{R}^{m\times r}\) satisfy \(\bm{U}^\top \bm{U}=\bm{I}_r\), with
\(1\le r\le m\). Then
\begin{equation}
\label{eq:app-linear-03}
    \min_{\bm{A}\in\mathbb{R}^{r\times d}}
    \|\bm{P}^\star-\bm{U}\bm{A}\|_F^2
    =
    \|(\bm{I}_m-\bm{U}\bm{U}^\top)\bm{P}^\star\|_F^2.
\end{equation}
\end{lemma}

\begin{proof}
The objective is
\[
    \|\bm{P}^\star-\bm{U}\bm{A}\|_F^2.
\]
Since \(\bm{U}^\top \bm{U}=\bm{I}_r\), this is a standard least-squares problem. Setting the
gradient with respect to \(\bm{A}\) to zero gives
\[
    \bm{U}^\top(\bm{U}\bm{A}-\bm{P}^\star)=\bm{0}.
\]
Therefore,
\[
    \bm{A}^\star = \bm{U}^\top \bm{P}^\star.
\]
Substituting this minimizer gives
\[
    \bm{P}^\star-\bm{U}\bm{A}^\star
    =
    \bm{P}^\star-\bm{U}\bm{U}^\top \bm{P}^\star
    =
    (\bm{I}_m-\bm{U}\bm{U}^\top)\bm{P}^\star.
\]
Taking the squared Frobenius norm gives the result.
\end{proof}

\noindent Lemma~\ref{lem:fixed-B} shows that fixing \(\bm{B}=\bm{U}\) restricts the learned prompt
only through a projection in the token-position space \(\mathbb{R}^m\).

\begin{lemma}[Fixed embedding-side factor]
\label{lem:fixed-A}
Let \(\bm{Q}\in\mathbb{R}^{r\times d}\) satisfy \(\bm{Q}\bm{Q}^\top=\bm{I}_r\), with
\(1\le r\le d\). Then
\begin{equation}
\label{eq:app-linear-08}
    \min_{\bm{B}\in\mathbb{R}^{m\times r}}
    \|\bm{P}^\star-\bm{B}\bm{Q}\|_F^2
    =
    \|\bm{P}^\star(\bm{I}_d-\bm{Q}^\top \bm{Q})\|_F^2.
\end{equation}
\end{lemma}

\begin{proof}
Define
\[
    f(\bm{B})=\|\bm{P}^\star-\bm{B}\bm{Q}\|_F^2.
\]
Setting the
gradient with respect to \(\bm{B}\) to zero gives
\[
    \nabla_{\bm{B}} f(\bm{B})=2(\bm{B}\bm{Q}-\bm{P}^\star)\bm{Q}^\top=\bm{0}.
\]
Using \(\bm{Q}\bm{Q}^\top=\bm{I}_r\), this gives
\[
    \bm{B}^\star = \bm{P}^\star \bm{Q}^\top.
\]
The residual is therefore
\[
    \bm{P}^\star-\bm{B}^\star \bm{Q}
    =
    \bm{P}^\star(\bm{I}_d-\bm{Q}^\top \bm{Q}).
\]
Taking the squared Frobenius norm gives the result.
\end{proof}

\noindent Lemma~\ref{lem:fixed-A} shows that fixing \(\bm{A}=\bm{Q}\) restricts the learned prompt
through a projection in the CLIP embedding space \(\mathbb{R}^d\). Since
typically \(d\gg m\), this restriction is substantially more severe.

\begin{theorem}[Expected prompt-factor asymmetry]
\label{thm:prompt-factor-asymmetry}
Let \(\bm{P}^\star\in\mathbb{R}^{m\times d}\) be nonzero, and let
\(1\le r\le m < d\). Let $\bm{U}\sim {\rm Unif}(\mathrm{St}(m,r))$
be a Haar-uniform random \(m\times r\) matrix with orthonormal columns, and let $\bm{Q}^\top\sim {\rm Unif}(\mathrm{St}(d,r))$,
so that \(\bm{Q}\in\mathbb{R}^{r\times d}\) has orthonormal rows. Then
\begin{equation}
\label{eq:app-linear-13}
    \mathbb{E}_{\bm{U}}
    \min_{\bm{A}}
    \|\bm{P}^\star-\bm{U}\bm{A}\|_F^2
    =
    \left(1-\frac{r}{m}\right)
    \|\bm{P}^\star\|_F^2,
\end{equation}
whereas
\begin{equation}
\label{eq:app-linear-14}
    \mathbb{E}_{\bm{Q}}
    \min_{\bm{B}}
    \|\bm{P}^\star-\bm{B}\bm{Q}\|_F^2
    =
    \left(1-\frac{r}{d}\right)
    \|\bm{P}^\star\|_F^2.
\end{equation}
Consequently,
\begin{equation}
\label{eq:app-linear-15}
    \mathbb{E}_{\bm{U}}
    \min_{\bm{A}}
    \|\bm{P}^\star-\bm{U}\bm{A}\|_F^2
    <
    \mathbb{E}_{\bm{Q}}
    \min_{\bm{B}}
    \|\bm{P}^\star-\bm{B}\bm{Q}\|_F^2.
\end{equation}
Thus, under the linear prompt-space surrogate, fixing the token-side factor
\(\bm{B}\) and optimizing \(\bm{A}\) has strictly smaller expected approximation error
than fixing \(\bm{A}\) and optimizing \(\bm{B}\).
\end{theorem}

\begin{proof}
By Lemma~\ref{lem:fixed-B},
\[
    \min_{\bm{A}} \|\bm{P}^\star-\bm{U}\bm{A}\|_F^2
    =
    \|(\bm{I}_m-\bm{U}\bm{U}^\top)\bm{P}^\star\|_F^2.
\]
Since \(\bm{U}\bm{U}^\top\) is an orthogonal projector,
\[
    \|(\bm{I}_m-\bm{U}\bm{U}^\top)\bm{P}^\star\|_F^2
    =
    \|\bm{P}^\star\|_F^2
    -
    \|\bm{U}\bm{U}^\top \bm{P}^\star\|_F^2.
\]
The captured energy is
\[
    \|\bm{U}\bm{U}^\top \bm{P}^\star\|_F^2
    =
    \operatorname{tr}\big((\bm{P}^\star)^\top \bm{U}\bm{U}^\top \bm{P}^\star\big).
\]
Taking expectation and using linearity of trace,
\[
    \mathbb{E}_{\bm{U}}
    \|\bm{U}\bm{U}^\top \bm{P}^\star\|_F^2
    =
    \operatorname{tr}
    \big(
        (\bm{P}^\star)^\top
        \mathbb{E}_{\bm{U}}[\bm{U}\bm{U}^\top]
        \bm{P}^\star
    \big).
\]
Since \(\bm{U}\) is Haar-uniform on \(\mathrm{St}(m,r)\), rotational invariance
implies
\[
    \mathbb{E}_{\bm{U}}[\bm{U}\bm{U}^\top]=\frac{r}{m}\bm{I}_m.
\]
Therefore,
\[
    \mathbb{E}_{\bm{U}}
    \|\bm{U}\bm{U}^\top \bm{P}^\star\|_F^2
    =
    \frac{r}{m}\|\bm{P}^\star\|_F^2.
\]
Hence,
\[
    \mathbb{E}_{\bm{U}}
    \min_{\bm{A}}
    \|\bm{P}^\star-\bm{U}\bm{A}\|_F^2
    =
    \left(1-\frac{r}{m}\right)
    \|\bm{P}^\star\|_F^2.
\]
Similarly, we have,
\[
    \mathbb{E}_{\bm{Q}}
    \min_{\bm{B}}
    \|\bm{P}^\star-\bm{B}\bm{Q}\|_F^2
    =
    \left(1-\frac{r}{d}\right)
    \|\bm{P}^\star\|_F^2.
\]
Finally, since \(m<d\),
\[
    1-\frac{r}{m}
    <
    1-\frac{r}{d}.
\]
Because \(\bm{P}^\star\ne \bm{0}\), the inequality between the expected errors is strict.
\end{proof}

\begin{corollary}[Implication for CLIP prompt dimensions]
\label{cor:clip-dimensions}
In the CLIP setting used in our experiments, \(m=16\) and \(d=512\). Therefore,
for any fixed rank \(r\le 16\),
\begin{equation}
\label{eq:app-linear-25}
    \frac{1-r/m}{1-r/d}
    =
    \frac{1-r/16}{1-r/512}.
\end{equation}
Thus, a random fixed token-side basis \(\bm{B}\) removes a fraction \(r/16\) of the
ideal prompt energy in expectation, whereas a random fixed embedding-side
factor \(\bm{A}\) removes only a fraction \(r/512\). Equivalently, fixing \(\bm{B}\) and
training \(\bm{A}\) preserves a much larger fraction of the ideal local prompt
update than fixing \(\bm{A}\) and training \(\bm{B}\).
\end{corollary}
\subsection{Local Linearization of the CLIP Prompt Objective}
\label{subsec:local_linearization}

The linear surrogate above isolates the prompt-space approximation geometry.
However, the actual CoOp objective is not the Frobenius problem
$\min_{\bm{P}} \|\bm{P}-\bm{P}^\star\|_F^2$. Instead, the continuous prompt \(\bm{P}\)
is passed through the frozen CLIP text encoder and optimized through a
classification loss. Let $\mathcal{L}(\bm{P})$ denote the prompt-learning objective,
and let
\begin{equation}
\label{eq:app-local-01}
    \bm{p}=\operatorname{vec}(\bm{P})\in\mathbb{R}^{md}.
\end{equation}
Around a reference prompt \(\bm{p}_0=\operatorname{vec}(\bm{P}_0)\), write
$\bm{\delta} = \bm{p}-\bm{p}_0.$ A local second-order approximation gives
\begin{equation}
\label{eq:app-local-02}
    \mathcal{L}(\bm{p}_0+\bm{\delta})
    \approx
    \mathcal{L}(\bm{p}_0)
    +
    \bm{g}^\top \bm{\delta}
    +
    \frac{1}{2}\bm{\delta}^\top \bm{H}\bm{\delta},
\end{equation}
where $\bm{g}=\nabla_{\bm{p}} \mathcal{L}(\bm{p}_0)$, and \(\bm{H}\succeq \bm{0}\)
denotes a local positive-semidefinite curvature approximation, such as a
Gauss--Newton or Fisher matrix. If \(\bm{H}\) is positive definite, or after adding
a damping term, define
\begin{equation}
\label{eq:app-local-03}
    \bm{\delta}^\star = -\bm{H}^{-1}\bm{g}.
\end{equation}
Completing the square gives
\[
    \bm{g}^\top\bm{\delta}
    +
    \frac{1}{2}\bm{\delta}^\top \bm{H}\bm{\delta}
    =
    \frac{1}{2}
    (\bm{\delta}-\bm{\delta}^\star)^\top \bm{H}(\bm{\delta}-\bm{\delta}^\star)
    +
    \mathrm{constant}.
\]
Therefore, up to constants independent of \(\bm{\delta}\), the local prompt-learning
problem becomes
\begin{equation}
\label{eq:app-local-05}
    \min_{\bm{\delta}\in\mathcal{S}}
    \|\bm{\delta}-\bm{\delta}^\star\|_{\bm{H}}^2,
    \qquad
    \|\bm{x}\|_{\bm{H}}^2=\bm{x}^\top \bm{H}\bm{x},
\end{equation}
where \(\mathcal{S}\) is the set of prompt updates allowed by the chosen
parameterization. Thus, \(\bm{\delta}^\star\) is the ideal local update suggested by
the quadratic approximation, while \(\mathcal{S}\) captures the restrictions
imposed by fixing one factor and training the other.

\begin{assumption}[Factorized reference prompt]
\label{app:assump:factorized_reference_prompt}
The reference prompt admits a rank-\(r\) factorization
\begin{equation}
\label{eq:app-local-06}
    \bm{P}_0=\bm{B}_0\bm{A}_0,
    \qquad
    \bm{B}_0\in\mathbb{R}^{m\times r},
    \quad
    \bm{A}_0\in\mathbb{R}^{r\times d}.
\end{equation}
Furthermore, \(\bm{B}_0\) has full column rank \(r\), and \(\bm{A}_0\) has full row
rank \(r\).
\end{assumption}

\noindent We now identify the local update spaces \(\mathcal{S}\) for the two one-sided
adaptation modes.

\begin{proposition}[Local update space for fixed \(\bm{B}\)]
\label{prop:fixed_B_update_space}
Suppose Assumption~\ref{app:assump:factorized_reference_prompt} holds. If
\(\bm{B}_0\) is fixed and $\bm{A}=\bm{A}_0+\Delta\bm{A},$
then the prompt update satisfies $\Delta\bm{P} = \bm{B}_0\Delta\bm{A}.$ Hence the
allowed local update space is
\begin{equation}
\label{eq:app-local-07}
    \mathcal{S}_{\bm{A}}(\bm{B}_0)
    =
    \left\{
        \operatorname{vec}(\bm{B}_0\bm{X}):
        \bm{X}\in\mathbb{R}^{r\times d}
    \right\}
    =
    \operatorname{Col}(\bm{I}_d\otimes \bm{B}_0).
\end{equation}
If \(\operatorname{rank}(\bm{B}_0)=r\), then
\begin{equation}
\label{eq:app-local-08}
    \dim \mathcal{S}_{\bm{A}}(\bm{B}_0)=rd.
\end{equation}
\end{proposition}

\begin{proof}
Since \(\bm{B}_0\) is fixed and \(\bm{A}=\bm{A}_0+\Delta\bm{A}\),
\[
    \bm{P}
    =
    \bm{B}_0(\bm{A}_0+\Delta\bm{A})
    =
    \bm{B}_0\bm{A}_0+\bm{B}_0\Delta\bm{A}.
\]
Because \(\bm{P}_0=\bm{B}_0\bm{A}_0\), the prompt update is
\[
    \Delta\bm{P} = \bm{P}-\bm{P}_0 = \bm{B}_0\Delta\bm{A}.
\]
Vectorizing,
\[
    \operatorname{vec}(\Delta\bm{P})
    =
    \operatorname{vec}(\bm{B}_0\Delta\bm{A})
    =
    (\bm{I}_d\otimes \bm{B}_0)\operatorname{vec}(\Delta\bm{A}).
\]
Therefore,
\[
    \mathcal{S}_{\bm{A}}(\bm{B}_0)
    =
    \operatorname{Col}(\bm{I}_d\otimes \bm{B}_0).
\]
Using the Kronecker rank identity,
\[
    \operatorname{rank}(\bm{I}_d\otimes \bm{B}_0)
    =
    \operatorname{rank}(\bm{I}_d)\operatorname{rank}(\bm{B}_0)
    =
    d r.
\]
Thus,
\[
    \dim \mathcal{S}_{\bm{A}}(\bm{B}_0)=rd.
\]
\end{proof}

\begin{proposition}[Local update space for fixed \(\bm{A}\)]
\label{prop:fixed_A_update_space}
Suppose Assumption~\ref{app:assump:factorized_reference_prompt} holds. If
\(\bm{A}_0\) is fixed and
\[
    \bm{B}=\bm{B}_0+\Delta\bm{B},
\]
then the prompt update satisfies
\[
    \Delta\bm{P} = \Delta\bm{B}\, \bm{A}_0.
\]
Hence the allowed local update space is
\begin{equation}
\label{eq:app-local-17}
    \mathcal{S}_{\bm{B}}(\bm{A}_0)
    =
    \left\{
        \operatorname{vec}(\bm{Y}\bm{A}_0):
        \bm{Y}\in\mathbb{R}^{m\times r}
    \right\}
    =
    \operatorname{Col}(\bm{A}_0^\top\otimes \bm{I}_m).
\end{equation}
If \(\operatorname{rank}(\bm{A}_0)=r\), then
\begin{equation}
\label{eq:app-local-18}
    \dim \mathcal{S}_{\bm{B}}(\bm{A}_0)=mr.
\end{equation}
\end{proposition}

\begin{proof}
Similar to Proposition \ref{prop:fixed_B_update_space}.
\end{proof}

\begin{corollary}[Local dimension gap]
\label{app:cor:local_dimension_gap}
Under Assumption~\ref{app:assump:factorized_reference_prompt},
\begin{equation}
\label{eq:app-local-19}
    \frac{\dim\mathcal{S}_{\bm{A}}(\bm{B}_0)}
    {\dim\mathcal{S}_{\bm{B}}(\bm{A}_0)}
    =
    \frac{rd}{mr}
    =
    \frac{d}{m}.
\end{equation}
In the CLIP setting used in our experiments, \(m=16\) and \(d=512\), so
\[
    \frac{\dim\mathcal{S}_{\bm{A}}(\bm{B}_0)}
    {\dim\mathcal{S}_{\bm{B}}(\bm{A}_0)}
    =
    32.
\]
Thus, fixing \(\bm{B}\) and training \(\bm{A}\) provides \(32\) times more local prompt
update directions than fixing \(\bm{A}\) and training \(\bm{B}\).
\end{corollary}

\noindent The local objective
\[
    \min_{\bm{\delta}\in\mathcal{S}}
    \|\bm{\delta}-\bm{\delta}^\star\|_{\bm{H}}^2
\]
shows how this dimension gap connects to the actual CLIP training loss. The
two one-sided adaptation modes define different feasible subspaces:
\[
    \mathcal{S}=\mathcal{S}_{\bm{A}}(\bm{B}_0)
    \quad
    \text{for fixed \(\bm{B}\), train \(\bm{A}\),}
\]
and
\[
    \mathcal{S}=\mathcal{S}_{\bm{B}}(\bm{A}_0)
    \quad
    \text{for fixed \(\bm{A}\), train \(\bm{B}\).}
\]
Therefore, the comparison between the two modes is a comparison between how
well these two subspaces approximate the ideal local update \(\bm{\delta}^\star\)
under the curvature metric induced by \(\bm{H}\).
\\~\\
\noindent If \(\bm{H}\) is close to isotropic, or if one considers the simplified case
\(\bm{H}=\bm{I}\), this local objective reduces to Euclidean projection in prompt
space:
\[
    \min_{\bm{\delta}\in\mathcal{S}}
    \|\bm{\delta}-\bm{\delta}^\star\|_2^2.
\]
Identifying \(\bm{\delta}^\star=\operatorname{vec}(\Delta^\star)\), this is
equivalent to the Frobenius approximation problem
\begin{equation}
\label{eq:app-local-25}
    \min_{\Delta\bm{P}\in\mathcal{S}}
    \|\Delta\bm{P}-\Delta^\star\|_F^2.
\end{equation}
This isotropic special case shows that fixed \(\bm{B}\), train \(\bm{A}\), projects the
ideal update in the small token-position space \(\mathbb{R}^m\), whereas fixed
\(\bm{A}\), train \(\bm{B}\), projects it in the much larger embedding space
\(\mathbb{R}^d\). Together, the update-space dimension gap and the projection
result explain why the embedding-side factor \(\bm{A}\) carries most of the useful
adaptation, while the token-side factor \(\bm{B}\) can often be fixed.
\subsection{Smoothness-Only Optimization Guarantees}
\label{subsec:smoothness_guarantee}

The previous results compare the approximation geometry of one-sided
factorized prompt learning. We now state a separate optimization guarantee.
This guarantee shows that, once the token-side factor \(\bm{B}\) is fixed,
gradient-based optimization of the embedding-side factor \(\bm{A}\) converges to a
first-order stationary point under directional smoothness. The result does not
require convexity or strong convexity.
\\~\\
Let $\Phi(\bm{B},\bm{A})=F(\bm{B}\bm{A})$,
where $F:\mathbb{R}^{m\times d}\to\mathbb{R}$
denotes the prompt objective as a function of the continuous prompt matrix
\(\bm{P}=\bm{B}\bm{A}\). For a fixed token-side factor \(\bm{B}\), define the
restricted objective
\begin{equation}
\label{eq:app-conv-01}
    \phi_{\bm{B}}(\bm{A})=\Phi(\bm{B},\bm{A})=F(\bm{B}\bm{A}).
\end{equation}

\begin{assumption}[\(\bm{A}\)-directional smoothness]
\label{app:assump:A_directional_smoothness}
For every fixed \(\bm{B}\in\mathbb{R}^{m\times r}\), the restricted objective
\[
    \bm{A}\mapsto \phi_{\bm{B}}(\bm{A})
\]
is \(L_A\)-smooth in the Frobenius norm. That is, for all
\(\bm{A}_1,\bm{A}_2\in\mathbb{R}^{r\times d}\),
\begin{equation}
\label{eq:app-conv-03}
    \|\nabla_{\bm{A}}\phi_{\bm{B}}(\bm{A}_1)-\nabla_{\bm{A}}\phi_{\bm{B}}(\bm{A}_2)\|_F
    \le
    L_A\|\bm{A}_1-\bm{A}_2\|_F.
\end{equation}
\end{assumption}

\begin{assumption}[Lower boundedness]
\label{app:assump:lower_bounded_phi_B}
For every fixed \(\bm{B}\), the restricted objective \(\phi_{\bm{B}}\) is lower
bounded. That is, there exists \(\phi_{\bm{B}}^{\inf}>-\infty\) such that
\begin{equation}
\label{eq:app-conv-04}
    \phi_{\bm{B}}(\bm{A})\ge \phi_{\bm{B}}^{\inf}
\end{equation}
for all \(\bm{A}\in\mathbb{R}^{r\times d}\).
\end{assumption}

\begin{theorem}[Full-gradient convergence for fixed-\(\bm{B}\), train-\(\bm{A}\)]
\label{app:thm:fixed_B_gd_convergence_directional}
Suppose Assumptions~\ref{app:assump:A_directional_smoothness}
and~\ref{app:assump:lower_bounded_phi_B} hold. Consider full-gradient descent on
\(\phi_{\bm{B}}\):
\begin{equation}
\label{eq:app-conv-05}
    \bm{A}_{t+1}
    =
    \bm{A}_t-\eta\nabla_{\bm{A}}\phi_{\bm{B}}(\bm{A}_t),
\end{equation}
with step size $0<\eta\le \frac{1}{L_A}.$
Then
\begin{equation}
\label{eq:app-conv-06}
    \phi_{\bm{B}}(\bm{A}_{t+1})
    \le
    \phi_{\bm{B}}(\bm{A}_t)
    -
    \frac{\eta}{2}
    \|\nabla_{\bm{A}}\phi_{\bm{B}}(\bm{A}_t)\|_F^2.
\end{equation}
Consequently, after \(T\) iterations,
\begin{equation}
\label{eq:app-conv-07}
    \min_{0\le t<T}
    \|\nabla_{\bm{A}}\phi_{\bm{B}}(\bm{A}_t)\|_F^2
    \le
    \frac{
        2(\phi_{\bm{B}}(\bm{A}_0)-\phi_{\bm{B}}^{\inf})
    }{
        \eta T
    }.
\end{equation}
In particular, choosing \(\eta=1/L_A\) gives
\begin{equation}
\label{eq:app-conv-08}
    \min_{0\le t<T}
    \|\nabla_{\bm{A}}\phi_{\bm{B}}(\bm{A}_t)\|_F^2
    \le
    \frac{
        2L_A(\phi_{\bm{B}}(\bm{A}_0)-\phi_{\bm{B}}^{\inf})
    }{
        T
    }.
\end{equation}
\end{theorem}

\begin{proof}
By Assumption~\ref{app:assump:A_directional_smoothness}, \(\phi_{\bm{B}}\) is
\(L_A\)-smooth. Therefore, the descent lemma gives
\[
    \phi_{\bm{B}}(\bm{A}_{t+1})
    \le
    \phi_{\bm{B}}(\bm{A}_t)
    +
    \langle \nabla_{\bm{A}}\phi_{\bm{B}}(\bm{A}_t), \bm{A}_{t+1}-\bm{A}_t\rangle
    +
    \frac{L_A}{2}\|\bm{A}_{t+1}-\bm{A}_t\|_F^2.
\]
Using the gradient descent update,
\[
    \bm{A}_{t+1}-\bm{A}_t
    =
    -\eta\nabla_{\bm{A}}\phi_{\bm{B}}(\bm{A}_t),
\]
we obtain
\[
\begin{aligned}
    \phi_{\bm{B}}(\bm{A}_{t+1})
    &\le
    \phi_{\bm{B}}(\bm{A}_t)
    -
    \eta\|\nabla_{\bm{A}}\phi_{\bm{B}}(\bm{A}_t)\|_F^2
    +
    \frac{L_A\eta^2}{2}
    \|\nabla_{\bm{A}}\phi_{\bm{B}}(\bm{A}_t)\|_F^2 \\
    &=
    \phi_{\bm{B}}(\bm{A}_t)
    -
    \eta
    \left(
        1-\frac{L_A\eta}{2}
    \right)
    \|\nabla_{\bm{A}}\phi_{\bm{B}}(\bm{A}_t)\|_F^2.
\end{aligned}
\]
Since \(\eta\le 1/L_A\),
\[
    1-\frac{L_A\eta}{2}
    \ge
    \frac{1}{2}.
\]
Thus,
\[
    \phi_{\bm{B}}(\bm{A}_{t+1})
    \le
    \phi_{\bm{B}}(\bm{A}_t)
    -
    \frac{\eta}{2}
    \|\nabla_{\bm{A}}\phi_{\bm{B}}(\bm{A}_t)\|_F^2.
\]
Summing this inequality from \(t=0\) to \(T-1\) gives
\[
    \frac{\eta}{2}
    \sum_{t=0}^{T-1}
    \|\nabla_{\bm{A}}\phi_{\bm{B}}(\bm{A}_t)\|_F^2
    \le
    \phi_{\bm{B}}(\bm{A}_0)-\phi_{\bm{B}}(\bm{A}_T).
\]
Using the lower bound \(\phi_{\bm{B}}(\bm{A}_T)\ge \phi_{\bm{B}}^{\inf}\), we obtain
\[
    \sum_{t=0}^{T-1}
    \|\nabla_{\bm{A}}\phi_{\bm{B}}(\bm{A}_t)\|_F^2
    \le
    \frac{
        2(\phi_{\bm{B}}(\bm{A}_0)-\phi_{\bm{B}}^{\inf})
    }{
        \eta
    }.
\]
Finally, the minimum is no larger than the average:
\[
    \min_{0\le t<T}
    \|\nabla_{\bm{A}}\phi_{\bm{B}}(\bm{A}_t)\|_F^2
    \le
    \frac{1}{T}
    \sum_{t=0}^{T-1}
    \|\nabla_{\bm{A}}\phi_{\bm{B}}(\bm{A}_t)\|_F^2.
\]
Combining the last two inequalities gives
\[
    \min_{0\le t<T}
    \|\nabla_{\bm{A}}\phi_{\bm{B}}(\bm{A}_t)\|_F^2
    \le
    \frac{
        2(\phi_{\bm{B}}(\bm{A}_0)-\phi_{\bm{B}}^{\inf})
    }{
        \eta T
    }.
\]
Choosing \(\eta=1/L_A\) gives the stated special case.
\end{proof}

\begin{assumption}[Unbiased stochastic gradients with bounded variance]
\label{assump:stochastic_gradient}
For every fixed \(\bm{B}\), suppose that at each iteration \(t\) we have access to
a stochastic gradient estimator \(\bm{G}_t\) satisfying
\begin{equation}
\label{eq:app-conv-18}
    \mathbb{E}[\bm{G}_t\mid \bm{A}_t]
    =
    \nabla_{\bm{A}}\phi_{\bm{B}}(\bm{A}_t),
\end{equation}
and
\begin{equation}
\label{eq:app-conv-19}
    \mathbb{E}
    \left[
        \|\bm{G}_t-\nabla_{\bm{A}}\phi_{\bm{B}}(\bm{A}_t)\|_F^2
        \mid \bm{A}_t
    \right]
    \le
    \sigma^2.
\end{equation}
\end{assumption}

\begin{theorem}[Stochastic-gradient convergence for fixed-\(\bm{B}\), train-\(\bm{A}\)]
\label{app:thm:fixed_B_sgd_convergence_directional}
Suppose Assumptions~\ref{app:assump:A_directional_smoothness},
\ref{app:assump:lower_bounded_phi_B}, and
\ref{assump:stochastic_gradient} hold. Consider stochastic gradient descent
on \(\phi_{\bm{B}}\):
\begin{equation}
\label{eq:app-conv-20}
    \bm{A}_{t+1}
    =
    \bm{A}_t-\eta \bm{G}_t,
\end{equation}
with step size $0<\eta\le \frac{1}{L_A}.$
Then after \(T\) iterations,
\begin{equation}
\label{eq:app-conv-21}
    \min_{0\le t<T}
    \mathbb{E}
    \left[
        \|\nabla_{\bm{A}}\phi_{\bm{B}}(\bm{A}_t)\|_F^2
    \right]
    \le
    \frac{
        2(\phi_{\bm{B}}(\bm{A}_0)-\phi_{\bm{B}}^{\inf})
    }{
        \eta T
    }
    +
    L_A\eta\sigma^2.
\end{equation}
In particular, choosing \(\eta=\Theta(T^{-1/2})\), subject to
\(\eta\le 1/L_A\), gives
\begin{equation}
\label{eq:app-conv-22}
    \min_{0\le t<T}
    \mathbb{E}
    \left[
        \|\nabla_{\bm{A}}\phi_{\bm{B}}(\bm{A}_t)\|_F^2
    \right]
    =
    O(T^{-1/2}).
\end{equation}
\end{theorem}

\begin{proof}
By Assumption~\ref{app:assump:A_directional_smoothness}, \(\phi_{\bm{B}}\) is
\(L_A\)-smooth. The descent lemma gives
\[
    \phi_{\bm{B}}(\bm{A}_{t+1})
    \le
    \phi_{\bm{B}}(\bm{A}_t)
    +
    \langle \nabla_{\bm{A}}\phi_{\bm{B}}(\bm{A}_t), \bm{A}_{t+1}-\bm{A}_t\rangle
    +
    \frac{L_A}{2}\|\bm{A}_{t+1}-\bm{A}_t\|_F^2.
\]
Using the stochastic update
\[
    \bm{A}_{t+1}-\bm{A}_t=-\eta \bm{G}_t,
\]
we get
\[
    \phi_{\bm{B}}(\bm{A}_{t+1})
    \le
    \phi_{\bm{B}}(\bm{A}_t)
    -
    \eta
    \langle \nabla_{\bm{A}}\phi_{\bm{B}}(\bm{A}_t),\bm{G}_t\rangle
    +
    \frac{L_A\eta^2}{2}
    \|\bm{G}_t\|_F^2.
\]
Taking conditional expectation given \(\bm{A}_t\), and using
\[
    \mathbb{E}[\bm{G}_t\mid \bm{A}_t]
    =
    \nabla_{\bm{A}}\phi_{\bm{B}}(\bm{A}_t),
\]
we obtain
\[
    \mathbb{E}
    [
        \phi_{\bm{B}}(\bm{A}_{t+1})
        \mid \bm{A}_t
    ]
    \le
    \phi_{\bm{B}}(\bm{A}_t)
    -
    \eta
    \|\nabla_{\bm{A}}\phi_{\bm{B}}(\bm{A}_t)\|_F^2
    +
    \frac{L_A\eta^2}{2}
    \mathbb{E}
    [
        \|\bm{G}_t\|_F^2
        \mid \bm{A}_t
    ].
\]
By Assumption~\ref{assump:stochastic_gradient},
\[
\begin{aligned}
    \mathbb{E}
    [
        \|\bm{G}_t\|_F^2
        \mid \bm{A}_t
    ]
    &=
    \|\nabla_{\bm{A}}\phi_{\bm{B}}(\bm{A}_t)\|_F^2
    +
    \mathbb{E}
    \left[
        \|\bm{G}_t-\nabla_{\bm{A}}\phi_{\bm{B}}(\bm{A}_t)\|_F^2
        \mid \bm{A}_t
    \right] \\
    &\le
    \|\nabla_{\bm{A}}\phi_{\bm{B}}(\bm{A}_t)\|_F^2+\sigma^2.
\end{aligned}
\]
Therefore,
\[
    \mathbb{E}
    [
        \phi_{\bm{B}}(\bm{A}_{t+1})
        \mid \bm{A}_t
    ]
    \le
    \phi_{\bm{B}}(\bm{A}_t)
    -
    \eta
    \left(
        1-\frac{L_A\eta}{2}
    \right)
    \|\nabla_{\bm{A}}\phi_{\bm{B}}(\bm{A}_t)\|_F^2
    +
    \frac{L_A\eta^2}{2}\sigma^2.
\]
Since \(\eta\le 1/L_A\),
\[
    1-\frac{L_A\eta}{2}
    \ge
    \frac{1}{2}.
\]
Thus,
\[
    \mathbb{E}
    [
        \phi_{\bm{B}}(\bm{A}_{t+1})
        \mid \bm{A}_t
    ]
    \le
    \phi_{\bm{B}}(\bm{A}_t)
    -
    \frac{\eta}{2}
    \|\nabla_{\bm{A}}\phi_{\bm{B}}(\bm{A}_t)\|_F^2
    +
    \frac{L_A\eta^2}{2}\sigma^2.
\]
Taking total expectation gives
\[
    \mathbb{E}[\phi_{\bm{B}}(\bm{A}_{t+1})]
    \le
    \mathbb{E}[\phi_{\bm{B}}(\bm{A}_t)]
    -
    \frac{\eta}{2}
    \mathbb{E}
    \left[
        \|\nabla_{\bm{A}}\phi_{\bm{B}}(\bm{A}_t)\|_F^2
    \right]
    +
    \frac{L_A\eta^2}{2}\sigma^2.
\]
Rearranging,
\[
    \frac{\eta}{2}
    \mathbb{E}
    \left[
        \|\nabla_{\bm{A}}\phi_{\bm{B}}(\bm{A}_t)\|_F^2
    \right]
    \le
    \mathbb{E}[\phi_{\bm{B}}(\bm{A}_t)]
    -
    \mathbb{E}[\phi_{\bm{B}}(\bm{A}_{t+1})]
    +
    \frac{L_A\eta^2}{2}\sigma^2.
\]
Summing from \(t=0\) to \(T-1\) yields
\[
    \frac{\eta}{2}
    \sum_{t=0}^{T-1}
    \mathbb{E}
    \left[
        \|\nabla_{\bm{A}}\phi_{\bm{B}}(\bm{A}_t)\|_F^2
    \right]
    \le
    \phi_{\bm{B}}(\bm{A}_0)
    -
    \mathbb{E}[\phi_{\bm{B}}(\bm{A}_T)]
    +
    \frac{L_A\eta^2T}{2}\sigma^2.
\]
Using \(\phi_{\bm{B}}(\bm{A}_T)\ge \phi_{\bm{B}}^{\inf}\), we obtain
\[
    \frac{\eta}{2}
    \sum_{t=0}^{T-1}
    \mathbb{E}
    \left[
        \|\nabla_{\bm{A}}\phi_{\bm{B}}(\bm{A}_t)\|_F^2
    \right]
    \le
    \phi_{\bm{B}}(\bm{A}_0)-\phi_{\bm{B}}^{\inf}
    +
    \frac{L_A\eta^2T}{2}\sigma^2.
\]
Dividing by \(\eta T/2\) gives
\[
    \frac{1}{T}
    \sum_{t=0}^{T-1}
    \mathbb{E}
    \left[
        \|\nabla_{\bm{A}}\phi_{\bm{B}}(\bm{A}_t)\|_F^2
    \right]
    \le
    \frac{
        2(\phi_{\bm{B}}(\bm{A}_0)-\phi_{\bm{B}}^{\inf})
    }{
        \eta T
    }
    +
    L_A\eta\sigma^2.
\]
Finally, the minimum is no larger than the average:
\[
    \min_{0\le t<T}
    \mathbb{E}
    \left[
        \|\nabla_{\bm{A}}\phi_{\bm{B}}(\bm{A}_t)\|_F^2
    \right]
    \le
    \frac{1}{T}
    \sum_{t=0}^{T-1}
    \mathbb{E}
    \left[
        \|\nabla_{\bm{A}}\phi_{\bm{B}}(\bm{A}_t)\|_F^2
    \right].
\]
This proves the stated bound. Choosing \(\eta=\Theta(T^{-1/2})\), while
maintaining \(\eta\le 1/L_A\), gives the \(O(T^{-1/2})\) rate.
\end{proof}

\clearpage
\section{Experimental protocol and reproducibility}
\label{app:protocol}

\subsection{Datasets and evaluation splits}

We evaluate on Caltech101, DTD, FGVC Aircraft, Food101,
Oxford Flowers102, Oxford-IIIT Pets, and UCF101. UCF101 is
used as a static-image classification benchmark through
extracted video midframes. The few-shot settings use
$K\in\{1,4,16\}$ training examples per class and experiment
seeds $\{1,2,3\}$.

The dataset loaders sample up to $K$ training examples and
$\min(K,4)$ validation examples per class, without replacement.
Few-shot samples are cached separately for each shot count
and seed. When a class contains fewer than the requested
number of examples, the loader retains the available
examples rather than duplicating them.

For base-to-new evaluation, the implementation sorts the
original class-label identifiers and assigns the first
$\lceil C/2\rceil$ classes to the base set and the remaining
classes to the new set, where $C$ is the number of classes.
This class partition is fixed across experiment seeds.
Few-shot sampling precedes class subsampling.

Where available, the loaders use the CoOp-style
\texttt{split\_zhou} JSON files. Otherwise, they construct
splits using dataset-specific fallback procedures.
FGVC Aircraft uses its provided training, validation, and
test annotation lists. The original split JSON files,
few-shot caches, and image manifests were not recovered
in the repository audit, so the exact historical image
assignments cannot be reconstructed from the preserved
results alone.

\subsection{Model and prompt configurations}

The implementation uses the repository's OpenAI-CLIP-derived
model code with the RN50 and ViT-B/16 checkpoints.
The prompt-token width is $d=512$ for both backbones.
This width refers to the text-transformer input representation,
rather than necessarily to the final projected CLIP embedding.

Context tokens are shared across classes. The class name
follows the learned context tokens, corresponding to
\texttt{CLASS\_TOKEN\_POSITION=end}.
Dense CoOp is evaluated with context lengths $m=4$ and $m=16$.
The factorized prompt uses $m=16$ and
\begin{equation}
\label{eq:supp-factorization}
P=BA,\qquad
B\in\mathbb{R}^{m\times r},\quad
A\in\mathbb{R}^{r\times d}.
\end{equation}
The rank sweep uses $r\in\{1,2,4,8\}$.
The transfer experiments use rank 4 with ViT-B/16.
The reported base-to-new comparisons use Dense-16 and
rank-4 or rank-8 joint and fixed-orthogonal-$B$ variants.

Dense CoOp trains the context tensor. Joint factorization
trains both $A$ and $B$. Fixed-$B$ variants train only $A$,
whereas fixed-$A$ transfer variants train only $B$.
The remaining CLIP parameters are frozen.

The trainable parameter counts are $md$ for dense prompts,
$r(m+d)$ for joint factorization, $rd$ for fixed-$B$, and
$mr$ for fixed-$A$. At $m=16$, $d=512$, and $r=4$, these
counts are 8192, 2112, 2048, and 64, respectively.
Dense-4 also has 2048 trainable parameters.
Freezing a factor does not by itself eliminate the need
to retain that factor for inference.

\subsection{Training and initialization}

The preserved publication launch paths select configurations
with SGD, an initial learning rate of 0.002, and 200 training
epochs. The learning-rate schedule is cosine decay with one
epoch of constant warm-up at $10^{-5}$. The inherited SGD
settings are momentum 0.9 and weight decay $5\times10^{-4}$,
with zero dampening and Nesterov momentum disabled.
Training uses cross-entropy loss, a batch size of 32,
and eight data-loader workers. Evaluation uses a batch
size of 100.

Images are processed at $224\times224$ resolution.
Training uses random resized cropping with a default scale
range of $[0.08,1.0]$, random horizontal flipping, and
CLIP normalization. The normalization mean is
$(0.48145466,0.4578275,0.40821073)$ and its standard deviation
is $(0.26862954,0.26130258,0.27577711)$.
Evaluation resizes the shorter image side to 224 pixels
and applies a center crop. Interpolation is bicubic.

The preserved configuration defaults to FP16 computation
and final-step checkpoint selection. The base-to-new
workflow evaluates the selected checkpoint separately on
the seen and unseen class sets. Although the implementation
supports other precision modes and checkpoint-selection
rules, their availability does not establish their use
in the reported experiments. Historical resolved
configurations and invocation logs were not recovered,
so the settings above describe the preserved launch paths.

\paragraph{Dense and joint initialization.}
With an empty context-initialization phrase, dense CoOp
draws the context entries from
$\mathcal{N}(0,0.02^2)$.
For joint factorization, the sampled dense context $P_0$
is converted to FP32 for the singular value decomposition
$P_0=U\Sigma V^\top$. The factors are initialized as
\begin{equation}
\label{eq:supp-balanced-initialization}
B_0=U_r\Sigma_r^{1/2},
\qquad
A_0=\Sigma_r^{1/2}V_r^\top,
\end{equation}
and then cast to the model dtype.
Before dtype rounding, their product is the truncated-SVD
rank-$r$ approximation of $P_0$.

\paragraph{Fixed-basis initialization.}
The SVD-derived basis is the balanced factor
$U_r\Sigma_r^{1/2}$.
The Gaussian basis is sampled in FP32 and rescaled to
match the Frobenius norm of the initialized SVD basis.
The orthogonal variant applies reduced QR decomposition
to a Gaussian matrix and then applies the same norm
matching. Its columns are therefore scaled orthogonal,
rather than necessarily orthonormal.

The learned-basis variant loads the final $B$ from the
matching joint experiment. It requires an additional
joint-training stage and does not apply the random-basis
norm-matching rule. For all four fixed-basis modes,
the trainable factor is initialized by
\begin{equation}
\label{eq:supp-ls-initialization}
A_0=B_{\mathrm{fixed}}^\dagger P_0,
\end{equation}
computed in FP32 and then cast to the model dtype.

\paragraph{Transfer initialization.}
Transfer runs first construct target-seed SVD factors.
They replace the factor designated as frozen with the
source factor and retain the target initialization of
the trainable factor. Random-transfer controls obtain
their frozen factor from an SVD of a separately sampled
random dense prompt. These transfer branches do not
perform the least-squares reinitialization used in the
fixed-basis ablation.

Consequently, the variants do not generally start from
identical prompt products. Matching seed labels also
does not guarantee identical initial dense tensors
across separately launched runs, because earlier
random-number consumption may differ.

\subsection{Random seeds and statistical reporting}

Experiments use seeds $\{1,2,3\}$.
The entry point seeds Python, NumPy, PyTorch, and CUDA
random-number generators. Depending on cache availability,
the seed can affect few-shot sampling, prompt initialization,
training order, and augmentation. CUDA benchmarking is
enabled, so the implementation does not guarantee
deterministic execution.

For an individual dataset and experimental configuration,
we report the mean and sample standard deviation across
the three seeds. In the standard few-shot macro summaries,
dataset accuracies are first averaged within each seed;
the reported mean and sample standard deviation are then
computed across the three seed-level averages.

The preserved fixed-basis aggregate exports instead
summarize 21 dataset-seed observations.
The base-to-new aggregate exports summarize 21 run-level
harmonic means, with
\begin{equation}
\label{eq:supp-harmonic-mean}
H=\frac{2SU}{S+U}
\end{equation}
computed for each run before aggregation.
These pooled standard deviations include between-dataset
variation and should not be interpreted as standard
deviations across three seed-level macro averages.

No inferential significance or equivalence tests were
recovered. Reported standard deviations describe empirical
variation and do not establish statistical equivalence.

\subsection{Factor-geometry measurements}

The implementation compares the column spaces of $B$
and the row spaces of $A$. Given orthonormal bases
$Q_1$ and $Q_2$, let $\sigma_i$ be the singular values
of $Q_1^\top Q_2$, clipped to $[0,1]$, and let $k$
be the smaller retained subspace rank.
The implemented overlap is
\begin{equation}
\label{eq:supp-subspace-overlap}
\operatorname{overlap}(Q_1,Q_2)
=\frac{1}{k}\sum_{i=1}^{k}\sigma_i.
\end{equation}
This is the mean cosine of the principal angles,
rather than their mean squared cosine.
The mean principal angle is computed separately as
$k^{-1}\sum_i\arccos(\sigma_i)$ and expressed in degrees.

The comparisons concern final factors from different
experimental configurations. They are descriptive
comparisons of trained endpoints, with dependencies
between pairs that share an endpoint. They do not
measure displacement from initialization, and therefore
do not establish how little or how much $B$ changes
during training.

The archived prompt-product diagnostics require additional
care. Saved low-precision products can contain rounding
residuals that change their measured numerical rank.
In particular, the archived nominal rank-4 product
records were assigned numerical rank 16 under the
implemented threshold. Product-space diagnostics
therefore cannot be interpreted as measurements of
exact rank-4 subspaces without reconstructing the
products from the factors at suitable precision.

\subsection{Hardware, software, and computational cost}

Experiments were conducted on NVIDIA Tesla V100 GPUs with
32\,GB of device memory. The four-worker launch scripts assign
one visible GPU to each worker and distribute independent
experimental configurations across workers. This supports up to
four concurrent single-GPU training runs; it does not constitute
four-GPU distributed training of an individual model.

A snapshot collected from the cluster environment
recorded Python 3.8.20, PyTorch 2.4.1,
torchvision 0.20.0, NumPy 1.24.3, SciPy 1.10.1,
pandas 2.0.3, and Dassl 0.6.3. PyTorch reported a CUDA 11.8
build and cuDNN version 90100. This is a retrospective
environment snapshot, not an archived environment specification
from the original training jobs.

Historical per-run timing, peak allocated GPU memory, and
complete job-accounting records were not recovered.
We therefore report trainable parameter
counts without interpreting them as measured reductions in
runtime, peak memory, or total computational cost.

\clearpage
\section{Additional experimental results}
\label{app:results}

\subsection{Factorized prompt rank sweep}
\begin{table*}[h]
\centering
\scriptsize
\caption{
Fully trainable factorized CoOp across datasets, shots, ranks, and
backbones. Per-dataset accuracy of the factorization $\bm{P}=\bm{B}\bm{A}$ with both
factors trained, for all seven datasets, the 1-/4-/16-shot settings, ranks
$r\in\{1,2,4,8\}$, and the RN50 and ViT-B/16 CLIP backbones. Values are mean accuracy
(\%) $\pm$ sample standard deviation over three seeds.
}
\label{tab:factorized_coop_results}
\resizebox{\textwidth}{!}{
\begin{tabular}{lccccccccc}
\toprule
\textbf{Dataset}
& \textbf{Shots}
& \makecell{\textbf{Rank 1}\\\textbf{RN50}}
& \makecell{\textbf{Rank 1}\\\textbf{ViT-B/16}}
& \makecell{\textbf{Rank 2}\\\textbf{RN50}}
& \makecell{\textbf{Rank 2}\\\textbf{ViT-B/16}}
& \makecell{\textbf{Rank 4}\\\textbf{RN50}}
& \makecell{\textbf{Rank 4}\\\textbf{ViT-B/16}}
& \makecell{\textbf{Rank 8}\\\textbf{RN50}}
& \makecell{\textbf{Rank 8}\\\textbf{ViT-B/16}}
\\
\midrule
Caltech101 & 1
& $87.36 \pm 0.47$
& $93.25 \pm 0.70$
& $85.48 \pm 2.58$
& $92.60 \pm 1.29$
& $83.45 \pm 3.44$
& $91.48 \pm 2.36$
& $83.71 \pm 2.15$
& $91.14 \pm 1.72$
\\
& 4
& $89.94 \pm 0.07$
& $94.28 \pm 0.59$
& $89.21 \pm 0.35$
& $94.37 \pm 0.13$
& $88.74 \pm 0.45$
& $93.97 \pm 0.55$
& $88.61 \pm 1.02$
& $94.16 \pm 0.32$
\\
& 16
& $91.67 \pm 0.40$
& $95.42 \pm 0.11$
& $91.75 \pm 0.30$
& $95.59 \pm 0.18$
& $91.94 \pm 0.06$
& $95.51 \pm 0.29$
& $91.98 \pm 0.05$
& $95.58 \pm 0.15$
\\
\midrule
DTD & 1
& $43.03 \pm 2.23$
& $48.94 \pm 2.71$
& $42.57 \pm 0.79$
& $49.47 \pm 2.93$
& $41.96 \pm 0.59$
& $50.45 \pm 1.58$
& $43.42 \pm 0.34$
& $50.33 \pm 1.22$
\\
& 4
& $52.15 \pm 1.40$
& $56.34 \pm 2.03$
& $51.46 \pm 1.10$
& $58.25 \pm 2.41$
& $52.44 \pm 1.92$
& $58.51 \pm 0.91$
& $53.33 \pm 0.25$
& $58.25 \pm 1.50$
\\
& 16
& $59.28 \pm 1.74$
& $64.64 \pm 2.31$
& $61.74 \pm 1.37$
& $67.83 \pm 0.36$
& $63.20 \pm 0.75$
& $68.26 \pm 0.20$
& $63.48 \pm 0.41$
& $69.29 \pm 0.65$
\\
\midrule
FGVC Aircraft & 1
& $18.24 \pm 0.60$
& $26.41 \pm 1.23$
& $17.21 \pm 1.67$
& $28.11 \pm 0.53$
& $16.94 \pm 1.01$
& $27.95 \pm 0.93$
& $13.81 \pm 5.26$
& $26.49 \pm 1.41$
\\
& 4
& $19.88 \pm 0.54$
& $29.36 \pm 0.67$
& $15.65 \pm 7.92$
& $31.38 \pm 0.99$
& $21.37 \pm 0.49$
& $31.75 \pm 1.40$
& $22.64 \pm 0.71$
& $30.42 \pm 3.81$
\\
& 16
& $23.74 \pm 0.44$
& $34.96 \pm 1.04$
& $26.70 \pm 0.15$
& $38.17 \pm 0.12$
& $27.92 \pm 0.62$
& $40.46 \pm 0.48$
& $30.31 \pm 0.60$
& $42.06 \pm 0.11$
\\
\midrule
Food101 & 1
& $72.56 \pm 3.44$
& $84.20 \pm 0.68$
& $70.65 \pm 2.03$
& $82.10 \pm 2.09$
& $68.11 \pm 0.69$
& $80.85 \pm 1.63$
& $66.60 \pm 1.92$
& $79.79 \pm 2.36$
\\
& 4
& $76.80 \pm 0.90$
& $85.83 \pm 0.47$
& $75.30 \pm 1.52$
& $85.60 \pm 0.86$
& $72.61 \pm 0.48$
& $84.45 \pm 0.70$
& $70.34 \pm 1.45$
& $82.32 \pm 0.34$
\\
& 16
& $78.96 \pm 0.19$
& $86.87 \pm 0.13$
& $78.49 \pm 0.13$
& $86.20 \pm 0.16$
& $77.32 \pm 0.34$
& $85.78 \pm 0.15$
& $75.97 \pm 0.36$
& $85.00 \pm 0.08$
\\
\midrule
Oxford Flowers & 1
& $67.71 \pm 0.68$
& $75.45 \pm 3.93$
& $68.26 \pm 1.64$
& $76.44 \pm 5.79$
& $69.79 \pm 1.96$
& $79.32 \pm 0.79$
& $71.23 \pm 0.87$
& $80.34 \pm 1.77$
\\
& 4
& $73.27 \pm 1.49$
& $79.55 \pm 2.40$
& $76.23 \pm 2.22$
& $85.44 \pm 2.40$
& $84.59 \pm 0.13$
& $90.03 \pm 0.43$
& $86.40 \pm 1.97$
& $92.11 \pm 2.30$
\\
& 16
& $82.74 \pm 0.90$
& $90.68 \pm 0.65$
& $88.69 \pm 0.48$
& $94.63 \pm 0.19$
& $92.39 \pm 0.43$
& $96.20 \pm 0.29$
& $93.92 \pm 0.14$
& $96.67 \pm 0.32$
\\
\midrule
Oxford Pets & 1
& $86.43 \pm 2.38$
& $90.67 \pm 0.57$
& $83.00 \pm 2.11$
& $90.86 \pm 0.41$
& $82.87 \pm 1.41$
& $91.27 \pm 0.44$
& $81.49 \pm 1.53$
& $90.59 \pm 1.30$
\\
& 4
& $87.38 \pm 1.46$
& $91.80 \pm 1.22$
& $85.15 \pm 1.68$
& $92.32 \pm 0.50$
& $84.36 \pm 1.65$
& $91.81 \pm 0.50$
& $84.70 \pm 2.71$
& $91.78 \pm 0.76$
\\
& 16
& $88.24 \pm 0.57$
& $92.81 \pm 0.53$
& $87.82 \pm 0.40$
& $93.05 \pm 0.52$
& $87.48 \pm 0.88$
& $92.48 \pm 0.57$
& $86.63 \pm 1.00$
& $92.41 \pm 0.55$
\\
\midrule
UCF101 & 1
& $62.03 \pm 0.51$
& $70.11 \pm 0.78$
& $60.27 \pm 1.08$
& $69.60 \pm 0.42$
& $60.35 \pm 0.71$
& $69.13 \pm 0.53$
& $57.52 \pm 2.18$
& $68.81 \pm 2.57$
\\
& 4
& $67.22 \pm 0.74$
& $74.77 \pm 0.73$
& $67.18 \pm 1.11$
& $77.22 \pm 0.46$
& $67.83 \pm 0.58$
& $76.69 \pm 0.37$
& $67.27 \pm 0.58$
& $76.01 \pm 0.33$
\\
& 16
& $70.61 \pm 0.70$
& $78.89 \pm 0.89$
& $73.38 \pm 0.78$
& $80.53 \pm 1.64$
& $75.35 \pm 0.94$
& $81.86 \pm 1.32$
& $74.76 \pm 1.03$
& $82.20 \pm 0.26$
\\
\bottomrule
\end{tabular}
}
\end{table*}

\begin{table*}[h]
\centering
\scriptsize
\caption{
Best factorized CoOp accuracy and deltas. For each dataset and shot setting,
the highest factorized CoOp accuracy across ranks and backbones, with the
rank/backbone that attains it. $\Delta_{\text{Rank4}-\text{Best}}$ is the rank-4
factorization minus this best; $\Delta_{\text{Best}-\text{Ctx4}}$ and
$\Delta_{\text{Best}-\text{Ctx16}}$ are the best minus dense CoOp with 4 and 16
context tokens, respectively. Accuracies are mean (\%) $\pm$ sample standard deviation
over three seeds; deltas are in percentage points
}
\label{tab:factorized_best_delta}
\resizebox{0.9\textwidth}{!}{
\begin{tabular}{llccccc}
\toprule
\textbf{Dataset}
& \textbf{Shots}
& \makecell{\textbf{Best Factorized}\\\textbf{CoOp}}
& \makecell{\textbf{Best}\\\textbf{Rank/Backbone}}
& \makecell{\(\Delta\)\\\textbf{Rank4--Best}}
& \makecell{\(\Delta\)\\\textbf{Best--Ctx4}}
& \makecell{\(\Delta\)\\\textbf{Best--Ctx16}}
\\
\midrule
\multirow{3}{*}{Caltech101}
& 1
& $93.25 \pm 0.70$
& {\scriptsize Rank 1, ViT-B/16}
& $-1.77$
& $+2.75$
& $+3.79$
\\
& 4
& $94.37 \pm 0.13$
& {\scriptsize Rank 2, ViT-B/16}
& $-0.41$
& $+0.51$
& $+0.21$
\\
& 16
& $95.59 \pm 0.18$
& {\scriptsize Rank 2, ViT-B/16}
& $-0.08$
& $+0.26$
& $+0.12$
\\
\midrule
\multirow{3}{*}{DTD}
& 1
& $50.45 \pm 1.58$
& {\scriptsize Rank 4, ViT-B/16}
& $+0.00$
& $+2.05$
& $+1.69$
\\
& 4
& $58.51 \pm 0.91$
& {\scriptsize Rank 4, ViT-B/16}
& $+0.00$
& $-0.96$
& $-0.72$
\\
& 16
& $69.29 \pm 0.65$
& {\scriptsize Rank 8, ViT-B/16}
& $-1.02$
& $+0.65$
& $+0.19$
\\
\midrule
\multirow{3}{*}{FGVC Aircraft}
& 1
& $28.11 \pm 0.53$
& {\scriptsize Rank 2, ViT-B/16}
& $-0.16$
& $+1.65$
& $+4.55$
\\
& 4
& $31.75 \pm 1.40$
& {\scriptsize Rank 4, ViT-B/16}
& $+0.00$
& $-0.21$
& $-0.31$
\\
& 16
& $42.06 \pm 0.11$
& {\scriptsize Rank 8, ViT-B/16}
& $-1.60$
& $+1.80$
& $-1.80$
\\
\midrule
\multirow{3}{*}{Food101}
& 1
& $84.20 \pm 0.68$
& {\scriptsize Rank 1, ViT-B/16}
& $-3.35$
& $+7.80$
& $+5.84$
\\
& 4
& $85.83 \pm 0.47$
& {\scriptsize Rank 1, ViT-B/16}
& $-1.39$
& $+3.83$
& $+3.77$
\\
& 16
& $86.87 \pm 0.13$
& {\scriptsize Rank 1, ViT-B/16}
& $-1.09$
& $+1.77$
& $+2.53$
\\
\midrule
\multirow{3}{*}{Oxford Flowers}
& 1
& $80.34 \pm 1.77$
& {\scriptsize Rank 8, ViT-B/16}
& $-1.02$
& $+1.10$
& $-1.00$
\\
& 4
& $92.11 \pm 2.30$
& {\scriptsize Rank 8, ViT-B/16}
& $-2.08$
& $+2.44$
& $-0.52$
\\
& 16
& $96.67 \pm 0.32$
& {\scriptsize Rank 8, ViT-B/16}
& $-0.47$
& $+0.14$
& $-0.40$
\\
\midrule
\multirow{3}{*}{Oxford Pets}
& 1
& $91.27 \pm 0.44$
& {\scriptsize Rank 4, ViT-B/16}
& $+0.00$
& $+3.54$
& $+3.20$
\\
& 4
& $92.32 \pm 0.50$
& {\scriptsize Rank 2, ViT-B/16}
& $-0.51$
& $+0.82$
& $+1.82$
\\
& 16
& $93.05 \pm 0.52$
& {\scriptsize Rank 2, ViT-B/16}
& $-0.57$
& $+0.72$
& $+1.15$
\\
\midrule
\multirow{3}{*}{UCF101}
& 1
& $70.11 \pm 0.78$
& {\scriptsize Rank 1, ViT-B/16}
& $-0.98$
& $+1.88$
& $+2.15$
\\
& 4
& $77.22 \pm 0.46$
& {\scriptsize Rank 2, ViT-B/16}
& $-0.53$
& $-0.04$
& $+1.02$
\\
& 16
& $82.20 \pm 0.26$
& {\scriptsize Rank 8, ViT-B/16}
& $-0.34$
& $+0.33$
& $+0.03$
\\
\bottomrule
\end{tabular}
}
\end{table*}

\clearpage
\subsection{Factor-transfer results}
\begin{table*}[h]
\centering
\scriptsize
\caption{
Per-pair transfer results for ViT-B/16 rank-4 factorized prompts. For each
ordered source--target pair, one factor is trained on the target while the other is
frozen, giving four modes: freeze $\bm{A}$/train $\bm{B}$ and freeze $\bm{B}$/train
$\bm{A}$, each with the frozen factor either carried from a source-trained model or
drawn at random. Values are mean accuracy (\%) $\pm$ sample standard deviation over
three seeds; the best mode per pair is in \textbf{bold}.
}
\label{tab:transfer_four_modes}
\resizebox{\textwidth}{!}{
\begin{tabular}{llcccc}
\toprule
\textbf{Target Dataset}
& \textbf{Source Dataset}
& \makecell{\textbf{Freeze \(A\),}\\\textbf{Train \(B\)}}
& \makecell{\textbf{Freeze \(B\),}\\\textbf{Train \(A\)}}
& \makecell{\textbf{Freeze Random \(A\),}\\\textbf{Train \(B\)}}
& \makecell{\textbf{Freeze Random \(B\),}\\\textbf{Train \(A\)}}
\\
\midrule
Caltech101
& Oxford Pets
& $94.12 \pm 0.44$
& $\mathbf{95.52 \pm 0.26}$
& $92.51 \pm 0.88$
& $95.51 \pm 0.27$
\\
Caltech101
& UCF101
& $93.38 \pm 0.85$
& $\mathbf{95.52 \pm 0.21}$
& $92.51 \pm 0.88$
& $95.51 \pm 0.27$
\\
\midrule
DTD
& FGVC Aircraft
& $46.55 \pm 1.79$
& $\mathbf{68.86 \pm 0.65}$
& $47.61 \pm 1.56$
& $68.51 \pm 0.80$
\\
DTD
& Oxford Flowers
& $46.53 \pm 1.81$
& $\mathbf{68.83 \pm 0.77}$
& $47.61 \pm 1.56$
& $68.51 \pm 0.80$
\\
\midrule
FGVC Aircraft
& DTD
& $13.28 \pm 10.77$
& $\mathbf{40.07 \pm 0.62}$
& $26.84 \pm 0.46$
& $39.78 \pm 0.54$
\\
FGVC Aircraft
& Oxford Pets
& $17.02 \pm 8.18$
& $\mathbf{40.02 \pm 0.82}$
& $26.84 \pm 0.46$
& $39.78 \pm 0.54$
\\
\midrule
Food101
& Oxford Flowers
& $86.09 \pm 0.90$
& $85.50 \pm 0.27$
& $\mathbf{86.44 \pm 0.30}$
& $85.80 \pm 0.27$
\\
Food101
& UCF101
& $\mathbf{86.47 \pm 0.21}$
& $85.82 \pm 0.22$
& $86.44 \pm 0.30$
& $85.80 \pm 0.27$
\\
\midrule
Oxford Flowers
& DTD
& $70.74 \pm 1.74$
& $\mathbf{96.03 \pm 0.36}$
& $70.12 \pm 0.92$
& $95.96 \pm 0.31$
\\
Oxford Flowers
& Food101
& $70.76 \pm 0.70$
& $\mathbf{96.15 \pm 0.35}$
& $70.12 \pm 0.92$
& $95.96 \pm 0.31$
\\
Oxford Flowers
& Oxford Pets
& $70.22 \pm 0.80$
& $\mathbf{96.07 \pm 0.35}$
& $70.12 \pm 0.92$
& $95.96 \pm 0.31$
\\
\midrule
Oxford Pets
& FGVC Aircraft
& $91.22 \pm 0.78$
& $\mathbf{92.63 \pm 0.38}$
& $91.48 \pm 0.75$
& $92.53 \pm 0.71$
\\
Oxford Pets
& Oxford Flowers
& $91.53 \pm 0.74$
& $\mathbf{92.72 \pm 0.48}$
& $91.48 \pm 0.75$
& $92.53 \pm 0.71$
\\
Oxford Pets
& Caltech101
& $91.55 \pm 0.94$
& $92.39 \pm 0.73$
& $91.48 \pm 0.75$
& $\mathbf{92.53 \pm 0.71}$
\\
\midrule
UCF101
& Caltech101
& $69.69 \pm 1.40$
& $\mathbf{81.88 \pm 0.42}$
& $69.20 \pm 1.13$
& $81.56 \pm 0.58$
\\
UCF101
& Food101
& $70.56 \pm 0.91$
& $\mathbf{81.66 \pm 0.71}$
& $69.20 \pm 1.13$
& $81.56 \pm 0.58$
\\
\bottomrule
\end{tabular}
}
\end{table*}

\begin{table*}[h]
\centering
\scriptsize
\caption{
Best transfer result per source--target pair (ViT-B/16, rank 4).
$\Delta_{\text{Rand}-\text{Freeze}}$ is the random-frozen mode minus its
source-frozen counterpart; $\Delta_{\text{Best}-\text{Dense}}$ and
$\Delta_{\text{Best}-\text{Factorized}}$ compare the best transfer result against
dense CoOp and the best from-scratch factorized model trained on the same target.
Accuracies are mean (\%) $\pm$ sample standard deviation over three seeds; deltas are
in percentage points.
}
\label{tab:transfer_best_delta}
\resizebox{\textwidth}{!}{
\begin{tabular}{llccccc}
\toprule
\textbf{Target Dataset}
& \textbf{Source Dataset}
& \makecell{\textbf{Best Transfer}\\\textbf{Result}}
& \makecell{\textbf{Best}\\\textbf{Mode}}
& \makecell{\(\Delta\)\\\textbf{Rand--Freeze}}
& \makecell{\(\Delta\)\\\textbf{Best--Dense}}
& \makecell{\(\Delta\)\\\textbf{Best--Factorized}}
\\
\midrule
Caltech101
& Oxford Pets
& $95.52 \pm 0.26$
& {\scriptsize Freeze \(B\), Train \(A\)}
& $-0.01$
& $+0.05$
& $-0.07$
\\
Caltech101
& UCF101
& $95.52 \pm 0.21$
& {\scriptsize Freeze \(B\), Train \(A\)}
& $-0.01$
& $+0.05$
& $-0.07$
\\
\midrule
DTD
& FGVC Aircraft
& $68.86 \pm 0.65$
& {\scriptsize Freeze \(B\), Train \(A\)}
& $-0.35$
& $-0.24$
& $-0.43$
\\
DTD
& Oxford Flowers
& $68.83 \pm 0.77$
& {\scriptsize Freeze \(B\), Train \(A\)}
& $-0.32$
& $-0.27$
& $-0.46$
\\
\midrule
FGVC Aircraft
& DTD
& $40.07 \pm 0.62$
& {\scriptsize Freeze \(B\), Train \(A\)}
& $-0.28$
& $-3.80$
& $-2.00$
\\
FGVC Aircraft
& Oxford Pets
& $40.02 \pm 0.82$
& {\scriptsize Freeze \(B\), Train \(A\)}
& $-0.24$
& $-3.84$
& $-2.04$
\\
\midrule
Food101
& Oxford Flowers
& $86.44 \pm 0.30$
& {\scriptsize Freeze Random \(A\), Train \(B\)}
& $+0.30$
& $+1.34$
& $-0.43$
\\
Food101
& UCF101
& $86.47 \pm 0.21$
& {\scriptsize Freeze \(A\), Train \(B\)}
& $-0.02$
& $+1.37$
& $-0.40$
\\
\midrule
Oxford Flowers
& DTD
& $96.03 \pm 0.36$
& {\scriptsize Freeze \(B\), Train \(A\)}
& $-0.07$
& $-1.04$
& $-0.64$
\\
Oxford Flowers
& Food101
& $96.15 \pm 0.35$
& {\scriptsize Freeze \(B\), Train \(A\)}
& $-0.19$
& $-0.91$
& $-0.52$
\\
Oxford Flowers
& Oxford Pets
& $96.07 \pm 0.35$
& {\scriptsize Freeze \(B\), Train \(A\)}
& $-0.11$
& $-1.00$
& $-0.60$
\\
\midrule
Oxford Pets
& FGVC Aircraft
& $92.63 \pm 0.38$
& {\scriptsize Freeze \(B\), Train \(A\)}
& $-0.10$
& $+0.30$
& $-0.42$
\\
Oxford Pets
& Oxford Flowers
& $92.72 \pm 0.48$
& {\scriptsize Freeze \(B\), Train \(A\)}
& $-0.19$
& $+0.39$
& $-0.33$
\\
Oxford Pets
& Caltech101
& $92.53 \pm 0.71$
& {\scriptsize Freeze Random \(B\), Train \(A\)}
& $+0.14$
& $+0.20$
& $-0.52$
\\
\midrule
UCF101
& Caltech101
& $81.88 \pm 0.42$
& {\scriptsize Freeze \(B\), Train \(A\)}
& $-0.31$
& $-0.29$
& $-0.33$
\\
UCF101
& Food101
& $81.66 \pm 0.71$
& {\scriptsize Freeze \(B\), Train \(A\)}
& $-0.10$
& $-0.51$
& $-0.54$
\\
\bottomrule
\end{tabular}
}
\end{table*}

\clearpage
\subsection{Fixed-basis ablations}


\begin{table*}[h]
\centering
\scriptsize
\caption{
Token-side basis ablation, RN50 at rank 1. Mean accuracy (\%) over seven
datasets for the trainable factorization (both factors learned) and four
fixed-$\bm{B}$ variants that freeze the token-side basis and train only $\bm{A}$:
Gaussian, learned-then-frozen, orthogonal, and SVD-derived bases. Dense CoOp is shown
as a reference. Values are mean $\pm$ sample standard deviation over three seeds;
\textbf{bold} marks the best mean among the trainable and fixed-$\bm{B}$ variants.
}
\label{tab:b_ablation_rank1_rn50}
\resizebox{\textwidth}{!}{%
\begin{tabular}{llcccccc}
\toprule
\textbf{Dataset} & \textbf{Shots} & \makecell{\textbf{Dense CoOp}\\\textbf{Best}} & \makecell{\textbf{Trainable}\\\textbf{\(BA\)}} & \makecell{\textbf{Gaussian}\\\textbf{\(B\)}} & \makecell{\textbf{Learned}\\\textbf{\(B\)}} & \makecell{\textbf{Orthogonal}\\\textbf{\(B\)}} & \makecell{\textbf{SVD}\\\textbf{\(B\)}} \\
\midrule
\multirow{3}{*}{Caltech101} & 1 & $83.90 \pm 3.20$ & $87.36 \pm 0.47$ & $87.52 \pm 1.00$ & $87.76 \pm 1.89$ & $87.25 \pm 1.01$ & $\mathbf{87.87} \pm \mathbf{0.25}$ \\
 & 4 & $88.63 \pm 0.49$ & $89.94 \pm 0.07$ & $90.01 \pm 0.21$ & $90.25 \pm 0.12$ & $\mathbf{90.29} \pm \mathbf{0.15}$ & $89.82 \pm 0.15$ \\
 & 16 & $91.83 \pm 0.38$ & $\mathbf{91.67} \pm \mathbf{0.40}$ & $91.35 \pm 0.23$ & $91.16 \pm 0.18$ & $91.35 \pm 0.47$ & $91.14 \pm 0.31$ \\
\midrule
\multirow{3}{*}{DTD} & 1 & $42.93 \pm 1.18$ & $\mathbf{43.03} \pm \mathbf{2.23}$ & $42.49 \pm 1.92$ & $42.10 \pm 0.45$ & $42.59 \pm 1.66$ & $42.43 \pm 2.54$ \\
 & 4 & $52.90 \pm 0.30$ & $\mathbf{52.15} \pm \mathbf{1.40}$ & $51.00 \pm 1.53$ & $51.64 \pm 0.89$ & $50.33 \pm 2.39$ & $50.95 \pm 1.24$ \\
 & 16 & $63.03 \pm 1.04$ & $59.28 \pm 1.74$ & $58.73 \pm 0.65$ & $59.30 \pm 1.77$ & $58.79 \pm 0.92$ & $\mathbf{59.89} \pm \mathbf{1.90}$ \\
\midrule
\multirow{3}{*}{FGVC Aircraft} & 1 & $18.17 \pm 1.50$ & $\mathbf{18.24} \pm \mathbf{0.60}$ & $17.63 \pm 0.57$ & $17.29 \pm 0.21$ & $17.74 \pm 0.45$ & $18.23 \pm 0.35$ \\
 & 4 & $22.60 \pm 1.05$ & $19.88 \pm 0.54$ & $20.65 \pm 1.01$ & $19.31 \pm 1.58$ & $20.87 \pm 1.30$ & $\mathbf{21.00} \pm \mathbf{0.34}$ \\
 & 16 & $31.30 \pm 0.62$ & $23.74 \pm 0.44$ & $24.00 \pm 0.21$ & $\mathbf{24.02} \pm \mathbf{0.89}$ & $23.84 \pm 0.22$ & $23.30 \pm 0.75$ \\
\midrule
\multirow{3}{*}{Food101} & 1 & $63.70 \pm 0.52$ & $72.56 \pm 3.44$ & $73.89 \pm 2.08$ & $73.99 \pm 1.24$ & $\mathbf{74.08} \pm \mathbf{2.00}$ & $74.05 \pm 2.34$ \\
 & 4 & $70.07 \pm 3.51$ & $76.80 \pm 0.90$ & $77.37 \pm 0.30$ & $76.37 \pm 1.59$ & $\mathbf{77.46} \pm \mathbf{0.20}$ & $76.70 \pm 0.78$ \\
 & 16 & $76.17 \pm 0.31$ & $78.96 \pm 0.19$ & $78.81 \pm 0.14$ & $\mathbf{78.99} \pm \mathbf{0.31}$ & $78.79 \pm 0.08$ & $78.89 \pm 0.26$ \\
\midrule
\multirow{3}{*}{Oxford Flowers} & 1 & $70.93 \pm 1.96$ & $67.71 \pm 0.68$ & $\mathbf{68.80} \pm \mathbf{1.77}$ & $66.40 \pm 0.56$ & $68.18 \pm 0.82$ & $68.02 \pm 1.04$ \\
 & 4 & $87.13 \pm 2.04$ & $73.27 \pm 1.49$ & $73.88 \pm 2.12$ & $\mathbf{74.90} \pm \mathbf{1.15}$ & $74.00 \pm 2.35$ & $72.91 \pm 1.09$ \\
 & 16 & $94.70 \pm 0.10$ & $\mathbf{82.74} \pm \mathbf{0.90}$ & $80.39 \pm 2.20$ & $81.30 \pm 0.43$ & $81.17 \pm 1.53$ & $81.39 \pm 1.05$ \\
\midrule
\multirow{3}{*}{Oxford Pets} & 1 & $79.43 \pm 0.81$ & $86.43 \pm 2.38$ & $87.26 \pm 0.82$ & $86.54 \pm 1.38$ & $87.31 \pm 0.57$ & $\mathbf{87.78} \pm \mathbf{0.61}$ \\
 & 4 & $85.70 \pm 1.13$ & $87.38 \pm 1.46$ & $88.07 \pm 1.14$ & $87.58 \pm 0.95$ & $88.03 \pm 0.75$ & $\mathbf{88.10} \pm \mathbf{0.99}$ \\
 & 16 & $87.40 \pm 0.69$ & $88.24 \pm 0.57$ & $88.28 \pm 0.96$ & $88.18 \pm 0.57$ & $88.34 \pm 0.76$ & $\mathbf{88.88} \pm \mathbf{1.02}$ \\
\midrule
\multirow{3}{*}{UCF101} & 1 & $58.20 \pm 3.04$ & $62.03 \pm 0.51$ & $62.28 \pm 0.23$ & $62.04 \pm 1.49$ & $\mathbf{62.60} \pm \mathbf{0.68}$ & $62.49 \pm 0.81$ \\
 & 4 & $67.90 \pm 0.44$ & $67.22 \pm 0.74$ & $\mathbf{67.42} \pm \mathbf{0.33}$ & $65.86 \pm 0.89$ & $67.18 \pm 0.33$ & $66.01 \pm 1.40$ \\
 & 16 & $75.63 \pm 0.40$ & $\mathbf{70.61} \pm \mathbf{0.70}$ & $69.77 \pm 0.89$ & $70.23 \pm 0.11$ & $70.17 \pm 0.22$ & $69.70 \pm 0.69$ \\
\bottomrule
\end{tabular}%
}
\end{table*}

\begin{table*}[h]
\centering
\scriptsize
\caption{
Token-side basis ablation, ViT-B/16 at rank 1. Mean accuracy (\%) over seven
datasets for the trainable factorization (both factors learned) and four
fixed-$\bm{B}$ variants that freeze the token-side basis and train only $\bm{A}$:
Gaussian, learned-then-frozen, orthogonal, and SVD-derived bases. Dense CoOp is shown
as a reference. Values are mean $\pm$ sample standard deviation over three seeds;
\textbf{bold} marks the best mean among the trainable and fixed-$\bm{B}$ variants.
}
\label{tab:b_ablation_rank1_vitb16}
\resizebox{\textwidth}{!}{%
\begin{tabular}{llcccccc}
\toprule
\textbf{Dataset} & \textbf{Shots} & \makecell{\textbf{Dense CoOp}\\\textbf{Best}} & \makecell{\textbf{Trainable}\\\textbf{\(BA\)}} & \makecell{\textbf{Gaussian}\\\textbf{\(B\)}} & \makecell{\textbf{Learned}\\\textbf{\(B\)}} & \makecell{\textbf{Orthogonal}\\\textbf{\(B\)}} & \makecell{\textbf{SVD}\\\textbf{\(B\)}} \\
\midrule
\multirow{3}{*}{Caltech101} & 1 & $90.50 \pm 2.07$ & $93.25 \pm 0.70$ & $\mathbf{93.36} \pm \mathbf{0.55}$ & $92.58 \pm 2.12$ & $\mathbf{93.36} \pm \mathbf{0.55}$ & $92.74 \pm 0.88$ \\
 & 4 & $94.17 \pm 0.15$ & $94.28 \pm 0.59$ & $94.47 \pm 0.38$ & $94.48 \pm 0.19$ & $94.47 \pm 0.38$ & $\mathbf{94.73} \pm \mathbf{0.21}$ \\
 & 16 & $95.47 \pm 0.06$ & $\mathbf{95.42} \pm \mathbf{0.11}$ & $95.21 \pm 0.37$ & $95.13 \pm 0.16$ & $95.21 \pm 0.37$ & $94.83 \pm 0.06$ \\
\midrule
\multirow{3}{*}{DTD} & 1 & $48.77 \pm 1.36$ & $48.94 \pm 2.71$ & $48.84 \pm 2.83$ & $49.78 \pm 2.24$ & $48.84 \pm 2.83$ & $\mathbf{50.81} \pm \mathbf{0.48}$ \\
 & 4 & $59.47 \pm 0.67$ & $56.34 \pm 2.03$ & $57.21 \pm 2.03$ & $\mathbf{57.70} \pm \mathbf{1.07}$ & $57.21 \pm 2.03$ & $56.74 \pm 1.78$ \\
 & 16 & $69.10 \pm 0.44$ & $64.64 \pm 2.31$ & $\mathbf{65.86} \pm \mathbf{1.45}$ & $64.52 \pm 0.59$ & $\mathbf{65.86} \pm \mathbf{1.45}$ & $65.03 \pm 0.65$ \\
\midrule
\multirow{3}{*}{FGVC Aircraft} & 1 & $26.47 \pm 2.49$ & $26.41 \pm 1.23$ & $26.65 \pm 1.64$ & $24.77 \pm 2.16$ & $26.65 \pm 1.64$ & $\mathbf{27.29} \pm \mathbf{1.20}$ \\
 & 4 & $32.07 \pm 1.63$ & $29.36 \pm 0.67$ & $\mathbf{29.61} \pm \mathbf{1.42}$ & $24.20 \pm 8.21$ & $\mathbf{29.61} \pm \mathbf{1.42}$ & $28.17 \pm 0.06$ \\
 & 16 & $43.87 \pm 0.96$ & $\mathbf{34.96} \pm \mathbf{1.04}$ & $34.71 \pm 0.45$ & $34.55 \pm 0.40$ & $34.71 \pm 0.45$ & $34.76 \pm 0.47$ \\
\midrule
\multirow{3}{*}{Food101} & 1 & $78.37 \pm 1.66$ & $\mathbf{84.20} \pm \mathbf{0.68}$ & $82.69 \pm 3.02$ & $83.42 \pm 1.38$ & $82.69 \pm 3.02$ & $84.15 \pm 1.77$ \\
 & 4 & $82.07 \pm 1.67$ & $85.83 \pm 0.47$ & $84.67 \pm 0.76$ & $85.89 \pm 0.96$ & $84.67 \pm 0.76$ & $\mathbf{86.26} \pm \mathbf{0.23}$ \\
 & 16 & $85.10 \pm 0.10$ & $86.87 \pm 0.13$ & $\mathbf{87.02} \pm \mathbf{0.11}$ & $86.86 \pm 0.09$ & $\mathbf{87.02} \pm \mathbf{0.11}$ & $86.92 \pm 0.05$ \\
\midrule
\multirow{3}{*}{Oxford Flowers} & 1 & $81.33 \pm 0.35$ & $\mathbf{75.45} \pm \mathbf{3.93}$ & $74.42 \pm 2.73$ & $72.66 \pm 2.26$ & $74.42 \pm 2.73$ & $74.18 \pm 3.00$ \\
 & 4 & $92.63 \pm 1.11$ & $79.55 \pm 2.40$ & $\mathbf{82.27} \pm \mathbf{1.98}$ & $78.74 \pm 1.83$ & $\mathbf{82.27} \pm \mathbf{1.98}$ & $79.52 \pm 2.03$ \\
 & 16 & $97.07 \pm 0.38$ & $\mathbf{90.68} \pm \mathbf{0.65}$ & $90.34 \pm 1.52$ & $89.65 \pm 0.81$ & $90.34 \pm 1.52$ & $88.89 \pm 1.24$ \\
\midrule
\multirow{3}{*}{Oxford Pets} & 1 & $88.07 \pm 1.67$ & $90.67 \pm 0.57$ & $\mathbf{91.21} \pm \mathbf{0.59}$ & $90.32 \pm 0.23$ & $\mathbf{91.21} \pm \mathbf{0.59}$ & $90.86 \pm 0.13$ \\
 & 4 & $91.50 \pm 1.30$ & $91.80 \pm 1.22$ & $92.30 \pm 0.25$ & $92.22 \pm 0.96$ & $92.30 \pm 0.25$ & $\mathbf{92.32} \pm \mathbf{0.70}$ \\
 & 16 & $92.33 \pm 0.83$ & $92.81 \pm 0.53$ & $\mathbf{92.92} \pm \mathbf{0.49}$ & $92.75 \pm 0.49$ & $\mathbf{92.92} \pm \mathbf{0.49}$ & $92.84 \pm 0.07$ \\
\midrule
\multirow{3}{*}{UCF101} & 1 & $68.23 \pm 1.30$ & $70.11 \pm 0.78$ & $69.90 \pm 1.25$ & $\mathbf{70.19} \pm \mathbf{2.51}$ & $69.90 \pm 1.25$ & $69.96 \pm 1.12$ \\
 & 4 & $77.27 \pm 0.74$ & $74.77 \pm 0.73$ & $\mathbf{75.73} \pm \mathbf{0.85}$ & $75.44 \pm 1.19$ & $\mathbf{75.73} \pm \mathbf{0.85}$ & $75.30 \pm 0.62$ \\
 & 16 & $82.17 \pm 0.67$ & $78.89 \pm 0.89$ & $79.08 \pm 0.80$ & $78.02 \pm 0.96$ & $79.08 \pm 0.80$ & $\mathbf{79.10} \pm \mathbf{0.44}$ \\
\bottomrule
\end{tabular}%
}
\end{table*}

\begin{table*}[h]
\centering
\scriptsize
\caption{
Token-side basis ablation, RN50 at rank 2. Mean accuracy (\%) over seven
datasets for the trainable factorization (both factors learned) and four
fixed-$\bm{B}$ variants that freeze the token-side basis and train only $\bm{A}$:
Gaussian, learned-then-frozen, orthogonal, and SVD-derived bases. Dense CoOp is shown
as a reference. Values are mean $\pm$ sample standard deviation over three seeds;
\textbf{bold} marks the best mean among the trainable and fixed-$\bm{B}$ variants.
}
\label{tab:b_ablation_rank2_rn50}
\resizebox{\textwidth}{!}{%
\begin{tabular}{llcccccc}
\toprule
\textbf{Dataset} & \textbf{Shots} & \makecell{\textbf{Dense CoOp}\\\textbf{Best}} & \makecell{\textbf{Trainable}\\\textbf{\(BA\)}} & \makecell{\textbf{Gaussian}\\\textbf{\(B\)}} & \makecell{\textbf{Learned}\\\textbf{\(B\)}} & \makecell{\textbf{Orthogonal}\\\textbf{\(B\)}} & \makecell{\textbf{SVD}\\\textbf{\(B\)}} \\
\midrule
\multirow{3}{*}{Caltech101} & 1 & $83.90 \pm 3.20$ & $85.48 \pm 2.58$ & $\mathbf{87.57} \pm \mathbf{1.54}$ & $85.57 \pm 0.77$ & $87.10 \pm 1.20$ & $85.37 \pm 0.40$ \\
 & 4 & $88.63 \pm 0.49$ & $89.21 \pm 0.35$ & $\mathbf{90.07} \pm \mathbf{0.38}$ & $89.36 \pm 0.95$ & $88.92 \pm 0.99$ & $89.18 \pm 0.81$ \\
 & 16 & $91.83 \pm 0.38$ & $91.75 \pm 0.30$ & $91.39 \pm 0.12$ & $\mathbf{91.81} \pm \mathbf{0.35}$ & $91.39 \pm 0.19$ & $91.14 \pm 0.13$ \\
\midrule
\multirow{3}{*}{DTD} & 1 & $42.93 \pm 1.18$ & $\mathbf{42.57} \pm \mathbf{0.79}$ & $40.94 \pm 0.25$ & $42.06 \pm 1.10$ & $42.42 \pm 0.33$ & $41.78 \pm 1.69$ \\
 & 4 & $52.90 \pm 0.30$ & $51.46 \pm 1.10$ & $53.39 \pm 1.01$ & $50.77 \pm 1.38$ & $\mathbf{53.59} \pm \mathbf{0.74}$ & $51.26 \pm 1.97$ \\
 & 16 & $63.03 \pm 1.04$ & $61.74 \pm 1.37$ & $61.76 \pm 0.67$ & $\mathbf{62.55} \pm \mathbf{1.11}$ & $62.43 \pm 0.69$ & $62.16 \pm 1.95$ \\
\midrule
\multirow{3}{*}{FGVC Aircraft} & 1 & $18.17 \pm 1.50$ & $17.21 \pm 1.67$ & $17.31 \pm 0.83$ & $\mathbf{17.93} \pm \mathbf{1.57}$ & $17.02 \pm 0.32$ & $17.22 \pm 1.09$ \\
 & 4 & $22.60 \pm 1.05$ & $15.65 \pm 7.92$ & $16.85 \pm 7.05$ & $18.53 \pm 5.55$ & $\mathbf{20.17} \pm \mathbf{1.01}$ & $16.71 \pm 5.65$ \\
 & 16 & $31.30 \pm 0.62$ & $\mathbf{26.70} \pm \mathbf{0.15}$ & $25.96 \pm 0.28$ & $26.44 \pm 0.67$ & $25.42 \pm 0.84$ & $25.89 \pm 0.55$ \\
\midrule
\multirow{3}{*}{Food101} & 1 & $63.70 \pm 0.52$ & $70.65 \pm 2.03$ & $71.38 \pm 1.84$ & $\mathbf{72.99} \pm \mathbf{3.60}$ & $71.63 \pm 1.78$ & $70.63 \pm 2.72$ \\
 & 4 & $70.07 \pm 3.51$ & $\mathbf{75.30} \pm \mathbf{1.52}$ & $74.33 \pm 1.44$ & $74.18 \pm 0.88$ & $74.69 \pm 1.66$ & $74.63 \pm 1.13$ \\
 & 16 & $76.17 \pm 0.31$ & $78.49 \pm 0.13$ & $78.47 \pm 0.21$ & $\mathbf{78.61} \pm \mathbf{0.14}$ & $78.59 \pm 0.15$ & $78.33 \pm 0.19$ \\
\midrule
\multirow{3}{*}{Oxford Flowers} & 1 & $70.93 \pm 1.96$ & $68.26 \pm 1.64$ & $69.14 \pm 1.85$ & $\mathbf{69.29} \pm \mathbf{1.10}$ & $68.59 \pm 0.59$ & $69.01 \pm 0.93$ \\
 & 4 & $87.13 \pm 2.04$ & $76.23 \pm 2.22$ & $\mathbf{79.16} \pm \mathbf{2.81}$ & $78.16 \pm 1.74$ & $77.59 \pm 3.06$ & $76.98 \pm 0.94$ \\
 & 16 & $94.70 \pm 0.10$ & $\mathbf{88.69} \pm \mathbf{0.48}$ & $\mathbf{88.69} \pm \mathbf{0.47}$ & $87.81 \pm 0.27$ & $88.20 \pm 0.09$ & $87.72 \pm 0.55$ \\
\midrule
\multirow{3}{*}{Oxford Pets} & 1 & $79.43 \pm 0.81$ & $83.00 \pm 2.11$ & $84.55 \pm 1.57$ & $82.98 \pm 2.26$ & $\mathbf{84.68} \pm \mathbf{1.25}$ & $84.58 \pm 1.38$ \\
 & 4 & $85.70 \pm 1.13$ & $85.15 \pm 1.68$ & $86.35 \pm 0.84$ & $85.90 \pm 0.98$ & $\mathbf{86.75} \pm \mathbf{1.79}$ & $86.68 \pm 0.88$ \\
 & 16 & $87.40 \pm 0.69$ & $87.82 \pm 0.40$ & $\mathbf{88.27} \pm \mathbf{0.53}$ & $88.05 \pm 1.34$ & $87.60 \pm 0.81$ & $87.84 \pm 0.37$ \\
\midrule
\multirow{3}{*}{UCF101} & 1 & $58.20 \pm 3.04$ & $60.27 \pm 1.08$ & $\mathbf{61.25} \pm \mathbf{0.69}$ & $59.61 \pm 1.54$ & $60.56 \pm 1.25$ & $59.60 \pm 1.42$ \\
 & 4 & $67.90 \pm 0.44$ & $67.18 \pm 1.11$ & $\mathbf{67.89} \pm \mathbf{0.44}$ & $67.51 \pm 1.28$ & $67.76 \pm 1.10$ & $67.79 \pm 0.96$ \\
 & 16 & $75.63 \pm 0.40$ & $\mathbf{73.38} \pm \mathbf{0.78}$ & $72.62 \pm 1.76$ & $72.59 \pm 1.51$ & $73.27 \pm 1.22$ & $72.22 \pm 0.69$ \\
\bottomrule
\end{tabular}%
}
\end{table*}

\begin{table*}[h]
\centering
\scriptsize
\caption{
Token-side basis ablation, ViT-B/16 at rank 2. Mean accuracy (\%) over seven
datasets for the trainable factorization (both factors learned) and four
fixed-$\bm{B}$ variants that freeze the token-side basis and train only $\bm{A}$:
Gaussian, learned-then-frozen, orthogonal, and SVD-derived bases. Dense CoOp is shown
as a reference. Values are mean $\pm$ sample standard deviation over three seeds;
\textbf{bold} marks the best mean among the trainable and fixed-$\bm{B}$ variants.
}
\label{tab:b_ablation_rank2_vitb16}
\resizebox{\textwidth}{!}{%
\begin{tabular}{llcccccc}
\toprule
\textbf{Dataset} & \textbf{Shots} & \makecell{\textbf{Dense CoOp}\\\textbf{Best}} & \makecell{\textbf{Trainable}\\\textbf{\(BA\)}} & \makecell{\textbf{Gaussian}\\\textbf{\(B\)}} & \makecell{\textbf{Learned}\\\textbf{\(B\)}} & \makecell{\textbf{Orthogonal}\\\textbf{\(B\)}} & \makecell{\textbf{SVD}\\\textbf{\(B\)}} \\
\midrule
\multirow{3}{*}{Caltech101} & 1 & $90.50 \pm 2.07$ & $92.60 \pm 1.29$ & $\mathbf{92.70} \pm \mathbf{0.52}$ & $91.97 \pm 1.36$ & $91.93 \pm 2.29$ & $92.17 \pm 1.10$ \\
 & 4 & $94.17 \pm 0.15$ & $94.37 \pm 0.13$ & $94.58 \pm 0.10$ & $94.37 \pm 0.58$ & $\mathbf{94.66} \pm \mathbf{0.28}$ & $94.40 \pm 0.16$ \\
 & 16 & $95.47 \pm 0.06$ & $95.59 \pm 0.18$ & $95.32 \pm 0.06$ & $95.20 \pm 0.33$ & $\mathbf{95.62} \pm \mathbf{0.20}$ & $95.36 \pm 0.43$ \\
\midrule
\multirow{3}{*}{DTD} & 1 & $48.77 \pm 1.36$ & $49.47 \pm 2.93$ & $49.13 \pm 2.78$ & $\mathbf{50.32} \pm \mathbf{0.56}$ & $49.68 \pm 2.53$ & $49.67 \pm 2.00$ \\
 & 4 & $59.47 \pm 0.67$ & $58.25 \pm 2.41$ & $57.72 \pm 1.11$ & $\mathbf{58.88} \pm \mathbf{0.33}$ & $58.22 \pm 1.59$ & $58.31 \pm 0.95$ \\
 & 16 & $69.10 \pm 0.44$ & $67.83 \pm 0.36$ & $\mathbf{68.30} \pm \mathbf{1.14}$ & $67.16 \pm 0.92$ & $66.77 \pm 0.81$ & $67.47 \pm 0.83$ \\
\midrule
\multirow{3}{*}{FGVC Aircraft} & 1 & $26.47 \pm 2.49$ & $\mathbf{28.11} \pm \mathbf{0.53}$ & $26.56 \pm 0.06$ & $27.00 \pm 1.20$ & $27.57 \pm 0.81$ & $27.76 \pm 0.39$ \\
 & 4 & $32.07 \pm 1.63$ & $\mathbf{31.38} \pm \mathbf{0.99}$ & $25.65 \pm 9.02$ & $23.49 \pm 10.20$ & $26.67 \pm 6.65$ & $30.68 \pm 1.11$ \\
 & 16 & $43.87 \pm 0.96$ & $\mathbf{38.17} \pm \mathbf{0.12}$ & $37.02 \pm 0.50$ & $37.81 \pm 0.42$ & $36.64 \pm 0.68$ & $37.62 \pm 0.21$ \\
\midrule
\multirow{3}{*}{Food101} & 1 & $78.37 \pm 1.66$ & $82.10 \pm 2.09$ & $\mathbf{84.75} \pm \mathbf{0.78}$ & $82.89 \pm 0.40$ & $84.36 \pm 0.83$ & $82.47 \pm 0.89$ \\
 & 4 & $82.07 \pm 1.67$ & $85.60 \pm 0.86$ & $\mathbf{85.95} \pm \mathbf{0.60}$ & $83.84 \pm 0.88$ & $85.94 \pm 0.36$ & $85.78 \pm 0.46$ \\
 & 16 & $85.10 \pm 0.10$ & $86.20 \pm 0.16$ & $\mathbf{86.61} \pm \mathbf{0.28}$ & $86.26 \pm 0.21$ & $86.52 \pm 0.33$ & $86.47 \pm 0.21$ \\
\midrule
\multirow{3}{*}{Oxford Flowers} & 1 & $81.33 \pm 0.35$ & $76.44 \pm 5.79$ & $74.54 \pm 4.67$ & $\mathbf{78.25} \pm \mathbf{2.49}$ & $74.81 \pm 4.32$ & $74.88 \pm 5.72$ \\
 & 4 & $92.63 \pm 1.11$ & $85.44 \pm 2.40$ & $\mathbf{86.16} \pm \mathbf{1.68}$ & $85.65 \pm 1.40$ & $86.11 \pm 2.41$ & $85.33 \pm 1.69$ \\
 & 16 & $97.07 \pm 0.38$ & $\mathbf{94.63} \pm \mathbf{0.19}$ & $94.17 \pm 0.63$ & $94.00 \pm 0.55$ & $94.09 \pm 0.35$ & $94.34 \pm 0.61$ \\
\midrule
\multirow{3}{*}{Oxford Pets} & 1 & $88.07 \pm 1.67$ & $90.86 \pm 0.41$ & $90.47 \pm 1.36$ & $90.57 \pm 0.53$ & $90.31 \pm 1.37$ & $\mathbf{90.89} \pm \mathbf{0.61}$ \\
 & 4 & $91.50 \pm 1.30$ & $92.32 \pm 0.50$ & $92.62 \pm 0.42$ & $92.30 \pm 0.44$ & $\mathbf{92.66} \pm \mathbf{0.70}$ & $91.83 \pm 1.39$ \\
 & 16 & $92.33 \pm 0.83$ & $93.05 \pm 0.52$ & $92.83 \pm 0.11$ & $93.00 \pm 0.16$ & $92.79 \pm 0.31$ & $\mathbf{93.30} \pm \mathbf{0.33}$ \\
\midrule
\multirow{3}{*}{UCF101} & 1 & $68.23 \pm 1.30$ & $69.60 \pm 0.42$ & $70.81 \pm 0.90$ & $\mathbf{71.66} \pm \mathbf{1.02}$ & $71.42 \pm 0.37$ & $70.74 \pm 1.43$ \\
 & 4 & $77.27 \pm 0.74$ & $\mathbf{77.22} \pm \mathbf{0.46}$ & $76.60 \pm 0.31$ & $76.64 \pm 0.96$ & $76.57 \pm 0.32$ & $76.88 \pm 0.47$ \\
 & 16 & $82.17 \pm 0.67$ & $80.53 \pm 1.64$ & $\mathbf{80.76} \pm \mathbf{0.97}$ & $80.17 \pm 0.66$ & $79.95 \pm 0.92$ & $80.62 \pm 0.43$ \\
\bottomrule
\end{tabular}%
}
\end{table*}

\begin{table*}[h]
\centering
\scriptsize
\caption{
Token-side basis ablation, RN50 at rank 4. Mean accuracy (\%) over seven
datasets for the trainable factorization (both factors learned) and four
fixed-$\bm{B}$ variants that freeze the token-side basis and train only $\bm{A}$:
Gaussian, learned-then-frozen, orthogonal, and SVD-derived bases. Dense CoOp is shown
as a reference. Values are mean $\pm$ sample standard deviation over three seeds;
\textbf{bold} marks the best mean among the trainable and fixed-$\bm{B}$ variants.
}
\label{tab:b_ablation_rank4_rn50}
\resizebox{\textwidth}{!}{%
\begin{tabular}{llcccccc}
\toprule
\textbf{Dataset} & \textbf{Shots} & \makecell{\textbf{Dense CoOp}\\\textbf{Best}} & \makecell{\textbf{Trainable}\\\textbf{\(BA\)}} & \makecell{\textbf{Gaussian}\\\textbf{\(B\)}} & \makecell{\textbf{Learned}\\\textbf{\(B\)}} & \makecell{\textbf{Orthogonal}\\\textbf{\(B\)}} & \makecell{\textbf{SVD}\\\textbf{\(B\)}} \\
\midrule
\multirow{3}{*}{Caltech101} & 1 & $83.90 \pm 3.20$ & $83.45 \pm 3.44$ & $85.17 \pm 1.32$ & $84.77 \pm 2.53$ & $\mathbf{85.91} \pm \mathbf{0.76}$ & $85.45 \pm 2.08$ \\
 & 4 & $88.63 \pm 0.49$ & $88.74 \pm 0.45$ & $\mathbf{89.74} \pm \mathbf{0.23}$ & $89.03 \pm 0.41$ & $89.29 \pm 0.65$ & $89.53 \pm 0.60$ \\
 & 16 & $91.83 \pm 0.38$ & $91.94 \pm 0.06$ & $92.16 \pm 0.34$ & $\mathbf{92.17} \pm \mathbf{0.11}$ & $91.40 \pm 0.47$ & $91.74 \pm 0.16$ \\
\midrule
\multirow{3}{*}{DTD} & 1 & $42.93 \pm 1.18$ & $41.96 \pm 0.59$ & $41.63 \pm 1.46$ & $40.96 \pm 1.34$ & $\mathbf{42.16} \pm \mathbf{0.89}$ & $41.67 \pm 0.26$ \\
 & 4 & $52.90 \pm 0.30$ & $52.44 \pm 1.92$ & $52.94 \pm 1.21$ & $52.86 \pm 1.12$ & $53.01 \pm 1.35$ & $\mathbf{53.64} \pm \mathbf{1.33}$ \\
 & 16 & $63.03 \pm 1.04$ & $63.20 \pm 0.75$ & $62.08 \pm 0.44$ & $\mathbf{63.40} \pm \mathbf{0.55}$ & $62.67 \pm 1.30$ & $63.10 \pm 0.92$ \\
\midrule
\multirow{3}{*}{FGVC Aircraft} & 1 & $18.17 \pm 1.50$ & $16.94 \pm 1.01$ & $16.77 \pm 0.98$ & $16.93 \pm 0.12$ & $16.94 \pm 1.07$ & $\mathbf{17.19} \pm \mathbf{1.20}$ \\
 & 4 & $22.60 \pm 1.05$ & $21.37 \pm 0.49$ & $21.65 \pm 0.96$ & $21.38 \pm 0.51$ & $\mathbf{21.88} \pm \mathbf{0.29}$ & $21.17 \pm 1.18$ \\
 & 16 & $31.30 \pm 0.62$ & $27.92 \pm 0.62$ & $27.60 \pm 0.56$ & $27.75 \pm 0.54$ & $\mathbf{28.10} \pm \mathbf{0.98}$ & $27.76 \pm 0.48$ \\
\midrule
\multirow{3}{*}{Food101} & 1 & $63.70 \pm 0.52$ & $68.11 \pm 0.69$ & $\mathbf{69.67} \pm \mathbf{1.79}$ & $67.26 \pm 2.70$ & $67.72 \pm 2.76$ & $67.88 \pm 1.37$ \\
 & 4 & $70.07 \pm 3.51$ & $72.61 \pm 0.48$ & $\mathbf{73.69} \pm \mathbf{0.67}$ & $73.09 \pm 0.75$ & $72.34 \pm 0.23$ & $72.56 \pm 0.71$ \\
 & 16 & $76.17 \pm 0.31$ & $77.32 \pm 0.34$ & $\mathbf{77.51} \pm \mathbf{0.41}$ & $77.36 \pm 0.45$ & $77.28 \pm 0.14$ & $77.44 \pm 0.26$ \\
\midrule
\multirow{3}{*}{Oxford Flowers} & 1 & $70.93 \pm 1.96$ & $69.79 \pm 1.96$ & $70.81 \pm 1.52$ & $70.29 \pm 1.66$ & $69.59 \pm 2.72$ & $\mathbf{71.23} \pm \mathbf{2.14}$ \\
 & 4 & $87.13 \pm 2.04$ & $\mathbf{84.59} \pm \mathbf{0.13}$ & $83.30 \pm 0.42$ & $83.25 \pm 1.57$ & $83.81 \pm 1.87$ & $84.00 \pm 1.48$ \\
 & 16 & $94.70 \pm 0.10$ & $\mathbf{92.39} \pm \mathbf{0.43}$ & $91.97 \pm 0.57$ & $92.38 \pm 0.35$ & $92.25 \pm 0.20$ & $92.19 \pm 0.52$ \\
\midrule
\multirow{3}{*}{Oxford Pets} & 1 & $79.43 \pm 0.81$ & $\mathbf{82.87} \pm \mathbf{1.41}$ & $82.65 \pm 0.80$ & $82.46 \pm 1.39$ & $82.58 \pm 0.93$ & $82.47 \pm 0.82$ \\
 & 4 & $85.70 \pm 1.13$ & $84.36 \pm 1.65$ & $\mathbf{86.29} \pm \mathbf{0.68}$ & $85.55 \pm 1.11$ & $85.22 \pm 2.20$ & $84.26 \pm 2.25$ \\
 & 16 & $87.40 \pm 0.69$ & $87.48 \pm 0.88$ & $\mathbf{87.74} \pm \mathbf{0.44}$ & $87.46 \pm 0.64$ & $87.40 \pm 0.80$ & $86.84 \pm 0.56$ \\
\midrule
\multirow{3}{*}{UCF101} & 1 & $58.20 \pm 3.04$ & $\mathbf{60.35} \pm \mathbf{0.71}$ & $58.95 \pm 1.70$ & $58.76 \pm 2.47$ & $57.23 \pm 1.62$ & $58.60 \pm 1.80$ \\
 & 4 & $67.90 \pm 0.44$ & $67.83 \pm 0.58$ & $67.26 \pm 1.63$ & $\mathbf{68.46} \pm \mathbf{0.22}$ & $67.77 \pm 0.68$ & $67.72 \pm 0.14$ \\
 & 16 & $75.63 \pm 0.40$ & $75.35 \pm 0.94$ & $73.95 \pm 1.46$ & $74.70 \pm 1.15$ & $74.94 \pm 0.43$ & $\mathbf{75.81} \pm \mathbf{0.35}$ \\
\bottomrule
\end{tabular}%
}
\end{table*}

\begin{table*}[h]
\centering
\scriptsize
\caption{
Token-side basis ablation, ViT-B/16 at rank 4. Mean accuracy (\%) over seven
datasets for the trainable factorization (both factors learned) and four
fixed-$\bm{B}$ variants that freeze the token-side basis and train only $\bm{A}$:
Gaussian, learned-then-frozen, orthogonal, and SVD-derived bases. Dense CoOp is shown
as a reference. Values are mean $\pm$ sample standard deviation over three seeds;
\textbf{bold} marks the best mean among the trainable and fixed-$\bm{B}$ variants.
}
\label{tab:b_ablation_rank4_vitb16}
\resizebox{\textwidth}{!}{%
\begin{tabular}{llcccccc}
\toprule
\textbf{Dataset} & \textbf{Shots} & \makecell{\textbf{Dense CoOp}\\\textbf{Best}} & \makecell{\textbf{Trainable}\\\textbf{\(BA\)}} & \makecell{\textbf{Gaussian}\\\textbf{\(B\)}} & \makecell{\textbf{Learned}\\\textbf{\(B\)}} & \makecell{\textbf{Orthogonal}\\\textbf{\(B\)}} & \makecell{\textbf{SVD}\\\textbf{\(B\)}} \\
\midrule
\multirow{3}{*}{Caltech101} & 1 & $90.50 \pm 2.07$ & $91.48 \pm 2.36$ & $91.67 \pm 1.49$ & $91.90 \pm 1.55$ & $91.93 \pm 1.08$ & $\mathbf{92.17} \pm \mathbf{0.85}$ \\
 & 4 & $94.17 \pm 0.15$ & $93.97 \pm 0.55$ & $\mathbf{94.52} \pm \mathbf{0.31}$ & $\mathbf{94.52} \pm \mathbf{0.25}$ & $94.39 \pm 0.34$ & $94.28 \pm 0.88$ \\
 & 16 & $95.47 \pm 0.06$ & $95.51 \pm 0.29$ & $95.33 \pm 0.41$ & $95.43 \pm 0.45$ & $\mathbf{95.58} \pm \mathbf{0.41}$ & $95.51 \pm 0.28$ \\
\midrule
\multirow{3}{*}{DTD} & 1 & $48.77 \pm 1.36$ & $\mathbf{50.45} \pm \mathbf{1.58}$ & $50.41 \pm 2.30$ & $50.04 \pm 1.79$ & $49.55 \pm 2.16$ & $49.76 \pm 2.63$ \\
 & 4 & $59.47 \pm 0.67$ & $58.51 \pm 0.91$ & $\mathbf{59.40} \pm \mathbf{1.43}$ & $58.39 \pm 1.55$ & $58.51 \pm 0.41$ & $59.26 \pm 0.51$ \\
 & 16 & $69.10 \pm 0.44$ & $68.26 \pm 0.20$ & $68.79 \pm 1.97$ & $\mathbf{69.11} \pm \mathbf{0.94}$ & $68.85 \pm 0.53$ & $68.40 \pm 1.19$ \\
\midrule
\multirow{3}{*}{FGVC Aircraft} & 1 & $26.47 \pm 2.49$ & $27.95 \pm 0.93$ & $27.86 \pm 0.69$ & $26.95 \pm 1.01$ & $26.60 \pm 1.15$ & $\mathbf{28.08} \pm \mathbf{0.83}$ \\
 & 4 & $32.07 \pm 1.63$ & $31.75 \pm 1.40$ & $31.25 \pm 0.55$ & $\mathbf{31.97} \pm \mathbf{0.44}$ & $31.32 \pm 0.91$ & $31.92 \pm 1.40$ \\
 & 16 & $43.87 \pm 0.96$ & $40.46 \pm 0.48$ & $39.24 \pm 0.43$ & $\mathbf{40.53} \pm \mathbf{1.07}$ & $40.11 \pm 0.70$ & $39.85 \pm 0.79$ \\
\midrule
\multirow{3}{*}{Food101} & 1 & $78.37 \pm 1.66$ & $80.85 \pm 1.63$ & $81.94 \pm 1.16$ & $\mathbf{81.94} \pm \mathbf{1.40}$ & $80.10 \pm 1.67$ & $80.70 \pm 1.54$ \\
 & 4 & $82.07 \pm 1.67$ & $84.45 \pm 0.70$ & $84.44 \pm 0.92$ & $\mathbf{85.17} \pm \mathbf{0.28}$ & $84.44 \pm 0.99$ & $84.81 \pm 0.01$ \\
 & 16 & $85.10 \pm 0.10$ & $85.78 \pm 0.15$ & $\mathbf{85.91} \pm \mathbf{0.25}$ & $85.68 \pm 0.24$ & $85.66 \pm 0.08$ & $85.84 \pm 0.27$ \\
\midrule
\multirow{3}{*}{Oxford Flowers} & 1 & $81.33 \pm 0.35$ & $79.32 \pm 0.79$ & $76.83 \pm 5.06$ & $\mathbf{80.13} \pm \mathbf{1.11}$ & $78.39 \pm 2.83$ & $79.66 \pm 0.91$ \\
 & 4 & $92.63 \pm 1.11$ & $90.03 \pm 0.43$ & $87.44 \pm 3.21$ & $\mathbf{90.58} \pm \mathbf{0.48}$ & $88.94 \pm 3.41$ & $89.21 \pm 0.94$ \\
 & 16 & $97.07 \pm 0.38$ & $\mathbf{96.20} \pm \mathbf{0.29}$ & $96.05 \pm 0.47$ & $96.16 \pm 0.15$ & $95.91 \pm 0.37$ & $96.05 \pm 0.38$ \\
\midrule
\multirow{3}{*}{Oxford Pets} & 1 & $88.07 \pm 1.67$ & $\mathbf{91.27} \pm \mathbf{0.44}$ & $90.95 \pm 0.38$ & $89.76 \pm 1.31$ & $89.85 \pm 1.17$ & $91.05 \pm 1.00$ \\
 & 4 & $91.50 \pm 1.30$ & $91.81 \pm 0.50$ & $91.63 \pm 1.06$ & $91.47 \pm 1.16$ & $92.25 \pm 1.15$ & $\mathbf{92.33} \pm \mathbf{0.56}$ \\
 & 16 & $92.33 \pm 0.83$ & $92.48 \pm 0.57$ & $\mathbf{92.87} \pm \mathbf{0.48}$ & $92.35 \pm 0.51$ & $92.83 \pm 0.34$ & $92.45 \pm 0.62$ \\
\midrule
\multirow{3}{*}{UCF101} & 1 & $68.23 \pm 1.30$ & $69.13 \pm 0.53$ & $68.72 \pm 1.45$ & $68.61 \pm 1.87$ & $\mathbf{69.50} \pm \mathbf{0.54}$ & $68.76 \pm 0.78$ \\
 & 4 & $77.27 \pm 0.74$ & $76.69 \pm 0.37$ & $77.13 \pm 0.90$ & $76.30 \pm 0.61$ & $76.77 \pm 0.65$ & $\mathbf{77.94} \pm \mathbf{0.41}$ \\
 & 16 & $82.17 \pm 0.67$ & $81.86 \pm 1.32$ & $81.89 \pm 0.25$ & $\mathbf{82.26} \pm \mathbf{1.24}$ & $81.45 \pm 0.49$ & $\mathbf{82.26} \pm \mathbf{0.78}$ \\
\bottomrule
\end{tabular}%
}
\end{table*}

\begin{table*}[h]
\centering
\scriptsize
\caption{
Token-side basis ablation, RN50 at rank 8. Mean accuracy (\%) over seven
datasets for the trainable factorization (both factors learned) and four
fixed-$\bm{B}$ variants that freeze the token-side basis and train only $\bm{A}$:
Gaussian, learned-then-frozen, orthogonal, and SVD-derived bases. Dense CoOp is shown
as a reference. Values are mean $\pm$ sample standard deviation over three seeds;
\textbf{bold} marks the best mean among the trainable and fixed-$\bm{B}$ variants.
}
\label{tab:b_ablation_rank8_rn50}
\resizebox{\textwidth}{!}{%
\begin{tabular}{llcccccc}
\toprule
\textbf{Dataset} & \textbf{Shots} & \makecell{\textbf{Dense CoOp}\\\textbf{Best}} & \makecell{\textbf{Trainable}\\\textbf{\(BA\)}} & \makecell{\textbf{Gaussian}\\\textbf{\(B\)}} & \makecell{\textbf{Learned}\\\textbf{\(B\)}} & \makecell{\textbf{Orthogonal}\\\textbf{\(B\)}} & \makecell{\textbf{SVD}\\\textbf{\(B\)}} \\
\midrule
\multirow{3}{*}{Caltech101} & 1 & $83.90 \pm 3.20$ & $83.71 \pm 2.15$ & $\mathbf{85.95} \pm \mathbf{1.76}$ & $83.65 \pm 2.49$ & $84.81 \pm 2.40$ & $85.44 \pm 0.18$ \\
 & 4 & $88.63 \pm 0.49$ & $88.61 \pm 1.02$ & $\mathbf{88.80} \pm \mathbf{0.32}$ & $88.72 \pm 0.25$ & $88.33 \pm 0.48$ & $88.24 \pm 1.48$ \\
 & 16 & $91.83 \pm 0.38$ & $\mathbf{91.98} \pm \mathbf{0.05}$ & $91.66 \pm 0.21$ & $91.78 \pm 0.87$ & $91.72 \pm 0.28$ & $91.64 \pm 0.07$ \\
\midrule
\multirow{3}{*}{DTD} & 1 & $42.93 \pm 1.18$ & $43.42 \pm 0.34$ & $42.53 \pm 1.09$ & $\mathbf{43.52} \pm \mathbf{3.10}$ & $42.36 \pm 1.04$ & $42.22 \pm 1.01$ \\
 & 4 & $52.90 \pm 0.30$ & $53.33 \pm 0.25$ & $\mathbf{53.80} \pm \mathbf{0.39}$ & $53.41 \pm 0.86$ & $53.29 \pm 2.58$ & $53.07 \pm 0.27$ \\
 & 16 & $63.03 \pm 1.04$ & $\mathbf{63.48} \pm \mathbf{0.41}$ & $63.40 \pm 0.82$ & $63.20 \pm 1.09$ & $63.46 \pm 0.59$ & $63.44 \pm 0.76$ \\
\midrule
\multirow{3}{*}{FGVC Aircraft} & 1 & $18.17 \pm 1.50$ & $13.81 \pm 5.26$ & $\mathbf{17.05} \pm \mathbf{0.95}$ & $16.72 \pm 0.42$ & $16.93 \pm 0.98$ & $12.76 \pm 5.65$ \\
 & 4 & $22.60 \pm 1.05$ & $22.64 \pm 0.71$ & $22.40 \pm 1.34$ & $22.26 \pm 0.66$ & $\mathbf{22.72} \pm \mathbf{1.16}$ & $22.59 \pm 1.38$ \\
 & 16 & $31.30 \pm 0.62$ & $30.31 \pm 0.60$ & $28.66 \pm 0.08$ & $30.29 \pm 1.01$ & $\mathbf{30.39} \pm \mathbf{0.54}$ & $30.11 \pm 0.45$ \\
\midrule
\multirow{3}{*}{Food101} & 1 & $63.70 \pm 0.52$ & $66.60 \pm 1.92$ & $\mathbf{70.18} \pm \mathbf{3.19}$ & $65.88 \pm 1.70$ & $66.50 \pm 3.80$ & $68.13 \pm 0.83$ \\
 & 4 & $70.07 \pm 3.51$ & $70.34 \pm 1.45$ & $\mathbf{74.04} \pm \mathbf{1.47}$ & $71.24 \pm 0.56$ & $70.87 \pm 1.75$ & $72.14 \pm 0.75$ \\
 & 16 & $76.17 \pm 0.31$ & $75.97 \pm 0.36$ & $\mathbf{76.81} \pm \mathbf{0.33}$ & $75.96 \pm 0.50$ & $76.06 \pm 0.25$ & $76.19 \pm 0.34$ \\
\midrule
\multirow{3}{*}{Oxford Flowers} & 1 & $70.93 \pm 1.96$ & $71.23 \pm 0.87$ & $\mathbf{72.23} \pm \mathbf{0.71}$ & $70.25 \pm 1.22$ & $70.63 \pm 2.15$ & $70.96 \pm 1.96$ \\
 & 4 & $87.13 \pm 2.04$ & $\mathbf{86.40} \pm \mathbf{1.97}$ & $85.36 \pm 0.80$ & $85.97 \pm 1.70$ & $86.33 \pm 0.79$ & $86.37 \pm 1.31$ \\
 & 16 & $94.70 \pm 0.10$ & $93.92 \pm 0.14$ & $93.15 \pm 0.91$ & $\mathbf{94.15} \pm \mathbf{0.32}$ & $93.94 \pm 0.31$ & $93.87 \pm 0.53$ \\
\midrule
\multirow{3}{*}{Oxford Pets} & 1 & $79.43 \pm 0.81$ & $81.49 \pm 1.53$ & $81.51 \pm 2.17$ & $80.98 \pm 2.28$ & $81.45 \pm 1.29$ & $\mathbf{81.96} \pm \mathbf{1.53}$ \\
 & 4 & $85.70 \pm 1.13$ & $84.70 \pm 2.71$ & $\mathbf{85.46} \pm \mathbf{0.27}$ & $82.93 \pm 1.18$ & $83.55 \pm 0.46$ & $83.88 \pm 2.20$ \\
 & 16 & $87.40 \pm 0.69$ & $86.63 \pm 1.00$ & $\mathbf{87.74} \pm \mathbf{0.71}$ & $86.77 \pm 0.57$ & $87.23 \pm 0.57$ & $86.26 \pm 0.67$ \\
\midrule
\multirow{3}{*}{UCF101} & 1 & $58.20 \pm 3.04$ & $57.52 \pm 2.18$ & $\mathbf{58.28} \pm \mathbf{1.19}$ & $57.13 \pm 1.25$ & $56.23 \pm 0.40$ & $58.10 \pm 2.09$ \\
 & 4 & $67.90 \pm 0.44$ & $67.27 \pm 0.58$ & $66.88 \pm 1.56$ & $66.32 \pm 0.48$ & $\mathbf{67.54} \pm \mathbf{1.17}$ & $66.65 \pm 0.70$ \\
 & 16 & $75.63 \pm 0.40$ & $74.76 \pm 1.03$ & $75.10 \pm 0.94$ & $74.39 \pm 0.37$ & $75.73 \pm 0.72$ & $\mathbf{75.80} \pm \mathbf{0.66}$ \\
\bottomrule
\end{tabular}%
}
\end{table*}

\begin{table*}[h]
\centering
\scriptsize
\caption{
Token-side basis ablation, ViT-B/16 at rank 8. Mean accuracy (\%) over seven
datasets for the trainable factorization (both factors learned) and four
fixed-$\bm{B}$ variants that freeze the token-side basis and train only $\bm{A}$:
Gaussian, learned-then-frozen, orthogonal, and SVD-derived bases. Dense CoOp is shown
as a reference. Values are mean $\pm$ sample standard deviation over three seeds;
\textbf{bold} marks the best mean among the trainable and fixed-$\bm{B}$ variants.
}
\label{tab:b_ablation_rank8_vitb16}
\resizebox{\textwidth}{!}{%
\begin{tabular}{llcccccc}
\toprule
\textbf{Dataset} & \textbf{Shots} & \makecell{\textbf{Dense CoOp}\\\textbf{Best}} & \makecell{\textbf{Trainable}\\\textbf{\(BA\)}} & \makecell{\textbf{Gaussian}\\\textbf{\(B\)}} & \makecell{\textbf{Learned}\\\textbf{\(B\)}} & \makecell{\textbf{Orthogonal}\\\textbf{\(B\)}} & \makecell{\textbf{SVD}\\\textbf{\(B\)}} \\
\midrule
\multirow{3}{*}{Caltech101} & 1 & $90.50 \pm 2.07$ & $91.14 \pm 1.72$ & $91.39 \pm 1.86$ & $\mathbf{91.97} \pm \mathbf{1.83}$ & $90.48 \pm 2.66$ & $91.72 \pm 1.97$ \\
 & 4 & $94.17 \pm 0.15$ & $94.16 \pm 0.32$ & $94.27 \pm 0.37$ & $94.46 \pm 0.18$ & $93.82 \pm 0.70$ & $\mathbf{94.47} \pm \mathbf{0.06}$ \\
 & 16 & $95.47 \pm 0.06$ & $95.58 \pm 0.15$ & $95.71 \pm 0.14$ & $\mathbf{95.73} \pm \mathbf{0.27}$ & $95.33 \pm 0.29$ & $95.66 \pm 0.28$ \\
\midrule
\multirow{3}{*}{DTD} & 1 & $48.77 \pm 1.36$ & $50.33 \pm 1.22$ & $47.60 \pm 3.09$ & $50.28 \pm 0.39$ & $49.19 \pm 2.19$ & $\mathbf{51.16} \pm \mathbf{1.18}$ \\
 & 4 & $59.47 \pm 0.67$ & $58.25 \pm 1.50$ & $\mathbf{59.83} \pm \mathbf{1.00}$ & $59.75 \pm 0.89$ & $59.04 \pm 0.87$ & $59.63 \pm 1.08$ \\
 & 16 & $69.10 \pm 0.44$ & $69.29 \pm 0.65$ & $69.54 \pm 1.64$ & $69.41 \pm 0.44$ & $69.50 \pm 0.37$ & $\mathbf{69.74} \pm \mathbf{0.12}$ \\
\midrule
\multirow{3}{*}{FGVC Aircraft} & 1 & $26.47 \pm 2.49$ & $26.49 \pm 1.41$ & $26.52 \pm 0.35$ & $26.96 \pm 1.09$ & $27.31 \pm 0.84$ & $\mathbf{27.48} \pm \mathbf{1.59}$ \\
 & 4 & $32.07 \pm 1.63$ & $30.42 \pm 3.81$ & $\mathbf{33.47} \pm \mathbf{1.37}$ & $32.95 \pm 1.65$ & $32.88 \pm 0.89$ & $30.66 \pm 3.05$ \\
 & 16 & $43.87 \pm 0.96$ & $42.06 \pm 0.11$ & $41.30 \pm 0.52$ & $41.87 \pm 0.30$ & $41.68 \pm 0.47$ & $\mathbf{42.12} \pm \mathbf{0.56}$ \\
\midrule
\multirow{3}{*}{Food101} & 1 & $78.37 \pm 1.66$ & $79.79 \pm 2.36$ & $78.94 \pm 0.77$ & $77.91 \pm 0.79$ & $77.51 \pm 3.17$ & $\mathbf{80.32} \pm \mathbf{2.24}$ \\
 & 4 & $82.07 \pm 1.67$ & $82.32 \pm 0.34$ & $\mathbf{83.44} \pm \mathbf{0.42}$ & $82.85 \pm 0.66$ & $81.94 \pm 0.64$ & $82.99 \pm 1.05$ \\
 & 16 & $85.10 \pm 0.10$ & $85.00 \pm 0.08$ & $\mathbf{85.67} \pm \mathbf{0.23}$ & $85.01 \pm 0.33$ & $85.17 \pm 0.20$ & $84.93 \pm 0.10$ \\
\midrule
\multirow{3}{*}{Oxford Flowers} & 1 & $81.33 \pm 0.35$ & $80.34 \pm 1.77$ & $\mathbf{81.43} \pm \mathbf{0.82}$ & $81.31 \pm 0.22$ & $81.04 \pm 1.28$ & $80.09 \pm 1.64$ \\
 & 4 & $92.63 \pm 1.11$ & $\mathbf{92.11} \pm \mathbf{2.30}$ & $91.43 \pm 1.40$ & $91.89 \pm 2.08$ & $91.73 \pm 1.76$ & $91.49 \pm 1.44$ \\
 & 16 & $97.07 \pm 0.38$ & $96.67 \pm 0.32$ & $96.47 \pm 0.25$ & $\mathbf{96.74} \pm \mathbf{0.29}$ & $96.64 \pm 0.41$ & $96.71 \pm 0.15$ \\
\midrule
\multirow{3}{*}{Oxford Pets} & 1 & $88.07 \pm 1.67$ & $90.59 \pm 1.30$ & $90.08 \pm 1.11$ & $89.48 \pm 2.04$ & $89.55 \pm 2.53$ & $\mathbf{91.20} \pm \mathbf{1.07}$ \\
 & 4 & $91.50 \pm 1.30$ & $91.78 \pm 0.76$ & $91.63 \pm 0.45$ & $90.37 \pm 1.04$ & $91.07 \pm 1.34$ & $\mathbf{92.34} \pm \mathbf{0.26}$ \\
 & 16 & $92.33 \pm 0.83$ & $92.41 \pm 0.55$ & $92.36 \pm 0.21$ & $\mathbf{92.59} \pm \mathbf{0.65}$ & $92.02 \pm 0.60$ & $92.58 \pm 0.90$ \\
\midrule
\multirow{3}{*}{UCF101} & 1 & $68.23 \pm 1.30$ & $68.81 \pm 2.57$ & $68.88 \pm 0.51$ & $67.46 \pm 3.06$ & $68.58 \pm 1.41$ & $\mathbf{70.02} \pm \mathbf{1.61}$ \\
 & 4 & $77.27 \pm 0.74$ & $76.01 \pm 0.33$ & $\mathbf{77.12} \pm \mathbf{0.59}$ & $76.88 \pm 0.59$ & $76.84 \pm 0.69$ & $76.64 \pm 0.68$ \\
 & 16 & $82.17 \pm 0.67$ & $82.20 \pm 0.26$ & $\mathbf{82.40} \pm \mathbf{0.88}$ & $82.19 \pm 0.70$ & $81.97 \pm 0.12$ & $81.34 \pm 0.93$ \\
\bottomrule
\end{tabular}%
}
\end{table*}



\clearpage
\subsection{Base-to-new generalization}

\begin{table*}[h]
\centering
\scriptsize
\caption{Per-dataset base-to-new generalization on RN50. Each model is trained on the
seen (base) classes and evaluated on the seen classes, the held-out unseen (new)
classes, and their harmonic mean (\textbf{Seen}/\textbf{Unseen}/\textbf{H}), at 1, 4,
and 16 shots. Methods are dense CoOp, the trainable factorization at ranks 4 and 8,
and the fixed orthogonal-$\bm{B}$ factorization at ranks 4 and 8. Values are mean (\%)
$\pm$ sample standard deviation over three seeds; \textbf{bold} marks the best value
across methods for each dataset, shot, and metric.}
\label{tab:base_to_new_rn50_dataset}
\resizebox{\textwidth}{!}{
\begin{tabular}{llccccccccc}
\toprule
\multirow{2}{*}{\textbf{Dataset}} & \multirow{2}{*}{\textbf{Method}} & \multicolumn{3}{c}{\textbf{1-shot}} & \multicolumn{3}{c}{\textbf{4-shot}} & \multicolumn{3}{c}{\textbf{16-shot}} \\
\cmidrule(lr){3-5} \cmidrule(lr){6-8} \cmidrule(lr){9-11}
& & \textbf{Seen} & \textbf{Unseen} & \textbf{H} & \textbf{Seen} & \textbf{Unseen} & \textbf{H} & \textbf{Seen} & \textbf{Unseen} & \textbf{H} \\
\midrule
\multirow{5}{*}{Caltech101} & Dense CoOp & $91.33 \pm 0.81$ & $86.27 \pm 0.87$ & $88.73 \pm 0.83$ & $93.57 \pm 0.15$ & $83.70 \pm 0.85$ & $88.36 \pm 0.50$ & $95.37 \pm 0.15$ & $84.83 \pm 3.97$ & $89.76 \pm 2.20$ \\
 & Factorized CoOp, $r=4$ & $91.70 \pm 1.21$ & $\mathbf{89.33 \pm 0.86}$ & $\mathbf{90.50 \pm 0.68}$ & $93.80 \pm 0.85$ & $\mathbf{87.60 \pm 2.16}$ & $90.58 \pm 1.16$ & $\mathbf{95.40 \pm 0.52}$ & $84.10 \pm 4.26$ & $89.35 \pm 2.14$ \\
 & Factorized CoOp, $r=8$ & $\mathbf{93.00 \pm 0.61}$ & $87.63 \pm 2.85$ & $90.23 \pm 1.77$ & $93.37 \pm 1.01$ & $86.80 \pm 1.82$ & $89.96 \pm 1.43$ & $95.17 \pm 0.35$ & $82.70 \pm 3.12$ & $88.48 \pm 1.94$ \\
 & Fixed Orthogonal $B$, $r=4$ & $91.97 \pm 1.69$ & $87.50 \pm 2.44$ & $89.67 \pm 1.99$ & $94.03 \pm 0.31$ & $87.57 \pm 3.53$ & $\mathbf{90.66 \pm 2.05}$ & $95.23 \pm 0.21$ & $\mathbf{86.97 \pm 0.80}$ & $\mathbf{90.91 \pm 0.37}$ \\
 & Fixed Orthogonal $B$, $r=8$ & $91.60 \pm 0.17$ & $87.20 \pm 2.05$ & $89.34 \pm 1.15$ & $\mathbf{94.17 \pm 0.51}$ & $86.13 \pm 4.32$ & $89.94 \pm 2.54$ & $95.13 \pm 0.31$ & $85.60 \pm 1.51$ & $90.11 \pm 0.97$ \\
\midrule
\multirow{5}{*}{DTD} & Dense CoOp & $51.00 \pm 2.26$ & $44.33 \pm 4.27$ & $47.31 \pm 2.22$ & $65.50 \pm 0.69$ & $\mathbf{42.93 \pm 4.23}$ & $\mathbf{51.80 \pm 3.13}$ & $74.27 \pm 0.57$ & $34.77 \pm 1.94$ & $47.34 \pm 1.75$ \\
 & Factorized CoOp, $r=4$ & $49.10 \pm 3.02$ & $44.77 \pm 4.27$ & $46.70 \pm 2.16$ & $65.23 \pm 1.03$ & $42.90 \pm 2.14$ & $51.74 \pm 1.61$ & $\mathbf{75.33 \pm 0.57}$ & $36.03 \pm 2.25$ & $48.72 \pm 2.10$ \\
 & Factorized CoOp, $r=8$ & $50.80 \pm 4.50$ & $43.30 \pm 6.46$ & $46.33 \pm 2.28$ & $65.87 \pm 1.06$ & $42.47 \pm 1.63$ & $51.62 \pm 0.90$ & $74.80 \pm 0.98$ & $38.90 \pm 1.49$ & $51.18 \pm 1.52$ \\
 & Fixed Orthogonal $B$, $r=4$ & $\mathbf{52.40 \pm 3.75}$ & $40.13 \pm 4.90$ & $45.22 \pm 2.39$ & $64.63 \pm 1.57$ & $42.10 \pm 2.95$ & $50.97 \pm 2.60$ & $74.50 \pm 1.30$ & $37.20 \pm 3.48$ & $49.58 \pm 3.36$ \\
 & Fixed Orthogonal $B$, $r=8$ & $50.87 \pm 5.23$ & $\mathbf{45.53 \pm 5.26}$ & $\mathbf{47.69 \pm 1.23}$ & $\mathbf{65.97 \pm 1.80}$ & $41.63 \pm 5.86$ & $50.94 \pm 4.93$ & $74.83 \pm 0.76$ & $\mathbf{41.23 \pm 1.97}$ & $\mathbf{53.15 \pm 1.50}$ \\
\midrule
\multirow{5}{*}{FGVC Aircraft} & Dense CoOp & $9.53 \pm 7.07$ & $5.90 \pm 6.85$ & $7.04 \pm 7.25$ & $21.40 \pm 1.84$ & $9.03 \pm 4.50$ & $12.42 \pm 4.66$ & $\mathbf{31.40 \pm 0.46}$ & $12.73 \pm 1.53$ & $18.08 \pm 1.51$ \\
 & Factorized CoOp, $r=4$ & $\mathbf{18.87 \pm 0.81}$ & $18.57 \pm 0.97$ & $18.71 \pm 0.84$ & $\mathbf{22.13 \pm 1.10}$ & $\mathbf{19.17 \pm 1.36}$ & $\mathbf{20.51 \pm 0.80}$ & $28.30 \pm 1.14$ & $11.53 \pm 3.61$ & $16.20 \pm 3.58$ \\
 & Factorized CoOp, $r=8$ & $15.53 \pm 3.67$ & $12.60 \pm 6.56$ & $13.65 \pm 5.71$ & $22.10 \pm 1.42$ & $12.40 \pm 2.15$ & $15.85 \pm 2.11$ & $30.07 \pm 0.25$ & $14.73 \pm 0.21$ & $19.78 \pm 0.23$ \\
 & Fixed Orthogonal $B$, $r=4$ & $18.67 \pm 0.47$ & $\mathbf{19.17 \pm 3.07}$ & $\mathbf{18.85 \pm 1.69}$ & $21.90 \pm 0.70$ & $17.40 \pm 6.07$ & $18.97 \pm 3.23$ & $28.57 \pm 1.45$ & $14.90 \pm 0.40$ & $19.57 \pm 0.02$ \\
 & Fixed Orthogonal $B$, $r=8$ & $18.00 \pm 0.85$ & $18.57 \pm 1.72$ & $18.27 \pm 1.27$ & $22.10 \pm 1.57$ & $15.10 \pm 2.91$ & $17.75 \pm 1.61$ & $30.97 \pm 1.66$ & $\mathbf{15.90 \pm 1.30}$ & $\mathbf{21.01 \pm 1.49}$ \\
\midrule
\multirow{5}{*}{Food101} & Dense CoOp & $71.73 \pm 0.31$ & $74.30 \pm 2.21$ & $72.98 \pm 1.07$ & $74.67 \pm 2.07$ & $73.10 \pm 2.71$ & $73.87 \pm 2.38$ & $79.57 \pm 0.71$ & $74.53 \pm 1.57$ & $76.97 \pm 1.14$ \\
 & Factorized CoOp, $r=4$ & $77.87 \pm 1.17$ & $77.33 \pm 4.42$ & $77.58 \pm 2.79$ & $76.30 \pm 1.01$ & $\mathbf{74.67 \pm 0.84}$ & $\mathbf{75.47 \pm 0.86}$ & $81.70 \pm 0.66$ & $\mathbf{79.67 \pm 0.55}$ & $\mathbf{80.67 \pm 0.22}$ \\
 & Factorized CoOp, $r=8$ & $75.13 \pm 0.15$ & $77.70 \pm 0.61$ & $76.39 \pm 0.28$ & $\mathbf{76.60 \pm 1.75}$ & $74.00 \pm 2.70$ & $75.27 \pm 2.00$ & $80.77 \pm 0.31$ & $75.73 \pm 2.11$ & $78.16 \pm 1.21$ \\
 & Fixed Orthogonal $B$, $r=4$ & $\mathbf{78.83 \pm 2.49}$ & $\mathbf{79.00 \pm 6.32}$ & $\mathbf{78.88 \pm 4.43}$ & $75.93 \pm 0.50$ & $73.90 \pm 1.57$ & $74.89 \pm 0.61$ & $\mathbf{81.77 \pm 0.76}$ & $77.87 \pm 1.23$ & $79.77 \pm 0.90$ \\
 & Fixed Orthogonal $B$, $r=8$ & $75.20 \pm 0.85$ & $77.77 \pm 1.62$ & $76.46 \pm 1.17$ & $75.83 \pm 2.00$ & $74.37 \pm 3.45$ & $75.09 \pm 2.72$ & $80.43 \pm 1.25$ & $74.47 \pm 2.47$ & $77.33 \pm 1.85$ \\
\midrule
\multirow{5}{*}{Oxford Flowers} & Dense CoOp & $75.37 \pm 2.14$ & $62.90 \pm 4.26$ & $68.55 \pm 3.29$ & $88.60 \pm 1.48$ & $54.07 \pm 2.20$ & $67.12 \pm 1.43$ & $\mathbf{96.30 \pm 0.46}$ & $54.17 \pm 2.97$ & $69.30 \pm 2.45$ \\
 & Factorized CoOp, $r=4$ & $77.43 \pm 1.59$ & $65.67 \pm 0.97$ & $71.06 \pm 0.81$ & $88.90 \pm 1.40$ & $61.20 \pm 1.65$ & $72.49 \pm 1.46$ & $95.73 \pm 0.21$ & $\mathbf{59.70 \pm 1.42}$ & $\mathbf{73.53 \pm 1.14}$ \\
 & Factorized CoOp, $r=8$ & $\mathbf{78.17 \pm 0.60}$ & $65.77 \pm 2.10$ & $71.42 \pm 1.35$ & $89.37 \pm 1.93$ & $56.90 \pm 4.16$ & $69.43 \pm 2.61$ & $\mathbf{96.30 \pm 0.53}$ & $55.50 \pm 1.51$ & $70.41 \pm 1.31$ \\
 & Fixed Orthogonal $B$, $r=4$ & $77.50 \pm 0.17$ & $66.70 \pm 1.25$ & $71.69 \pm 0.73$ & $87.87 \pm 1.47$ & $\mathbf{64.30 \pm 1.71}$ & $\mathbf{74.25 \pm 1.28}$ & $95.40 \pm 0.70$ & $58.73 \pm 4.83$ & $72.64 \pm 3.81$ \\
 & Fixed Orthogonal $B$, $r=8$ & $77.27 \pm 3.42$ & $\mathbf{67.70 \pm 4.40}$ & $\mathbf{72.14 \pm 3.68}$ & $\mathbf{89.83 \pm 0.68}$ & $59.03 \pm 4.08$ & $71.18 \pm 2.81$ & $95.73 \pm 0.45$ & $55.80 \pm 2.82$ & $70.47 \pm 2.18$ \\
\midrule
\multirow{5}{*}{Oxford Pets} & Dense CoOp & $80.03 \pm 3.69$ & $91.00 \pm 3.50$ & $85.09 \pm 1.91$ & $87.20 \pm 1.57$ & $88.40 \pm 3.60$ & $87.76 \pm 1.83$ & $89.20 \pm 0.30$ & $85.03 \pm 6.51$ & $86.99 \pm 3.51$ \\
 & Factorized CoOp, $r=4$ & $84.87 \pm 2.87$ & $93.53 \pm 0.95$ & $88.97 \pm 1.77$ & $88.03 \pm 4.22$ & $93.37 \pm 0.32$ & $90.59 \pm 2.22$ & $89.83 \pm 0.81$ & $\mathbf{92.83 \pm 1.66}$ & $\mathbf{91.30 \pm 0.95}$ \\
 & Factorized CoOp, $r=8$ & $84.43 \pm 2.08$ & $\mathbf{93.90 \pm 0.36}$ & $88.90 \pm 1.06$ & $88.10 \pm 2.05$ & $\mathbf{94.60 \pm 1.15}$ & $91.22 \pm 1.14$ & $89.53 \pm 0.21$ & $91.30 \pm 2.67$ & $90.39 \pm 1.32$ \\
 & Fixed Orthogonal $B$, $r=4$ & $85.83 \pm 2.55$ & $92.93 \pm 2.00$ & $\mathbf{89.24 \pm 2.25}$ & $\mathbf{88.93 \pm 1.01}$ & $94.47 \pm 1.36$ & $\mathbf{91.61 \pm 1.04}$ & $89.10 \pm 1.67$ & $91.00 \pm 1.23$ & $90.04 \pm 1.41$ \\
 & Fixed Orthogonal $B$, $r=8$ & $\mathbf{86.00 \pm 2.96}$ & $92.43 \pm 1.86$ & $89.08 \pm 1.83$ & $87.33 \pm 2.23$ & $93.43 \pm 0.71$ & $90.27 \pm 1.49$ & $\mathbf{89.90 \pm 0.46}$ & $89.17 \pm 2.55$ & $89.52 \pm 1.30$ \\
\midrule
\multirow{5}{*}{UCF101} & Dense CoOp & $64.00 \pm 3.97$ & $56.97 \pm 2.91$ & $60.25 \pm 2.94$ & $72.93 \pm 1.50$ & $52.50 \pm 4.11$ & $60.98 \pm 2.71$ & $80.43 \pm 0.64$ & $46.27 \pm 2.74$ & $58.72 \pm 2.30$ \\
 & Factorized CoOp, $r=4$ & $\mathbf{68.83 \pm 2.11}$ & $\mathbf{63.20 \pm 2.78}$ & $\mathbf{65.89 \pm 2.46}$ & $72.90 \pm 2.07$ & $\mathbf{54.00 \pm 5.01}$ & $\mathbf{62.00 \pm 4.08}$ & $80.23 \pm 1.10$ & $\mathbf{50.33 \pm 5.80}$ & $\mathbf{61.72 \pm 4.22}$ \\
 & Factorized CoOp, $r=8$ & $66.10 \pm 1.15$ & $56.03 \pm 0.99$ & $60.65 \pm 1.03$ & $\mathbf{74.50 \pm 1.04}$ & $50.40 \pm 2.69$ & $60.10 \pm 1.97$ & $79.67 \pm 1.02$ & $46.43 \pm 7.45$ & $58.46 \pm 5.98$ \\
 & Fixed Orthogonal $B$, $r=4$ & $65.97 \pm 0.49$ & $60.30 \pm 2.19$ & $63.00 \pm 1.42$ & $73.47 \pm 1.91$ & $51.13 \pm 5.35$ & $60.24 \pm 4.31$ & $80.07 \pm 0.51$ & $43.57 \pm 3.85$ & $56.37 \pm 3.39$ \\
 & Fixed Orthogonal $B$, $r=8$ & $65.23 \pm 3.60$ & $58.20 \pm 0.90$ & $61.50 \pm 2.04$ & $72.27 \pm 2.43$ & $53.70 \pm 8.23$ & $61.47 \pm 6.14$ & $\mathbf{80.70 \pm 1.01}$ & $48.80 \pm 4.03$ & $60.77 \pm 3.40$ \\
\bottomrule
\end{tabular}
}
\end{table*}

\begin{table*}[h]
\centering
\scriptsize
\caption{Per-dataset base-to-new generalization on ViT-B/16. Each model is trained on
the seen (base) classes and evaluated on the seen classes, the held-out unseen (new)
classes, and their harmonic mean (\textbf{Seen}/\textbf{Unseen}/\textbf{H}), at 1, 4,
and 16 shots. Methods are dense CoOp, the trainable factorization at ranks 4 and 8,
and the fixed orthogonal-$\bm{B}$ factorization at ranks 4 and 8. Values are mean (\%)
$\pm$ sample standard deviation over three seeds; \textbf{bold} marks the best value
across methods for each dataset, shot, and metric.}
\label{tab:base_to_new_vitb16_dataset}
\resizebox{\textwidth}{!}{
\begin{tabular}{llccccccccc}
\toprule
\multirow{2}{*}{\textbf{Dataset}} & \multirow{2}{*}{\textbf{Method}} & \multicolumn{3}{c}{\textbf{1-shot}} & \multicolumn{3}{c}{\textbf{4-shot}} & \multicolumn{3}{c}{\textbf{16-shot}} \\
\cmidrule(lr){3-5} \cmidrule(lr){6-8} \cmidrule(lr){9-11}
& & \textbf{Seen} & \textbf{Unseen} & \textbf{H} & \textbf{Seen} & \textbf{Unseen} & \textbf{H} & \textbf{Seen} & \textbf{Unseen} & \textbf{H} \\
\midrule
\multirow{5}{*}{Caltech101} & Dense CoOp & $96.03 \pm 0.42$ & $92.43 \pm 0.38$ & $94.20 \pm 0.37$ & $96.90 \pm 0.20$ & $90.67 \pm 3.00$ & $93.66 \pm 1.54$ & $97.93 \pm 0.40$ & $89.27 \pm 3.23$ & $93.38 \pm 1.69$ \\
 & Factorized CoOp, $r=4$ & $\mathbf{96.87 \pm 0.57}$ & $92.83 \pm 1.46$ & $\mathbf{94.81 \pm 1.01}$ & $\mathbf{97.40 \pm 0.36}$ & $93.60 \pm 1.31$ & $\mathbf{95.46 \pm 0.73}$ & $98.00 \pm 0.17$ & $89.83 \pm 0.93$ & $93.74 \pm 0.49$ \\
 & Factorized CoOp, $r=8$ & $95.70 \pm 1.87$ & $\mathbf{93.60 \pm 0.70}$ & $94.63 \pm 1.11$ & $97.33 \pm 0.15$ & $92.97 \pm 1.25$ & $95.10 \pm 0.58$ & $\mathbf{98.17 \pm 0.12}$ & $\mathbf{92.07 \pm 1.01}$ & $\mathbf{95.02 \pm 0.57}$ \\
 & Fixed Orthogonal $B$, $r=4$ & $95.73 \pm 0.29$ & $92.80 \pm 1.45$ & $94.24 \pm 0.87$ & $97.33 \pm 0.70$ & $91.60 \pm 0.26$ & $94.38 \pm 0.47$ & $97.93 \pm 0.15$ & $90.67 \pm 2.16$ & $94.15 \pm 1.21$ \\
 & Fixed Orthogonal $B$, $r=8$ & $95.90 \pm 1.30$ & $92.43 \pm 1.63$ & $94.13 \pm 1.31$ & $97.17 \pm 0.06$ & $\mathbf{93.77 \pm 0.42}$ & $95.44 \pm 0.24$ & $98.03 \pm 0.50$ & $90.90 \pm 0.40$ & $94.33 \pm 0.20$ \\
\midrule
\multirow{5}{*}{DTD} & Dense CoOp & $58.73 \pm 3.26$ & $46.60 \pm 4.01$ & $51.80 \pm 1.47$ & $71.27 \pm 1.00$ & $47.33 \pm 3.40$ & $56.85 \pm 2.78$ & $\mathbf{80.97 \pm 1.87}$ & $44.63 \pm 7.08$ & $57.39 \pm 6.23$ \\
 & Factorized CoOp, $r=4$ & $\mathbf{59.50 \pm 3.72}$ & $\mathbf{53.33 \pm 3.98}$ & $\mathbf{56.08 \pm 0.84}$ & $71.00 \pm 1.51$ & $\mathbf{49.87 \pm 3.31}$ & $\mathbf{58.57 \pm 2.80}$ & $80.50 \pm 0.36$ & $47.13 \pm 3.47$ & $\mathbf{59.41 \pm 2.73}$ \\
 & Factorized CoOp, $r=8$ & $58.33 \pm 3.70$ & $47.87 \pm 1.40$ & $52.56 \pm 2.17$ & $71.83 \pm 0.46$ & $45.40 \pm 1.32$ & $55.63 \pm 1.04$ & $80.00 \pm 0.10$ & $\mathbf{47.20 \pm 1.25}$ & $59.36 \pm 1.00$ \\
 & Fixed Orthogonal $B$, $r=4$ & $56.23 \pm 2.84$ & $47.37 \pm 9.00$ & $50.93 \pm 4.26$ & $71.53 \pm 1.44$ & $49.60 \pm 3.30$ & $58.53 \pm 2.28$ & $80.23 \pm 0.50$ & $46.67 \pm 3.78$ & $58.94 \pm 2.89$ \\
 & Fixed Orthogonal $B$, $r=8$ & $57.87 \pm 5.45$ & $53.27 \pm 0.71$ & $55.37 \pm 2.06$ & $\mathbf{72.10 \pm 1.47}$ & $47.97 \pm 3.69$ & $57.56 \pm 2.85$ & $80.87 \pm 0.67$ & $46.80 \pm 2.50$ & $59.27 \pm 2.18$ \\
\midrule
\multirow{5}{*}{FGVC Aircraft} & Dense CoOp & $26.97 \pm 0.55$ & $23.67 \pm 5.70$ & $24.95 \pm 3.51$ & $31.27 \pm 1.81$ & $15.00 \pm 9.27$ & $19.20 \pm 9.75$ & $\mathbf{41.97 \pm 1.27}$ & $22.50 \pm 1.87$ & $29.26 \pm 1.64$ \\
 & Factorized CoOp, $r=4$ & $29.47 \pm 1.10$ & $\mathbf{27.97 \pm 0.59}$ & $\mathbf{28.69 \pm 0.79}$ & $32.27 \pm 0.65$ & $\mathbf{27.47 \pm 0.35}$ & $\mathbf{29.67 \pm 0.48}$ & $40.53 \pm 0.86$ & $\mathbf{25.10 \pm 4.61}$ & $\mathbf{30.80 \pm 3.12}$ \\
 & Factorized CoOp, $r=8$ & $22.33 \pm 9.99$ & $19.83 \pm 12.86$ & $20.72 \pm 12.03$ & $30.63 \pm 2.29$ & $12.37 \pm 9.31$ & $16.28 \pm 11.05$ & $41.63 \pm 1.08$ & $21.40 \pm 2.66$ & $28.21 \pm 2.40$ \\
 & Fixed Orthogonal $B$, $r=4$ & $\mathbf{29.63 \pm 0.86}$ & $26.57 \pm 2.72$ & $27.99 \pm 1.92$ & $31.50 \pm 3.64$ & $21.37 \pm 11.23$ & $24.59 \pm 10.11$ & $39.70 \pm 1.25$ & $21.77 \pm 3.96$ & $27.92 \pm 3.21$ \\
 & Fixed Orthogonal $B$, $r=8$ & $27.40 \pm 0.36$ & $25.67 \pm 2.58$ & $26.47 \pm 1.57$ & $\mathbf{32.63 \pm 0.93}$ & $21.30 \pm 3.14$ & $25.69 \pm 2.42$ & $40.73 \pm 0.80$ & $22.40 \pm 2.82$ & $28.82 \pm 2.31$ \\
\midrule
\multirow{5}{*}{Food101} & Dense CoOp & $84.90 \pm 0.53$ & $87.00 \pm 1.51$ & $85.93 \pm 0.76$ & $86.17 \pm 0.42$ & $85.13 \pm 1.74$ & $85.64 \pm 0.67$ & $87.90 \pm 0.17$ & $83.63 \pm 3.02$ & $85.70 \pm 1.62$ \\
 & Factorized CoOp, $r=4$ & $\mathbf{87.67 \pm 1.25}$ & $87.63 \pm 2.39$ & $87.65 \pm 1.74$ & $\mathbf{87.63 \pm 0.81}$ & $\mathbf{87.23 \pm 2.11}$ & $\mathbf{87.43 \pm 1.46}$ & $\mathbf{89.37 \pm 0.38}$ & $87.13 \pm 1.48$ & $88.23 \pm 0.84$ \\
 & Factorized CoOp, $r=8$ & $86.83 \pm 1.42$ & $\mathbf{88.80 \pm 1.20}$ & $\mathbf{87.80 \pm 1.19}$ & $86.87 \pm 0.67$ & $85.77 \pm 1.91$ & $86.31 \pm 1.12$ & $88.67 \pm 0.49$ & $84.93 \pm 1.47$ & $86.76 \pm 0.85$ \\
 & Fixed Orthogonal $B$, $r=4$ & $86.47 \pm 2.04$ & $87.37 \pm 1.30$ & $86.91 \pm 1.60$ & $87.33 \pm 1.18$ & $85.80 \pm 2.52$ & $86.56 \pm 1.83$ & $89.33 \pm 0.42$ & $\mathbf{87.97 \pm 0.55}$ & $\mathbf{88.64 \pm 0.45}$ \\
 & Fixed Orthogonal $B$, $r=8$ & $84.73 \pm 1.60$ & $85.47 \pm 1.97$ & $85.08 \pm 1.09$ & $85.90 \pm 0.95$ & $85.80 \pm 1.25$ & $85.85 \pm 0.84$ & $88.37 \pm 0.25$ & $85.13 \pm 4.20$ & $86.69 \pm 2.23$ \\
\midrule
\multirow{5}{*}{Oxford Flowers} & Dense CoOp & $83.27 \pm 2.99$ & $64.97 \pm 1.26$ & $72.98 \pm 1.77$ & $92.10 \pm 1.49$ & $60.63 \pm 3.43$ & $73.08 \pm 2.37$ & $97.70 \pm 0.53$ & $56.97 \pm 3.87$ & $71.92 \pm 3.09$ \\
 & Factorized CoOp, $r=4$ & $83.70 \pm 3.30$ & $72.33 \pm 2.10$ & $77.54 \pm 0.42$ & $91.23 \pm 1.39$ & $66.17 \pm 2.65$ & $76.70 \pm 2.26$ & $97.47 \pm 0.50$ & $61.53 \pm 1.86$ & $75.43 \pm 1.54$ \\
 & Factorized CoOp, $r=8$ & $83.90 \pm 1.93$ & $72.93 \pm 4.69$ & $77.97 \pm 2.82$ & $91.57 \pm 2.04$ & $65.03 \pm 1.27$ & $76.03 \pm 0.63$ & $\mathbf{97.77 \pm 0.74}$ & $\mathbf{64.40 \pm 6.03}$ & $\mathbf{77.53 \pm 4.20}$ \\
 & Fixed Orthogonal $B$, $r=4$ & $82.90 \pm 1.95$ & $\mathbf{73.03 \pm 2.19}$ & $77.65 \pm 2.00$ & $90.77 \pm 1.67$ & $\mathbf{70.53 \pm 3.93}$ & $\mathbf{79.31 \pm 1.87}$ & $97.57 \pm 0.25$ & $61.33 \pm 2.04$ & $75.30 \pm 1.51$ \\
 & Fixed Orthogonal $B$, $r=8$ & $\mathbf{85.40 \pm 1.21}$ & $72.17 \pm 3.85$ & $\mathbf{78.20 \pm 2.57}$ & $\mathbf{92.33 \pm 1.70}$ & $65.50 \pm 2.65$ & $76.63 \pm 2.30$ & $\mathbf{97.77 \pm 0.49}$ & $63.37 \pm 2.72$ & $76.88 \pm 2.08$ \\
\midrule
\multirow{5}{*}{Oxford Pets} & Dense CoOp & $89.13 \pm 2.03$ & $94.23 \pm 2.78$ & $91.57 \pm 0.23$ & $91.10 \pm 2.95$ & $92.73 \pm 4.45$ & $91.87 \pm 2.96$ & $93.93 \pm 0.45$ & $92.40 \pm 3.90$ & $93.13 \pm 1.76$ \\
 & Factorized CoOp, $r=4$ & $\mathbf{93.50 \pm 0.66}$ & $\mathbf{97.50 \pm 0.17}$ & $\mathbf{95.46 \pm 0.36}$ & $93.33 \pm 1.31$ & $\mathbf{96.43 \pm 0.96}$ & $\mathbf{94.86 \pm 1.11}$ & $94.00 \pm 0.66$ & $\mathbf{96.17 \pm 0.12}$ & $\mathbf{95.07 \pm 0.39}$ \\
 & Factorized CoOp, $r=8$ & $93.13 \pm 1.76$ & $97.07 \pm 0.21$ & $95.05 \pm 0.87$ & $\mathbf{94.30 \pm 0.82}$ & $95.30 \pm 2.10$ & $94.79 \pm 1.34$ & $93.63 \pm 0.93$ & $92.20 \pm 4.91$ & $92.88 \pm 2.90$ \\
 & Fixed Orthogonal $B$, $r=4$ & $92.77 \pm 1.17$ & $95.97 \pm 1.72$ & $94.33 \pm 0.75$ & $93.17 \pm 1.91$ & $95.90 \pm 1.25$ & $94.51 \pm 1.54$ & $\mathbf{94.50 \pm 0.17}$ & $95.57 \pm 0.85$ & $95.03 \pm 0.43$ \\
 & Fixed Orthogonal $B$, $r=8$ & $91.30 \pm 1.65$ & $\mathbf{97.50 \pm 0.44}$ & $94.30 \pm 1.08$ & $91.63 \pm 2.68$ & $96.17 \pm 0.31$ & $93.83 \pm 1.55$ & $93.73 \pm 0.65$ & $91.53 \pm 6.98$ & $92.53 \pm 3.66$ \\
\midrule
\multirow{5}{*}{UCF101} & Dense CoOp & $71.83 \pm 2.28$ & $61.40 \pm 4.05$ & $66.12 \pm 1.89$ & $\mathbf{79.57 \pm 0.83}$ & $52.93 \pm 2.08$ & $63.55 \pm 1.36$ & $\mathbf{84.60 \pm 1.04}$ & $53.83 \pm 8.78$ & $65.50 \pm 6.21$ \\
 & Factorized CoOp, $r=4$ & $\mathbf{74.77 \pm 2.50}$ & $\mathbf{70.17 \pm 1.85}$ & $\mathbf{72.38 \pm 1.83}$ & $79.20 \pm 1.61$ & $\mathbf{64.23 \pm 2.10}$ & $\mathbf{70.93 \pm 1.66}$ & $84.57 \pm 0.71$ & $53.83 \pm 1.70$ & $65.77 \pm 1.21$ \\
 & Factorized CoOp, $r=8$ & $73.87 \pm 3.04$ & $65.77 \pm 1.29$ & $69.56 \pm 1.45$ & $79.43 \pm 1.48$ & $63.60 \pm 10.23$ & $70.39 \pm 7.11$ & $83.77 \pm 0.15$ & $55.23 \pm 6.02$ & $66.44 \pm 4.44$ \\
 & Fixed Orthogonal $B$, $r=4$ & $72.03 \pm 2.57$ & $67.40 \pm 5.10$ & $69.60 \pm 3.59$ & $78.77 \pm 0.15$ & $63.07 \pm 2.17$ & $70.04 \pm 1.41$ & $83.83 \pm 0.64$ & $57.23 \pm 2.42$ & $68.01 \pm 1.92$ \\
 & Fixed Orthogonal $B$, $r=8$ & $73.87 \pm 0.85$ & $66.00 \pm 3.68$ & $69.69 \pm 2.44$ & $78.57 \pm 0.49$ & $58.63 \pm 8.18$ & $66.92 \pm 5.27$ & $84.37 \pm 0.15$ & $\mathbf{58.77 \pm 2.22}$ & $\mathbf{69.26 \pm 1.51}$ \\
\bottomrule
\end{tabular}
}
\end{table*}

\clearpage
\section{Research assets and responsible use}
\label{app:assets}

\subsection{Existing datasets, models, and software}

The implementation builds on CoOp, Dassl, and OpenAI CLIP.
The repository retains the inherited MIT license for
CoOp-derived code and the MIT license for bundled Dassl.
Its third-party notices identify the OpenAI CLIP
implementation and its upstream MIT license.
These software licenses do not establish the licensing
terms of the benchmark datasets.

The dataset and pretrained-model sources are credited
in the paper. A complete verified record of the terms
and versions applicable to all seven historical dataset
copies has not yet been recovered. Pretrained CLIP
weights are obtained through the upstream download
mechanism and are not included in the preserved
release assets.

\subsection{Code and result availability}

The preserved project contains implementation code,
experiment configurations, launch scripts, a workbook
exporter, numerical result records, and reproduction
documentation. Dense, joint-factorized, and transfer
records have archived CSV counterparts, while some
fixed-basis and class-generalization results are
preserved only in the workbook.

Complete historical training logs and trained factor
checkpoints were not recovered in the repository audit.
The available records support reconstruction of the
preserved numerical summaries, subject to the coverage
of the exporter, rather than exact replay of every
historical training run.

\subsection{Potential impacts and limitations of use}

Reducing the number of trainable prompt parameters
can reduce the adaptation state associated with each
task. This study does not establish reductions in
training energy, total GPU memory, or inference latency.

The adapted models inherit limitations of the pretrained
CLIP model and the benchmark data. The experiments do
not evaluate demographic fairness, privacy protection,
or suitability for consequential deployment.
The findings concern CoOp-style text prompts on the
evaluated classification benchmarks; they do not
establish the same behavior for deeper multimodal
prompting systems or other pretrained model families.


\end{document}